\newif\ifdraft \draftfalse
\newif\iffull \fulltrue

\documentclass{article}
\usepackage[utf8]{inputenc}

\usepackage{amssymb,amsmath,amsthm,amsfonts}
\usepackage{amsfonts,algorithm,algorithmic}
\usepackage{fullpage}
\usepackage{xcolor}
\usepackage{booktabs}
\usepackage{xspace}
\usepackage{kpfonts}
\usepackage[comma,sort&compress]{natbib}
\usepackage{xifthen}
\usepackage{cleveref}
\usepackage{enumitem}
\setlist{nolistsep}

\definecolor{DarkGreen}{rgb}{0.1,0.5,0.1}
\definecolor{DarkRed}{rgb}{0.5,0.1,0.1}
\definecolor{DarkBlue}{rgb}{0.1,0.1,0.5}

\newcommand{\ar}[1]{\ifdraft\textcolor{brown}{[AR: #1]}\fi}
\newcommand{\sw}[1]{\ifdraft\textcolor{blue}{[SW: #1]}\fi}
\newcommand{\bo}[1]{\ifdraft\textcolor{orange}{[Bo: #1]}\fi}

\newcommand{\jm}[1]{\ifdraft\textcolor{red}{[JM: #1]}\fi}

\newtheorem{fact}{Fact}[section]
\newtheorem{lemma}{Lemma}[section]
\newtheorem{theorem}{Theorem}[section]
\newtheorem{corollary}{Corollary}[section]

\newtheorem{condition}{Condition}

\theoremstyle{definition}
\newtheorem{definition}{Definition}
\theoremstyle{remark}
\newtheorem{remark}{Remark}

\newcommand{\given}{~\middle|~}

\DeclareMathOperator*{\Expect}{\mathbb{E}}
\DeclareMathOperator*{\argmax}{arg\,max}
\newcommand{\E}[1]{\Expect\left[ #1 \right]}
\newcommand{\Ex}[2]{\Expect_{#1}\left[ #2 \right]}

\newcommand{\pr}[1]{\mathbb{P}\left[ #1 \right]}
\renewcommand{\Pr}[2]{\mathbb{P}_{#1}\left[ #2 \right]}

\newcommand{\defeq}{:=}  

\newcommand{\Var}{\textrm{Var}}

\newcommand{\tr}{\intercal}  
\newcommand{\trans}[1]{{#1}^{\tr}}

\newcommand{\R}{\ensuremath{\mathbb{R}}\xspace}
\renewcommand{\D}{\mathcal{D}\xspace}
\newcommand{\Id}{\mathbf{I}}  

\newcommand{\regret}[1]{\textrm{Regret}(#1)\xspace}

\newcommand{\hgood}{\ensuremath{\widehat{\textrm{good}}}\xspace}
\newcommand{\haus}{\ensuremath{\widehat{\textrm{auspicious}}}\xspace}

\newcommand{\lmin}{\lambda_{\textrm{min}}}
\newcommand{\lmax}{\lambda_{\textrm{max}}}

\newcommand{\Adv}{\mathcal{A}}
\newcommand{\Asig}{\Adv_{\sigma}}

\newcommand{\bh}{\hat{\beta}}

\newcommand{\XTX}{\trans{(X^t)}X^t}
\newcommand{\XTXi}{\trans{(X_i^t)}X_i^t}
\newcommand{\XTXinv}{\left( \XTX \right)^{-1}}
\newcommand{\XTXinvXT}{\XTXinv \trans{(X^t)}}

\newcommand{\ux}{\ensuremath{\mu}}

\newcommand{\uxit}[2] {\ensuremath{\ux_{{#1}}^{{#2}}}}

\newcommand{\px}{\ensuremath{x}}
\newcommand{\pxi}[1]{\ensuremath{\px_{{#1}}}}
\newcommand{\pxt}[1]{\ensuremath{\px^{{#1}}}}
\newcommand{\pxit}[2]{\ensuremath{\px_{{#1}}^{{#2}}}}

\newcommand{\mxt}[1]{\ensuremath{X^{{#1}}}}
\newcommand{\mxit}[2]{\ensuremath{X^{{#2}}_{{#1}}}}

\renewcommand{\b}{\ensuremath{\beta}}
\newcommand{\bi}[1]{\ensuremath{\b_{{#1}}}}

\newcommand{\bhit}[2]{\ensuremath{{\bh_{{#1}}^{{#2}}}}}

\newcommand{\ei}[1]{\ensuremath{e_{{#1}}}}
\newcommand{\eit}[2]{\ensuremath{e_{{#1}}^{{#2}}}}
\newcommand{\tempit}[2]{\ensuremath{\varepsilon}_{{#1}}^{{#2}}}
\newcommand{\temppit}[2]{\ensuremath{\left(\varepsilon'\right)}_{{#1}}^{{#2}}}
\newcommand{\tempppit}[2]{\ensuremath{\left(\varepsilon''\right)}_{{#1}}^{{#2}}}

\newcommand{\nit}[2]{\eta_{{#1}}^{{#2}}}

\newcommand{\yit}[2]{y_{{#1}}^{{#2}}}

\newcommand{\ri}[1]{r_{{#1}}}
\newcommand{\rit}[2]{r_{{#1}}^{{#2}}}

\newcommand{\Ht}[1]{\mathcal{H}^{{#1}}}

\newcommand{\hist}[1]{h^{{#1}}}
\newcommand{\histt}[2]{h^{{#1, #2}}}

\newcommand{\cit}[2]{c_{{#1}}^{{#2}}}
\newcommand{\chit}[2]{\hat{c}_{{#1}}^{{#2}}}

\title{A Smoothed Analysis of the Greedy Algorithm for the Linear Contextual Bandit Problem}
\author{Sampath Kannan \and Jamie Morgenstern \and Aaron Roth \and Bo Waggoner \and Zhiwei Steven Wu}
\begin{document}
\maketitle

\begin{abstract}
  Bandit learning is characterized by the tension between long-term
  exploration and short-term exploitation.  However, as has recently
  been noted, in settings in which the choices of the learning
  algorithm correspond to important decisions about individual people
  (such as criminal recidivism prediction, lending, and sequential
  drug trials), exploration corresponds to explicitly sacrificing the
  well-being of one individual for the potential future benefit of
  others. This raises a fairness concern. In such settings, one might
  like to run a ``greedy'' algorithm, which always makes the
  (myopically) optimal decision for the individuals at hand --- but
  doing this can result in a catastrophic failure to learn. In this
  paper, we consider the linear contextual bandit problem and revisit
  the performance of the greedy algorithm.  We give a smoothed
  analysis, showing that even when contexts may be chosen by an
  adversary, small perturbations of the adversary's choices suffice
  for the algorithm to achieve ``no regret'', perhaps (depending on
  the specifics of the setting) with a constant amount of initial
  training data.  This suggests that ``generically'' (i.e. in slightly
  perturbed environments), exploration and exploitation need not be in
  conflict in the linear setting.
  \end{abstract}

\section{Introduction} \label{sec:intro}

Learning algorithms often need to operate in partial feedback settings
(also known as \emph{bandit} settings), in which the decisions of the
algorithm determine the data that it observes. Many real-world
application domains of machine learning have this flavor. Predictive
policing algorithms \citep{police} deploy police officers and receive
feedback about crimes committed and observed in areas the algorithm
chose to deploy officers. Lending algorithms \citep{lending} observe
whether individuals who were granted loans pay them back, but do not
get to observe counterfactuals: would an individual not granted a loan
have repaid such a loan? Algorithms which inform bail and parole
decisions \citep{sentencing} observe whether individuals who are
released go on to recidivate, but do not get to observe whether
individuals who remain incarcerated \emph{would} have committed crimes
had they been released. Algorithms assigning drugs to patients in
clinical trials do not get to observe the effects of the drugs that
were not assigned to particular patients.

Learning in partial feedback settings faces the well-understood
tension between \emph{exploration} and \emph{exploitation}. In order
to perform well, the algorithms need at some point to exploit the
information they have gathered and make the best decisions they
can. But they also need to explore: to make decisions which do not
seem optimal according to the algorithm's current point-predictions,
in order to gather more information about portions of the decision
space about which the algorithm has high uncertainty.

However, in practice, decision systems often do not explicitly
explore, for a number of reasons. Exploration is important for
maximizing long-run performance, but decision makers might be myopic
--- more interested in their short-term reward. In other situations,
the decisions made at each round affect the lives of individuals, and
explicit exploration might be objectionable on its face: it can be
repugnant to harm an individual today (explicitly sacrificing present
utility) for a potential benefit to hypothetical future individuals
(long-term learning rates) \citep{MSRExplorationWhitePaper,
  bastani}. For example, in a medical trial, it may be repugnant to
knowingly assign a patient a drug that is thought to be sub-optimal
(or even dangerous) given the current state of knowledge, in order to
increase statistical certainty. In a parole scenario, we may not want
to release a criminal that we estimate is at high risk for committing
violent crime. In such situations, exploration may be viewed as
\emph{unfair} to individuals, or to society at large.

On the other hand, a lack of exploration can lead to a catastrophic
failure to learn, which is highly undesirable -- and which can also
lead to unfairness. A lack of exploration (and a corresponding failure
to correctly learn about crime statistics) has been blamed as a source
of ``unfairness'' in predictive policing algorithms \citep{EFNSV17}. In
this paper, we seek to quantify how costly we should expect a lack of
exploration to be when the instances 
are not
entirely worst-case. In other words: is myopia a friction that we
should generically expect to quickly be overcome, or is it really a
long-term obstacle to learning?

\subsection{Our Results}
We study the \emph{linear contextual bandits problem}, which
informally, represents the following learning scenario which takes
place over a sequence of rounds $t$ (formal definitions appear in
Section \ref{sec:model}). At each round $t$, the learner must make a
decision amongst $k$ choices, which are represented by \emph{contexts}
$\pxit{i}{t} \in \mathbb{R}^d$. If the learner chooses action $i_t$ at
round $t$, he observes a reward $\rit{i_t}{t}$ --- but does not
observe the rewards corresponding to choices not taken. The rewards
are stochastic, and their expectations are governed by unknown linear
functions of the contexts. For an unknown set of parameters
$\beta_i \in \mathbb{R}^d$, $\mathbb{E}[r_i^t] = \beta_i \cdot
x_i^t$.
We consider two variants of the problem: in one (the \emph{single
  parameter setting}), all of the rewards are governed by the
\emph{same} linear function: $\bi{1} = \ldots = \bi{k} = \bi{}$. In
the other (the \emph{multiple parameter setting}), the parameter
vectors $\beta_i$ for each choice can be distinct. Normally, these two
settings are equivalent to one another (up to a factor of $k$ in the
problem dimension) --- but as we show, in our case, they have distinct
properties\footnote{Multi-parameter settings can be converted to
  single-parameter, by concatenating the $k$ multiple
  parameter vectors $\bi{i} \in \mathbb{R}^d$ into a single
  parameter vector $\b{} \in \mathbb{R}^{kd}$, and lifting contexts
  $\pxit{i}{t}$  into $k d$ dimensions with zeros in
  all but the $d$ relevant coordinates.}.

We study the \emph{greedy algorithm}, which trains a least-squares
estimate $\bhit{i}{t}$ on the current set of observations, and at each
round, picks the arm with the highest point-predicted reward:
$i_t = \arg\max_i\bhit{i}{t} \cdot \pxit{i}{t}$. In the single
parameter setting, greedy maintains just one regression estimate
$\bhit{}{t}$.

It is well known that the greedy algorithm does not obtain any
non-trivial worst-case regret bound: there are simple lower bound
instances for the greedy algorithm even in the simple stochastic
bandit setting, which is a special case of the contextual bandit
problem (in which the contexts are the same at every round). We give a
smoothed analysis which shows that these lower bound instances do not
arise ``generically.'' Specifically, we consider a model in which the
contexts $\pxit{i}{t}$ are chosen at each round by an adaptive
adversary, but are then perturbed by independent Gaussian perturbations in
each coordinate, with standard deviation $\sigma$. We show that in the
presence of these perturbations, the greedy algorithm recovers
diminishing regret bounds that depend only polynomially on
$1/\sigma$. We show that in this smoothed analysis setting, there is a
qualitative distinction between the single parameter and multiple
parameter settings:
\begin{enumerate}
\item In the single parameter setting (\Cref{sec:greedy}), the greedy
  algorithm with no initialization at all (having no initial
  observations)  with high probability obtains regret bounded by
  $\tilde{O}\left(\frac{\sqrt{T d}}{\sigma^2}\right)$ over $T$
  rounds. 
\item In the multiple parameter setting (\Cref{sec:multi}), the greedy
  algorithm requires a ``warm start'' -- that is, to start with a
  small number of observations for each action -- to obtain
  non-trivial regret bounds, even when facing a perturbed
  adversary. We show that if the warm start provides for each arm a
  constant number of examples (depending polynomially on fixed
  parameters of the instance, like $1/\sigma$, $d$, $k$, and
  $1/(\min_i ||\beta_i||)$), that may themselves be chosen by an
  adversary and perturbed, then with high probability greedy obtains
  regret
  $\tilde{O}\left(\frac{\sqrt{T k d}}{\sigma^2}\right)$. Moreover,
  this warm start is necessary: we give lower bounds showing that if
  the greedy algorithm is not initialized with a number of examples
  $n$ that grows polynomially with both $1/\sigma$ and with
  $1/\min_i ||\bi{i}||$, then there are simple fixed instances (that
  do not require an adaptive adversary) that force the algorithm to
  have regret growing linearly with $T$, with constant
  probability. (See Section \ref{sec:lb} for a formal statement of the lower bounds.)
\end{enumerate}
Our results extend beyond the particular perturbed adversary that we
study: we give more general conditions on the distribution over
contexts at each round that imply our regret bounds.

%
%
%
%

  \subsection{Related Work}\label{sec:rw}
  The most closely related piece of work (from which we take direct
  inspiration) is \citet{bastani}, who, in a stochastic setting, give conditions on the
  sampling distribution over contexts that causes the greedy algorithm
  to have diminishing regret in a closely related but incomparable
  version of the two-armed linear contextual bandits
  problem\footnote{\citet{bastani} study the setting in which there is
    only a single context at each round, shared between two
    actions. We study the setting in which each action is
    parameterized by its own context, and the number of actions $k$ can be arbitrary.}. The conditions on the context
  distribution given in that work are restrictive, however. They
  imply, for example, that every linear policy (and in particular the
  optimal policy) will choose each action with constant probability
  bounded away from zero. When translated to our perturbed adversarial
  setting, the distributional conditions of \citet{bastani} do not imply
  regret bounds that are sub-exponential in either the perturbation
  magnitude $\sigma$ or the dimension $d$ of the problem.

  There is a large literature focused on designing no-regret
  algorithms for contextual bandit problems
  (e.g. \citet{LCLS10,monster,LCLW11}), with a special focus on linear
  contextual bandit problems (e.g. \citet{CLRS11,APS11}). Unlike the
  greedy algorithm which we study, these algorithms explicitly
  explore. When the greedy algorithm gets a ``warm start'' (as is our
  case in the multiple parameter setting), it can be viewed as an
  instantiation of an \emph{explore-then-exploit} algorithm, in which
  the ``warm start'' is viewed as the exploration phase. However,
  there are two important distinctions between our results and
  standard explore-then-exploit algorithms. First,
  explore-then-exploit algorithms are typically analyzed in settings
  in which the contexts are drawn i.i.d. from some distribution, and
  do not obtain diminishing regret guarantees in adversarial
  settings. In our setting, the contexts are selected by a (perturbed)
  adversary. Second, the number of examples with which we need to seed
  the greedy algorithm is much smaller than the number of exploration
  rounds that would be needed for a non-trivial regret guarantee in
  the standard setting: to obtain a regret guarantee that is
  diminishing in $T$, the exploration phase needs to scale
  polynomially with $T$. In contrast, the size of our ``warm start''
  is independent of $T$.

  Our style of analysis is inspired by the smoothed analysis
  framework, introduced by \citet{ST04}. ``Smoothed analysis'' studies
  the performance of an algorithm on instances that can be chosen by
  an adversary, but are then perturbed randomly, and represent an
  interpolation between worst-case and average-case analyses. Smoothed
  analysis was originally introduced to study the \emph{running time}
  of algorithms, including the simplex algorithm \citep{ST04}, the
  Perceptron algorithm for linear programming \citep{BD02}, and the
  k-means algorithm \citep{AMR11}. In our case, we are not interested
  in running time (the algorithm we study always runs in polynomial
  time), but are instead interested in how the regret bound of the
  greedy algorithm behaves on smoothed instances. This is similar in
  spirit to how ``smoothed analyses'' are used to study other learning
  problems, as in \citet{KST09} and \citet{BCMV14}.
\iffull

  The general relationship between the exploration-exploitation
  tradeoff in bandit problems and fairness in machine learning has
  been studied by a number of recent works.  \citet{JKMR16,JKMNR16} and
  \citet{LRDMP17} study how to design algorithms which satisfy
  particular definitions of fairness at every round of the learning
  process.  \citet{JJKMR17} extend this line of work to
  reinforcement learning settings. \citet{KKMPRVW17} explicitly consider myopic behavior as a friction to fairness in bandit settings, and studies economic interventions to alleviate it. \citet{MSRExplorationWhitePaper} consider
  in general the ways in which exploration can be problematic when the
  decisions involved concern human beings, and
  \citet{EFNSV17} study models of predictive policing, and demonstrate
  how a lack of exploration can lead to feedback loops which
  perpetuate unfairness.  \citet{bastani} was motivated in
  part by the moral imperative not to ``explore'' in life-or-death
  applications like clinical trials. Our work can be viewed as showing
  that (in linear settings) in the presence of small perturbations, the frictions to learning
  and fairness introduced by myopic decision making should not be
  expected to persist indefinitely.
\fi





\section{Model and Preliminaries} \label{sec:model}
We now introduce the notation and definitions we use for this work.
For a vector $x$, $\|x\|$ represents its Euclidean norm.  We consider
two variants of the $k$-arm linear contextual bandits problem. The
first setting has a single $d$-dimensional model $\b$ which governs
rewards for all contexts $x \in \R^d$:
\paragraph{Single Parameter Setting.}{ There is a single unknown
  vector $\b \in \R^d$. In round $t$, contexts
  $\pxit{1}{t},\ldots,\pxit{k}{t}$, are presented, with each
  $\pxit{i}{t} \in \R^d$. The learner chooses an arm
  $i^t \in \{1,\ldots,k\}$, and obtains $s$-subgaussian\footnote{A
    random variable $Y$ with mean $\mu$ is $s$-subgaussian if
    $\E{e^{t(Y-\mu)}} \leq e^{t^2/2s}$ for all $t >
    0$.} 
  reward $\rit{i^t}{t}$ such that
  $\E{\rit{i^t}{t}} = \b \cdot \pxit{i^t}{t}$.  The regret of a
  sequence of actions and contexts of length $T$ is:
 \[ \textrm{Regret}(T) = \textrm{Regret}(\pxit{}{1},i^1,\ldots,\pxit{}{T},i^T) = \sum_{t=1}^T
    \left(\max_i \b\cdot \pxit{i}{t} - \b \cdot \pxit{i^t}{t}\right) . \]}
\noindent The second variant we consider has $k$ distinct models
governing the rewards for different arms:

\paragraph{Multiple Parameter Setting.}{ There are $k$ unknown vectors
  $\bi{1},\ldots,\bi{k} \in \R^d$.  In round $t$, contexts
  $\pxit{1}{t},\ldots, \pxit{k}{t}$ are presented, with each
  $\pxit{i}{t} \in \R^d$. The learner chooses an arm
  $i^t \in \{1,\ldots,k\}$, and obtains $s$-subgaussian reward
  $\rit{i^t}{t}$ such that
  $\E{\rit{i^t}{t}} = \bi{i^t} \cdot \pxit{i^t}{t}$. The regret of a
  sequence of actions and contexts of length $T$ is:
    \[
\textrm{Regret}(T) =  \textrm{Regret}(\pxt{1},i^1,\ldots,\pxt{T},i^T) = \sum_{t=1}^T \left(\max_i \bi{i}\cdot \pxit{i}{t} - \bi{i^t} \cdot \pxit{i^t}{t}\right) . \]
}

\noindent We refer to an algorithm as \emph{no-regret}  if, with high probability over the
randomness in the input, $\frac{1}{T}\textrm{Regret}(T) \leq f(T)$ for some
function $f(T) = o(1)$.
Throughout this paper we will fix a normalization so that
$\|\bi{i}\| \leq 1$.

\subsection{Perturbed Adversaries}
We consider a ``smoothed analysis'' setting that interpolates between
an i.i.d. distributional assumption on the contexts, and a worst-case,
completely adversarial input.  Specifically, we think of contexts as
generated by a two-stage process.  First, an adaptive adversary picks
the contexts arbitrarily from the unit ball.  Then each context is
perturbed independently by draws from a Gaussian with mean 0 and
variance $\sigma ^ 2$. We refer to the overall process as a
\textit{perturbed adversary}, and formalize it next.

The \emph{history} of an algorithm up until time $T$ represents a
transcript of the input-output behavior of the algorithm through time
$T$, and is sufficient for reconstructing the internal state of the
algorithm. In this paper we will concern ourselves with deterministic
algorithms, and so will not complicate notation by referring to a
transcript that encodes the internal randomness of the algorithm, but in
general, a history would include any internal randomness of the
algorithm used up until time step $T$ as well:
\begin{definition}
  The domain of \emph{histories} up through time $T$ is denoted
  $\Ht{T} = \left(\left(\R^d\right)^k \times \{1,\dots,k\} \times
    \R\right)^T$.
  A history $\hist{T} \in \Ht{T}$ represents a $T$-tuple.  Each
  element $\histt{T}{t}$ corresponding to time step $t$ is of the form
  $(\pxi{1},\dots,\pxi{k}, i^t, \rit{i^t}{t})$, i.e. the context list,
  the action choice $i^t$ and its corresponding reward
  $\rit{i^t}{t} = \bi{i^t} \pxit{i^t}{t} + \nit{i^t}{t}$, where
  $\nit{i^t}{t}$ represents the subgaussian deviation of the
  realization of the reward from its mean.
\end{definition}
\noindent Formally, an \emph{adversary} $\mathcal{A}$ (possibly
randomly) maps from histories to $k$ contexts:
$\mathcal{A}:\Ht{*} \to \left(\R^d\right)^k$. In particular, this
means that $\mathcal{A}$ can be adaptive. We denote the output of the
adversary by $(\mu_1, \mu_2, \ldots , \mu_k)$\footnote{The notation is
  chosen since these outputs will be the means of the outputs of the
  perturbed adversary.}.  We assume that the adversary always outputs
contexts with norms bounded by 1.  Next we define the notion of a
perturbed adversary, which encompasses both stages of the
context-generation process.
\begin{definition}[Perturbed Adversary]
  For any $\mathcal{A}$, the \emph{$\sigma$-perturbed adversary
    $\Asig$} is defined by, in  round $t$:
\begin{enumerate}
\item Given history $\hist{t-1}$, let
  $\mathcal{A}(\hist{t}) = \uxit{1}{t},\dots,\uxit{k}{t}$. Independently
  of the adversary's choice, each $\eit{1}{t},\dots,\eit{k}{t}$ is drawn independently from $\mathcal{N}(0,\sigma^2 I)$.
  \item The $\sigma$-perturbed adversary $\mathcal{A}_\sigma$ outputs the list of contexts $(\pxit{1}{t}, \ldots , \pxit{k}{t})  =  (\uxit{1}{t} + \eit{1}{t}, \dots, \uxit{k}{t} +\eit{k}{t})$.
\end{enumerate}
\end{definition}
It will be useful for analysis to consider more general perturbations.
We use \emph{perturbed adversary} to refer to the same process, but
where the perturbations \ar{I suggest calling these ``perturbations''
  everywhere} \jm{agreed, ``noise'' can refer to the variation in
  payoff if we use it at all.} $\eit{1}{t},\dots,\eit{k}{t}$ may be
drawn from different mean-zero distributions $\D_1^t,\dots,\D_k^t$,
chosen possibly as a function of $\hist{t}$.  They are still required
to be independent of each other and the adversary's choices
$\uxit{1}{t},\dots,\uxit{k}{t}$ conditioned on the history $\hist{t}$.

While the adversary always picks points within the unit ball,
perturbations can push contexts outside the ball. We will
sometimes truncate the perturbations so that the resulting contexts
are contained in a ball that is not too big. Thinking of such
truncated perturbations, we define a perturbed adversary to be
$R$-bounded if with probability $1$, $\|\pxit{i}{t}\| \leq R$ for all
$i$ and $t$ and all histories $\hist{t}$.  We call perturbations
\emph{$(r,\delta)$-centrally bounded} if, for each $\hist{t}$, arm
$i$, and fixed unit vector $w$, we have $w \cdot \eit{i}{t} \leq r$
with probability $1-\delta$.\jm{ I think it's a bit confusing to talk
  about boundedness of the adversary and of the perturbations not
  using a different qualifier besides the probability of failure,
  since the adversary could be randomized...}  \bo{A bit of trouble
  following you here. What I had in mind is the adversary can be any
  random process, but at the end, it produces a point $\uxit{i}{t}$
  inside the unit ball.  So that way, we can always talk about the
  randomness of the perturbations but not of the adversary. If that
  makes sense, maybe we can add some writing to make it more clear.}

We can interpret the output of a perturbed adversary as being a mild
perturbation of the (unperturbed) adaptive adversary when the
magnitude of the perturbations is smaller than the magnitude of the
original context choices $\mu_i$ themselves. Said another way, we can
think of the perturbations as being mild when they do not
substantially increase the norms of the contexts with probability at
least $1-\delta$. This will be the case throughout the run of the
algorithm (via a union bound over $T$) when
$\sigma \leq O\left(\tfrac{1}{\sqrt{d\ln \tfrac{Tkd}{\delta}}} \right)$. We
refer to this case as the ``low perturbation regime''. Because we view
this as the most interesting parameter regime, in the body of the
paper, we restrict attention to this case. The ``high perturbation
regime'' can also be analyzed, and we state results for the high
perturbation regime in the appendix.

\subsection{Proof Approach and Key Conditions} \label{subsec:proof-approach}
Our goal will be to show that the greedy algorithm achieves no regret
against any perturbed adversary in both the single-parameter and
multiple-parameter settings.  The key idea is to show that the
distribution on contexts generated by perturbed adversaries satisfy
certain conditions which suffice to prove a regret bound. The
conditions we work with are related to (but substantially weaker than)
the conditions shown to be sufficient for a no regret guarantee by
\citet{bastani}.

The first key condition, \emph{diversity} of contexts, considers the
positive semidefinite matrix $\E{x\trans{x}}$ for a context $x$, and
asks for a lower bound on its minimum eigenvalue.  Intuitively, this
corresponds to $x$'s distribution having non-trivial \emph{variance}
in all directions, which is necessary for the least squares estimator
to converge to the underlying parameter $\beta$: when we make
observations of the subgaussian reward centered at $\beta \cdot x$, we
infer more information about $\beta$.\footnote{If the minimum
  eigenvalue is zero, the covariance matrix is not of full rank, and
  $\beta$ would not be uniquely specified by the data.}

We only observe the rewards for contexts $x$ conditioned on Greedy
selecting them, implying that we see a biased or \emph{conditional}
distribution on $x$.  To handle this, we have a different notion of
diversity (a departure from \citet{bastani}, who require a related
condition on the unconditioned distribution).
\begin{condition}[Diversity] \label{condition:diversity}
  Let $\ei{} \sim \D$ on $\R^d$.
  We call $\D$  \emph{$(r,\lambda_0)$-diverse} if for all $\bh$, $\mu \leq 1$, and  $\hat{b} \leq r\|\bh\|$, for  $x = \mu + \ei{}$:
  \[ \lmin \left( \Ex{\ei{} \sim \D}{x \trans{x} \given \bh \cdot \ei{} \geq \hat{b} }\right) \geq \lambda_0 . \]
A perturbed adversary satisfies $(r,\lambda_0)$-diversity if for all $\hist{t}$ and all $i$, the distribution $\D_i^t$ is $(r,\lambda_0)$-diverse.

\end{condition}

Intuitively, diversity will suffice to imply no regret in the single
parameter setting, because when an arm is pulled, the context-reward
pair gives useful information about all components of the (single)
parameter $\beta$.  In the multiple parameter setting, diversity will
suffice to guarantee that the learner's estimate of arm $i$'s
parameter vector converges to $\beta_i$ as a function of the number of
times arm $i$ is pulled; but it is not on its own enough to guarantee
that arm $i$ will be pulled sufficiently often (even in rounds for
which it is the best alternative, when failing to pull it will cause
our algorithm to suffer regret)\footnote{The unconditioned diversity
  condition from \cite{bastani} \emph{is} enough to imply that each
  arm will be pulled sufficiently often, and hence yield a no regret
  guarantee --- but unfortunately, that condition will not be
  satisfied by a perturbed adversary.}.

The second key condition, \emph{margins}, will intuitively imply that
conditioned on an arm being optimal on a given round, there is a
non-trivial probability (over the randomness in the contexts) that
Greedy perceives it to be optimal based on current estimates
$\{\bh_i^t\}$, so long as the current estimates achieve at least some
constant baseline accuracy. Combined with a small initial training set
sufficient to guarantee that initial estimates achieve error bounded
by a constant, this implies that Greedy will continue to explore arms
with a frequency that is proportional to the number of rounds for
which they are optimal; this implies by diversity that estimates of
those arms' parameters will improve quickly (without promising
anything about arms that are rarely optimal -- and hence
inconsequential for regret). Again, \citet{bastani} require a related
but stronger margin condition that does not condition on the arm in
question being optimal.

\begin{condition}[Conditional Margins] \label{condition:margin} Let
  $\ei{} \sim \D$.  We say $\D$ has
  \emph{$(r,\alpha,\gamma)$ margins} if for all $\b \neq 0$ and
  $b \leq r\|\b\|$,
  \[ \pr{\b \ei{} > b + \alpha \|\b\| \given \b \cdot \ei{} \geq b } \geq \gamma  . \]
  A perturbed adversary satisfies \emph{$(r,\alpha,\gamma)$ margins} if for all $\hist{t}$ and $i$, the distribution $\D_i^t$ has $(r,\alpha,\gamma)$ margins.
\end{condition}
The condition intuitively requires that on those rounds for which arm
$i$ has the largest expected reward, there is a non-negligible
probability $(\gamma)$ that its expected reward is largest by at least
some margin ($\alpha\| \beta \|$).  If Greedy's
estimates $\{\bh_i^t\}$ are accurate enough, this will imply that arm
$i$ is actually pulled by Greedy.

\iffull
\paragraph{Complications: extreme perturbation realizations.}
When the realizations of the Gaussian perturbations have extremely
large magnitude, the diversity and margin conditions will not
hold\footnote{One can gain intuition from the one-dimensional case,
  where a lower truncated Gaussian becomes highly concentrated around
  the minimal value in its range.}. This is potentially problematic,
because the probabilistic conditioning in both conditions increases
the likelihood that the perturbations will be large.  This is the role
of the parameter $r$ in both conditions: to provide a reasonable upper
bound on the threshold that a perturbation variable should not exceed.
exceed.  In the succeeding sections, we will use conditions we call
``good'' and ``auspicious'' to formalize the intuition that this is
unlikely to happen often, so that for a constant fraction of rounds,
the diversity and margin conditions will be satisfied (and that this
is sufficient to prove a regret bound).
\fi

\section{Single Parameter Setting} \label{sec:greedy} In the linear
contextual bandits setting, we define the ``Greedy Algorithm'' as the
algorithm which myopically pulls the ``best'' arm at each round
according to the predictions of the classic least-squares estimator.

Let $\mxt{t}$ denote the $(t-1)\times d$ design matrix at time $t$, in
which each row $t'$ is some observed context $\pxit{i^{t'}}{t'}$ where
arm $i^{t'}$ was selected at round $t' < t$.  The corresponding vector
of rewards is denoted
$\yit{}{t} = (\rit{i^1}{1},\dots,\rit{i^{t-1}}{t-1})$. The transposes
of a matrix $Z$ and vector $z$ are denoted $\trans{Z}$ and
$\trans{z}$.  At each round $t$, Greedy first computes the
least-squares estimator based on the historical contexts and rewards:
$\bhit{}{t} \in \arg\min_{\b} ||\mxt{t}\b - \yit{}{t}||_2^2$, and then
greedily selects the arm with the highest estimated reward:
$i^t = \argmax_i \bhit{}{t} \cdot \pxit{i}{t}$. \iffull

The ``Greedy Algorithm'' in this setting is defined in
\Cref{alg:greedy-single}.
\begin{algorithm}
  \begin{algorithmic}
    \STATE Initialize the design matrix and reward vector $\mxt{1}, \yit{}{1}$, initially both empty.
    \FOR{$t = 1$ to $T$}
      \STATE Define $\bhit{}{t} \in \arg\min_{\b} ||\mxt{t}\b - \yit{}{t}||_2^2$ breaking ties arbitrarily when necessary. \\
             (When the covariance matrix is invertible the solution is unique: $\bhit{t} =\XTXinvXT \yit{}{t}$.)
      \STATE Observe contexts $\pxit{1}{t}, \ldots, \pxit{k}{t}$.
      \STATE Choose arm $i^t = \arg\max \bhit{}{t} \cdot \pxit{i}{t}$ and observe reward $\rit{i^t}{t}$.
      \STATE Append the new observations $\pxit{i^t}{t},\rit{i^t}{t}$ to $(\mxt{t},\yit{}{t})$ to form $(\mxt{t+1},\yit{}{t+1})$.
    \ENDFOR
  \end{algorithmic}
  \caption{Greedy (single parameter)} \label{alg:greedy-single}
\end{algorithm}
\fi
\paragraph{``Reasonable'' rounds.}
\iffull As discussed at the end of Section \ref{sec:model}, \fi
the diversity condition will only apply to an arm for perturbations
$\eit{i}{t}$ that are not too large; we formalize these ``good''
situations below.

Fix a round $t$, the current Greedy hypothesis $\bhit{}{t}$, and any
choices of the adversary $\uxit{1}{t},\dots,\uxit{k}{t}$ conditioned
on the entire history up to round $t$.  Now each value
$\bhit{}{t} \pxit{i}{t} = \bhit{}{t} \uxit{i}{t} + \bhit{}{t}
\eit{i}{t}$
is a random variable, and Greedy selects the arm corresponding to the
largest realized value.  In particular, consider arm $i$ and define
the ``threshold'' \iffull
\[ \chit{i}{t} \defeq \max_{j\neq i} \bhit{}{t} \pxit{j}{t} . \]
\fi Notice that $\chit{i}{t}$ is a random variable that depends on all
the perturbations $\eit{j}{t}$ for $j\neq i$, and also that Greedy
selects arm $i$ if and only if
$\bhit{}{t}\pxit{i}{t} \geq \chit{i}{t}$.\footnote{We ignore ties as
  they are measure-zero events.}

We say a realization of $\chit{i}{t}$ is \emph{$r$-\hgood (for arm $i$)} if
\iffull  \[ \chit{i}{t} \leq \bhit{}{t} \uxit{i}{t} + r\|\bhit{}{t}\|. \]
\fi
This is an event whose probability is determined by the perturbation added to the contexts of all arms except $i$, and it says intuitively that $\eit{i}{t}$ does not need to be too large in  order for arm $i$ to be selected.
Additionally, for a fixed $\bhit{}{t},\uxit{1}{t},\dots,\uxit{k}{t}$, we
would like to argue that, if arm $i$ was selected, then most likely
(over the randomness in \emph{all} the perturbations including $i$
), $\chit{i}{t}$ was $r$-\hgood.  We say
that $(\bhit{}{t},\uxit{1}{t},\dots,\uxit{k}{t})$ are \emph{$r$-\haus
  for $i$} if
\iffull  \[ \pr{\text{$\chit{i}{t}$ is $r$-\hgood} \given i^t = i} \geq \frac{1}{2}. \]
\fi
We will shorten this to saying a round $t$ is ``$r$-\haus for $i$''
with the implication that we are referring to Greedy's hypothesis and
the adversarial choices at that round.

\subsection{Regret framework for perturbed adversaries}
We first observe an upper-bound on Greedy's regret as a function of
the distance between $\bhit{}{t}$ and the true model $\b$.  Let
$i^*(t) = \argmax_i \b \cdot \pxit{i}{t}$, the optimal arm at time
$t$.  For shorthand, denote its context by
$\pxit{i^*}{t} := \pxit{i^*(t)}{t}$.

\begin{lemma} \label{lemma:singleparam-greedy-regret} Suppose for all
  $i, t$ that $\|\pxit{i}{t}\| \leq R$.  \iffull In the single-parameter
  setting \fi, for any $t_{\min} \in [T]$, we have:
    \[ \textrm{Regret}(\pxit{}{1},i^1,\ldots,\pxit{}{T},i^T) \leq 2R t_{\min} + 2 R \sum_{t=t_{\min}}^T \left\|\b - \bhit{}{t} \right\|  . \]
\end{lemma}
\noindent Given Lemma \ref{lemma:singleparam-greedy-regret}\iffull (whose proof
is in Appendix~\ref{sec:missing-single})\fi, it suffices to find
conditions under which $\bhit{}{t} \to \b$ quickly.  Intuitively, the
input data must be ``diverse'' enough (captured formally by
Definition~\ref{condition:diversity}) to infer $\b$.

\begin{lemma} \label{lemma:good-diversity} Against a perturbed
  adversary satisfying $(r,\lambda_0)$ diversity, for all $i,t$, we
  have
    \[ 
\lmin \left( \E{\trans{(\pxit{i}{t})}\pxit{i}{t} \given i^t = i, \chit{i}{t} \mbox{ is }r\mbox{-\hgood} }\right) \geq \lambda_0 . \]
  \end{lemma}
\iffull
\begin{proof}

  We begin by manipulating the quantity we wish to lower bound. Define $b = \chit{i}{t} - r || \bhit{}{t}$ for a fixed $\chit{i}{t}$. Then, we have
  \begin{align*}
  &   \lmin \left( \Ex{\substack{\forall j \\ \eit{j}{t}\sim \D_j}}{\trans{(\pxit{i}{t})}\pxit{i}{t} \given i^t = i, \chit{i}{t} \mbox{ is }r\mbox{-\hgood} }\right)\\
    & = \lmin \left( \Ex{\substack{\forall j \\ \eit{j}{t}\sim \D_j}}{\trans{(\pxit{i}{t})}\pxit{i}{t} \given \bhit{i}{t} \pxit{i}{t} = \bhit{}{t}\uxit{i}{t} + \bhit{}{t}\eit{i}{t} \geq \chit{i}{t}, \chit{i}{t} \mbox{ is }r\mbox{-\hgood} }\right)\\
    & = \lmin \left( \Ex{\substack{\forall j \\ \eit{j}{t}\sim \D_j}}{\trans{(\pxit{i}{t})}\pxit{i}{t} \given\bhit{}{t}\eit{i}{t} \geq \chit{i}{t} - \bhit{}{t}\uxit{i}{t}, \chit{i}{t} \mbox{ is }r\mbox{-\hgood} }\right)\\
    & = \lmin \left( \Ex{\substack{\forall j \\ \eit{j}{t}\sim \D_j}}{\trans{(\pxit{i}{t})}\pxit{i}{t} \given\bhit{}{t}\eit{i}{t} \geq \chit{i}{t} - \bhit{}{t}\uxit{i}{t}, \chit{i}{t} \leq \bhit{}{t}\uxit{i}{t} + r || \bhit{}{t}||}\right)\\
    & = \lmin \left( \Ex{\substack{\forall j\neq i \\ \eit{j}{t}\sim \D_j}|  i^t = i, \chit{i}{t} r-\hgood }{ \Ex{\eit{i}{t}\sim \D_i}{\trans{(\pxit{i}{t})}\pxit{i}{t} \given\bhit{}{t}\eit{i}{t} \geq \chit{i}{t} - \bhit{}{t}\uxit{i}{t}, \chit{i}{t} \leq \bhit{}{t}\uxit{i}{t} + r || \bhit{}{t}||}}\right)\\
    & \geq \Ex{\substack{\forall j\neq i \\ \eit{j}{t}\sim \D_j} | i^t = i, \chit{i}{t} r-\hgood } {\lmin \left(  \Ex{\eit{i}{t}\sim \D_i}{\trans{(\pxit{i}{t})}\pxit{i}{t} \given\bhit{}{t}\eit{i}{t} \geq \chit{i}{t} - \bhit{}{t}\uxit{i}{t}, \chit{i}{t} \leq \bhit{}{t}\uxit{i}{t} + r || \bhit{}{t}||}}\right)\\
    & = \Ex{\substack{\forall j\neq i \\ \eit{j}{t}\sim \D_j} | i^t = i, \chit{i}{t} r-\hgood } {\lmin \left(  \Ex{\eit{i}{t}\sim \D_i}{\trans{(\pxit{i}{t})}\pxit{i}{t} \given\bhit{}{t}\eit{i}{t} \geq b, b \leq  r || \bhit{}{t}||}}\right)\\
    & =  \Ex{\substack{\forall j\neq i \\ \eit{j}{t}\sim \D_j} | i^t = i, \chit{i}{t} r-\hgood } {\lambda_0}\\
    & = \lambda_0
  \end{align*}
  where the first string of equalities follow from the definitions,
  the inequality follows from the superadditivity of the minimum
  eigenvalue, and the second-to-last equality follows from diversity. 

\end{proof}
\fi Lemma~\ref{lem:beta-bound} shows $\bhit{}{t}$ approaches $\b$ at a
rate governed by the minimum eigenvalue of the design matrix. \iffull
Its proof is found in Appendix~\ref{sec:missing-single}.\fi

\begin{lemma}\label{lem:beta-bound}
  For each round $t$, let $Z^t = \XTX$.
  Suppose all contexts satisfy $\|\pxit{i}{t}\| \leq R$ and rewards are $s$-subgaussian.
 If $\lmin(Z^t) > 0$,
  then with probability $1-\delta$,
\[
  \| \b - \bhit{}{t} \| \leq \frac{
    \sqrt{2dRts\ln(td/\delta)}}{\lmin\left(Z^t\right)}.
\]
\end{lemma}

Observe that the matrix
$Z^t = \sum_{t'\leq t} (x_{i}^{t'})^\intercal x_{i}^{t'}$.
\Cref{lemma:good-diversity}, diversity, and a concentration inequality
for minimum eigenvalues~\citep{tropp2011user}, imply
$\lmin\left(Z^t\right)$ grows at a rate $\Omega(t)$ so long as must rounds are $r-\hgood$. Thus, a bounded,
diverse perturbed adversary with many auspicious rounds has estimators
that converge quickly as a function of $t$.

\begin{corollary} \label{cor:diverse-all-converge}
  Let $t_{\min}(\delta/T) = \max\left\{32\ln(4T/\delta) ~,~ \frac{80R^2 \ln(2dT/\delta)}{\lambda_0}\right\}$.
  Suppose the adversary is $R$-bounded and $(r,\lambda_0)$-diverse.
  If at most $\frac{t_{\min}(\delta/T)}{2}$ rounds $t$ are not $r$-\haus for $i^t$, then with probability $1 - \delta$, for all $t \geq t_{\min}(\delta/T)$,
    \[ \| \b - \bhit{}{t} \| \leq \frac{32 \sqrt{dRs \ln(2Td/\delta)}}{\lambda_0 \sqrt{t}} . \]
\end{corollary}

Furthermore, we can show that centrally bounded adversaries are auspicious in
nearly every round.

\begin{lemma} \label{lemma:single-half-ausp} For an
  $(r,\tfrac{1}{T})$-centrally-bounded adversary, with probability at least
  $1-\delta$, for each arm $i$, all but
  $2 + \sqrt{\tfrac{1}{2}\ln\frac{k}{\delta}}$ rounds in which
  $i^t = i$ are $r$-{\haus } for $i$.
\end{lemma}
\iffull
\begin{proof}
    \bo{There was confusion about this proof, has been UPDATED}
  Fix an arm $i$ and let $S_i = \{t : i^t = i\}$, the set of rounds in which $i$ is chosen by Greedy.
  Let $A_i$ be the set of rounds $t$ that are $r$-{\haus } for $i$.
  We wish to show $|\{t \in S_i : t \not\in A_i\}| \leq 2 + \sqrt{\tfrac{1}{2}\ln\frac{k}{\delta}}$.

  Let $\delta' = \tfrac{2}{T}$ and let $B_i$ be the set
  of rounds $t$ on which, fixing $\bhit{}{t}$ and
  $\{\uxit{j}{t}\}_{j=1}^k$, we have $\pr{i^t=i} \geq \delta'$.
  We claim that if $t \in B_i \cap S_i$, then $t$ is $r$-{\haus }
  for $i$, that is, $t \in A_i$.
  This claim implies $\{t \in S_i : t \not\in A_i\} \subseteq \{t \in S_i : t \not\in B_i\}$,
  so we will just need to upper-bound that size of the latter.
  To show this claim, fix some $t$ and $\bhit{}{t}, \{\uxit{j}{t}\}_{j=1}^k$ such that $t \in A_i$.
  Then
  \begin{align*}
    \pr{\chit{i}{t} > \bhit{}{t} \uxit{i}{t} + r\|\bhit{}{t}\| \given i=i^t}
      &= \frac{\pr{\chit{i}{t} > \bhit{}{t} \uxit{i}{t} + r\|\bhit{}{t}\|, ~ i=i^t}}{\pr{i=i^t}}  \\
      &= \frac{\pr{\bhit{}{t} \uxit{i}{t} + r\|\bhit{}{t}\| < \chit{i}{t} \leq \bhit{}{t} \pxit{i}{t}}} {\pr{i=i^t}}  \\
      &\leq \frac{\pr{\bhit{}{t} \uxit{i}{t} + r\|\bhit{}{t}\| < \chit{i}{t} \leq \bhit{}{t} \pxit{i}{t}}} {\delta'}  & (t \in A_i)  \\
      &\leq \frac{\pr{\bhit{}{t} \uxit{i}{t} + r\|\bhit{}{t}\| < \bhit{}{t} \pxit{i}{t}}} {\delta'}  \\
      &=    \frac{\pr{r\|\bhit{}{t}\| < \bhit{}{t} e_i^t}} {\delta'}  \\
      &= \frac{\pr{r < \frac{\bhit{}{t} e_i^t}{\|\bhit{}{t}\|}}} {\delta'}  \\
      &\leq \frac{1}{2}  & \left(\text{$\left(r,\frac{\delta'}{2}\right)$-centrally bounded}\right).
  \end{align*}
  \bo{We probably ought to special-case $\bhit{}{t}=0$ above, or point out this is a measure-$0$ event.}
  We now complete the proof of the lemma by upper-bounding $|\{t \in S_i : t\not\in B_i\}|$.
  Its distribution is
  stochastically dominated by a Binomial$(T,\delta')$ (the case where
  every round has $\pr{i^t=i} < \delta'$ by a tiny margin).  So by a
  tail bound for Binomials (Corollary \ref{cor:binomial-upper-tail}),
  with probability at most $\frac{\delta}{k}$, it is at most
  $T\delta' + \sqrt{\frac{1}{2}\ln\frac{k}{\delta}}$.  By a
  union-bound, this holds for all arms $i$ except with probability at
  most $\delta$.
\end{proof}\fi

We  have everything we need to show that the greedy algorithm has
low regret when facing a bounded, centrally bounded, diverse
adversary: since its regret can be captured in terms of the distance
of its OLS estimates to the true coefficients
(Lemma~\ref{lemma:singleparam-greedy-regret}), and its estimates
converge quickly (Corollary~\ref{cor:diverse-all-converge}) if it has
mostly auspicious rounds (which it does, by
Lemma~\ref{lemma:single-half-ausp}).

\iffull
\begin{theorem} \label{thm:singleparam-conditions-regret} Suppose in
  the single-parameter setting that a perturbed adversary is
  $R$-bounded and, for some $r \leq R$, is
  $(r, \frac{1}{T})$-centrally-bounded and $(r,\lambda_0)$-diverse.
  Recall that the reward errors are $s$-subgaussian.  Then with
  probability $1-\delta$, the greedy algorithm has regret bounded by
  \begin{align*}
    \text{Regret}(T)
      &\leq \max \begin{cases} \frac{128 R^{3/2} \sqrt{T d s \ln(2Td/\delta)}}{\lambda_0}  \\
                               \frac{320 R^3 \ln(2dT/\delta)}{\lambda_0}  \\
                               128R \ln(4T/\delta)  \\
                               16R + 8R\sqrt{\frac{1}{2} \ln \frac{k}{\delta}}.
                 \end{cases}
  \end{align*}
\end{theorem}

\begin{proof}
  By Lemma \ref{lemma:singleparam-greedy-regret} and $R$-boundedness, for any $t^*_{\min}$,
  \begin{align*}
    \text{Regret}(T)
      &\leq 2Rt^*_{\min} + 2 R \sum_{t=t^*_{\min}}^T \| \b - \bhit{}{t} \| .
  \end{align*}
  Let $t^*_{\min} := \max\left\{4 + 2\sqrt{\frac{1}{2}\ln\frac{k}{\delta}} ~,~ t_{\min}(\delta/T)\right\}$, where $t_{\min}(\delta/T) = \max\left\{32\ln(4T/\delta) ~,~ \frac{80R^2 \ln(2dT/\delta)}{\lambda_0}\right\}$.
  We show that the conditions of Corollary \ref{cor:diverse-all-converge} are satisfied, which will give a bound on $\|\b - \bhit{}{t}\|$.
  By Lemma \ref{lemma:single-half-ausp} and $(r, \frac{1}{T})$-central-boundedness, the number of rounds where $i^t =i$ but which are not $r$-{\haus }~  for $i$ is at most $2 + \sqrt{\frac{1}{2}\ln\frac{k}{\delta}}$, for each $i$.
  Because $2 + \sqrt{\frac{1}{2}\ln\frac{k}{\delta}} \leq \frac{t^*_{\min}}{2}$, Corollary \ref{cor:diverse-all-converge} gives that with probability $1-\delta$, for all $t \geq t^*_{\min}$,
    \[ \|\b - \bhit{}{t}\| \leq \frac{32 \sqrt{dRs \ln(2Td/\delta)}}{\lambda_0 \sqrt{t}} . \]
  So with probability at least $1-\delta$, we get
  \begin{align*}
    \text{Regret}(T)
      &\leq 2Rt^*_{\min} + 2 R \sum_{t=t^*_{\min}}^T \left\| \b - \bhit{}{t} \right\|  \\
      &\leq 2Rt^*_{\min} + \sum_{t=t^*_{\min}}^T \frac{64 R^{3/2} \sqrt{d s \ln(2Td/\delta)}}{\lambda_0 \sqrt{t}} \\
      &\leq 2Rt^*_{\min} + \frac{64 R^{3/2} \sqrt{T d s \ln(2Td/\delta)}}{\lambda_0}  \\
      &\leq 4 R \max\left\{ t^*_{\min} ~,~ \frac{32 \sqrt{T d R s \ln(2Td/\delta)}}{\lambda_0} \right\}.
  \end{align*}
  Plugging in the definition of $t^*_{\min}$ as a maximum over three
  expressions, we obtain the stated bound.
\end{proof}

\begin{remark}
  The three different bounds in
  \Cref{thm:singleparam-conditions-regret} naturally correspond to
  three different regimes on the parameters in our problem.  For each
  regime, we list the ``intuitive'' case for which it holds.
  \[ \text{Regret}(T) \leq \max \begin{cases}
       \frac{128 R^{3/2} \sqrt{T d s \ln(2Td/\delta)}}{\lambda_0}  &  \text{(``default'' bound)} \\
       \frac{320 R^3 \ln(2dT/\delta)}{\lambda_0}                  &  \text{($R$ very large or $s$ very small)} \\
       128R \ln(4T/\delta)                                        &  \text{($\lambda_0$ very large). }  \\
       16R + 8R\sqrt{\frac{1}{2} \ln \frac{k}{\delta}}            &  \text{($k$ exponentially large). }
     \end{cases} \]
  In this paper we focus on the first case, i.e. for fixed $s,k$ we focus on the asymptotics with respect to $T, d$, and small perturbations (as captured by $\lambda_0 \to 0$).
\end{remark}

\begin{remark}
  An upper bound on expected regret follows directly.
  First take the bound of \Cref{thm:singleparam-conditions-regret} with $\delta := \frac{1}{T}$;
  then, with the remaining probability $\delta$, regret is upper-bounded by
  $2RT$, so this contributes an additional additive expected regret of at
  most $2 \delta R T = 2R$.
\end{remark}

\fi

\iffull\subsection{The Gaussian, $\sigma$-perturbed adversary} \label{subsection:single-perturbed-adv}

We now apply the tools we developed in the previous
section to show that Greedy has diminishing regret when facing a
$\sigma$-perturbed adversary in the single-parameter setting, formally
captured in the following theorem.

\ar{Perhaps this theorem statement should be promoted to the beginning of this section? It is the main result of section 3.}
\bo{If it's stated in the intro, then not sure if it's important to.}
\fi

\begin{theorem} \label{thm:singleparam-perturbed-regret} In the single
  parameter setting against the $\sigma$-perturbed adversary $\Asig$,
  with probability $1-\delta$, for fixed $s$ (rewards' subgaussian
  parameter) and $k$ (number of arms) and
  $\sigma \leq O\left(\left(\sqrt{d\ln(Tkd/\delta)}\right)^{-1}\right)$, Greedy
  has
    \[  \text{Regret}(T) \leq O\left( \frac{\sqrt{T d s} \left(\ln\frac{Td}{\delta}\right)^{3/2}} {\sigma^2} \right). \]
%
%
\end{theorem}

\iffull

The proof of this theorem, which will be formally presented at the end
of this section, boils down to showing that the adversary is bounded
with high probability (Lemma~\ref{lemma:rhat-prob-mixture}), then
(conditioned on the adversary being bounded) showing the adversary is
centrally bounded (Lemma ~\ref{lemma:asigprime-bounded}) and diverse
(Lemma~\ref{lemma:singleparam-bounded-adv-diverse}), then applying
Theorem \ref{thm:singleparam-conditions-regret}.

The formal proof is slightly more complicated than just stated. Rather
than performing this analysis with respect to the original
Gaussian-perturbed adversary $\Asig$, which is unbounded, the analysis
is done with respect to a perturbed adversary $\Asig'$ where the perturbation
is drawn from the truncated Gaussian distribution.  We can view
$\Asig$ as a mixture distribution of $\Asig'$ (with high probability)
and some unbounded adversary $\Asig''$ (with the remainder), where a
hidden coin is flipped prior to running the algorithm which determines
whether contexts will be chosen according to $\Asig'$ or $\Asig''$.

\bo{I tried to write the below more ``precisely'' in math, however, I'm not sure it's better.}
$\Asig'$ and $\Asig''$ are defined as follows, where $\Asig'$ has carefully-truncated Gaussian perturbations (truncations chosen for ease of analysis) while $\Asig''$ may be unbounded.
Recall that, given an adversary $\Adv$, the Gaussian-perturbed adversary $\Asig$ is defined as $\Asig(\hist{t})_i = \Adv(\hist{t})_i + \mathcal{N}(0,\sigma^2\Id_d)$ for all histories $\hist{t}$ and $i=1,\dots,k$.

For all $i,t$, let $\tempit{i}{t} \in \R^d$ be distributed with each coordinate independently Normal$(0,\sigma^2)$.
Let $\temppit{i}{t} \in \R^d$ have each coordinate $j$ distributed independently with density $\pr{ \left(\temppit{i}{t}\right)_j = z } = \pr{ \left(\tempit{i}{t}\right)_j = z \given |\tempit{i}{t}| \leq \hat{R}}$.
In other words, each coordinate is a $[-\hat{R},\hat{R}]$-truncated Gaussian.
Let $\tempppit{}{}$ be distributed as $\pr{ \tempppit{}{} = \vec{z} } = \pr{\tempit{}{} = \vec{z} \given \exists i,t,j ~ \text{s.t.} ~ |\left(\tempit{i}{t}\right)_j| > \hat{R}}$.
In other words, all coordinates of all $\tempppit{i}{t}$ are drawn as joint Gaussians, but conditioned on the fact that at least one coordinate of at least one $\tempppit{i}{t}$ has absolute value larger than $\hat{R}$.

We observe that $\tempit{}{}$ can be viewed as a mixture of $\temppit{i}{t}$, with the probability that all Gaussians have absolute value at most $\hat{R}$; and $\tempppit{i}{t}$, with the remaining probability.
Now, given $\bhit{}{t}$, let
$Q^t$ be an orthonormal change-of-basis matrix such that
$Q^t \bhit{}{t} = (\|\bhit{}{t}\|, 0, \ldots, 0)$.
Then for each $i=1,\dots,k$, we let
\begin{align*}
  \Asig(\hist{t})_i   &= \Adv(\hist{t})_i + (Q^t)^{-1}\tempit{i}{t}  \\
  \Asig'(\hist{t})_i  &= \Adv(\hist{t})_i + (Q^t)^{-1}\temppit{i}{t}  \\
  \Asig''(\hist{t})_i &= \Adv(\hist{t})_i + (Q^t)^{-1}\tempppit{i}{t}.
\end{align*}
We have the following claim for this construction.
\begin{lemma} \label{lemma:rhat-prob-mixture}
  $\Asig$, $\Asig'$, and $\Asig''$ satisfy the following:
  \begin{enumerate}
    \item $\Asig$ has $\eit{i}{t} \sim \mathcal{N}(0,\sigma^2\Id_d)$ independently, i.e. is the Gaussian $\sigma^2$-perturbed adversary.
    \item $\Asig$ is a mixture distribution of $\Asig'$ and $\Asig''$; furthermore, the probability of $\Asig'$ in this mixture is at least $1 - \delta$ for $\hat{R} \geq \sigma \sqrt{2\ln(Tkd/\delta)}$.
    \item Under $\Asig'$, at each time step $t$, each coordinate of $Q^t \eit{i}{t}$ is distributed independently as a Normal$(0,\sigma^2)$ variable truncated to $[-\hat{R},\hat{R}]$.
  \end{enumerate}
\end{lemma}
\iffull
\begin{proof}
  (1) follows immediately from rotational invariance of Gaussians, i.e. if $Q^t$ is an orthonormal change-of-basis matrix and $\tempit{i}{t} \sim \mathcal{N}(0,\sigma^2\Id_d)$, then $\eit{i}{t} = \left(Q^t\right)^{-1}\tempit{i}{t}$ has the same distribution.

  For (2): the fact that it is a mixture follows from the fact that $\tempit{}{}$ is a mixture distribution of $\temppit{}{}$ and $\tempppit{}{}$.
  The probability of $\Asig''$ in the mixture is the chance that there exists some $i,t,j$ where $\left|(\tempit{i}{t})_j\right| > \hat{R}$.
  Each $\mathcal{N}(0,\sigma^2)$ variable $(\tempit{i}{t})_j$ in particular is symmetric and $\sigma^2$-subgaussian, so
    \[ \pr{ |(\tempit{i}{t})_j| \geq \hat{R}} \leq 2e^{-\hat{R}^2 / (2\sigma^2)}. \]
  A union bound over the $d$ coordinates, $k$ arms, and $T$ time steps gives the result.

  (3) follows directly from the construction of $\Asig'$, i.e. $\temppit{i}{t}$ has the stated property and $\eit{i}{t} = \left(Q^t\right)^{-1} \temppit{i}{t}$.
\end{proof}\fi

Our next lemma states that $\Asig'$ is bounded and centrally bounded.
Its proof follows from analyzing the Euclidean norm of a Gaussian random
vector when conditioning on an upper bound in each coordinate.

\newcommand{\reit}[2]{\epsilon_{{#1}}^{{#2}}}
\begin{lemma} \label{lemma:asigprime-bounded} For any choice of
  $\hat{R}$, $\Asig'$ is $(1+\sqrt{d}\hat{R}, 0)$-bounded and
  $(r,\frac{1}{T})$-centrally bounded for
  $r \geq \sigma \sqrt{2\ln(T)}$.
\end{lemma}

The next lemma states that the truncated Gaussian-perturbed adversary
is diverse. The proof uses an orthonormal change of basis for the
perturbations, which maintains the perturbation's Gaussian
distribution but allows the form of the conditioning to be
simplified. We then lower-bound the variance of the truncated Gaussian
perturbations.

\begin{lemma} \label{lemma:singleparam-bounded-adv-diverse}
  $\Asig'$ satisfies $(r,\lambda_0)$ diversity for $\lambda_0 = \Omega(\sigma^4 / r^2)$
  when choosing $\hat{R} \geq 2r$ and $r \geq \sigma$.
\end{lemma}
\iffull
\begin{proof}
  Recall that for each $i=1,\ldots,k$, we have
  $\pxit{i}{t} = \uxit{i}{t} + \reit{i}{t}$ where $\reit{i}{t}$ is drawn
  from a special form of truncated Gaussian, the exact form of the
  truncation depending on previous time steps.  In particular, fix a
  time step $t$, write $\bh$ as shorthand for $\bhit{}{t}$ and $\px$ as
  shorthand for the context selected by Greedy, with
  $\px = \ux + \ei{}$.  Let $Q$ be the orthonormal matrix such that
  $Q\bh = (\|\bh\|, 0, \ldots, 0)$.

  Let $b \leq r \|\bh\|$ and take all probabilities conditioned on previous time steps and $\bh,\uxit{1}{t},\ldots,\uxit{k}{t}$:
  Using the ``variational characterization'' of eigenvalues, the minimum eigenvalue is
  \begin{align*}
    \lmin\left(\E{ \px \trans{\px} \given \bh \cdot \ei{} \geq b } \right)
      &=    \min_{w: \|w\| = 1} \trans{w} \left(\E{ \px \trans{\px} \given \bh \cdot \ei{} \geq b } \right) w  \\
      &=    \min_{\|w\|=1} \E{ \trans{w} \px \trans{\px} w \given \bh \cdot \ei{} \geq b } \\
      &=    \min_{\|w\|=1} \E{ (w \cdot \px)^2 \given \bh \cdot \ei{} \geq b } \\
      &\geq \min_{\|w\|=1} \Var \left[ w \cdot \px \given \bh \cdot \ei{} \geq b \right]  \\
      &=    \min_{\|w\|=1} \Var \left[ Qw \cdot Q\px \given Q\bh \cdot Q\ei{} \geq b \right]  \\
      &=    \min_{\|w\|=1} \Var \left[ Qw \cdot Q\px \given \|\bh\| (Q\ei{})_1 \geq b \right]  \\
      &=    \min_{\|w\|=1} \Var \left[ w \cdot Q\px \given (Q\ei{})_1 \geq r' \right]
  \end{align*}
  for $r' = \frac{b}{\| \bh\|}\leq r$, where the last line uses that minimizing over $w$ and over $Qw$ yield the same result.
  Note that $w \cdot Q\px = w \cdot Q\ux + w \cdot Q\ei{}$, so the variance is equal to $\sum_{j=1}^d \Var(w_j (Q\ei{})_j)$.
  Also, recall that by definition of $\Asig'$, each $(Q\ei{})_j$ is distributed independently as a Gaussian conditioned on $|(Q\ei{})_j| \leq \hat{R}$.
  So, if we let $\reit{}{} \sim \mathcal{N}(0,\sigma^2 \Id_d)$, then we may write
  \begin{align*}
    \lmin\left(\E{ x \trans{x} \given \bh \cdot \eit{}{} \geq b } \right)
      &\geq \min_{\|w\|=1} w_1^2 \Var\left({\reit{}{}}_1 \given r' \leq {\reit{}{}}_1 \leq \hat{R}\right) + \sum_{j=2}^d w_j^2 \Var\left({\reit{}{}}_j \given {\reit{}{}}_j \leq \hat{R}\right) .
  \end{align*}
  Now, by Lemma \ref{lemma:double-truncated-gaussian-variance}, we have for $\hat{R} \geq 2r'$ that
    \[ \Var\left((Q\ei{})_1 \given r' \leq (Q\ei{})_1 \leq \hat{R}\right) \geq \Omega\left(\frac{\sigma^4}{(r')^2}\right) \]
  while by Lemma \ref{lemma:upper-truncated-gaussian-variance},
    \[ \Var\left((Q\ei{})_j \given (Q\ei{})_j \leq \hat{R}\right) \geq \Omega\left(\sigma^2\right) . \]
  We get a worst-case bound of $\Omega(\sigma^4 / r^2)$.
  (This uses that by construction $r \geq \sigma$, so that this bound is smaller than $\Omega(\sigma^2)$.)
\end{proof}\fi

Because the Gaussian-perturbed truncated adversary is diverse and
centrally bounded, our framework in the form of Theorem
\ref{thm:singleparam-perturbed-regret} implies the greedy algorithm
facing the truncated perturbed adversary has diminishing regret.  We
find that our regret bound has two regimes, which correspond to
``large'' and ``small'' perturbations.  In the large-perturbations
regime, $R = \Omega(1)$, i.e. the perturbations $\eit{i}{t}$ are
generally larger than the underlying means $\uxit{i}{t}$.  We find
this less natural and well motivated than the small-perturbations
regime, where perturbations are small relative to $\uxit{i}{t}$ and we
are interested in the growth in regret as $\sigma \to 0$.  But the
regret bounds are reasonable for both regimes and in particular have a
$\tilde{O}(\sqrt{T})$ dependence on the time horizon.

\begin{lemma} \label{lemma:singleparam-bounded-perturbed-regret} Let
  $r = \sigma \sqrt{2\ln(T)}$ and
  $\hat{R} = 2\sigma \sqrt{2\ln(Tkd/\delta)}$ and
  consider the bounded
  perturbed adversary $\Asig'$ with this choice of $\hat{R}$.
  With probability at least $1-\frac{\delta}{2}$,
  for fixed $s$ and $k$ and $\sigma \leq O\left((\sqrt{d \ln(Tkd/\delta)})^{-1}\right)$, Greedy has
    \[  \text{Regret} \leq O\left( \frac{\sqrt{T d s} \left(\ln\frac{Td}{\delta}\right)^{3/2}} {\sigma^2} \right). \]
\end{lemma}

Finally, we conclude that greedy has low regret with respect to the
original (untruncated) perturbed adversary.

\begin{proof}[Proof of Theorem~\ref{thm:singleparam-perturbed-regret}]
  As described above, we view $\Asig$ as a mixture distribution over
  two adversaries, one of which is $\Asig'$.  Let
  $r = \sigma \sqrt{2 \ln(T)}$ and $\hat{R} = 2\sigma\sqrt{2\ln(Tkd/\delta)}$.
  By Lemma
  \ref{lemma:rhat-prob-mixture}, the probability of $\Asig$ choosing
  the bounded adversary $\Asig'$ is at least $1 - \frac{\delta}{2}$
  when choosing $\hat{R} \geq \sigma \sqrt{2\ln(2Tkd/\delta)}$, which
  our choice of $\hat{R}$ satisfies.
  Lemma \ref{lemma:singleparam-bounded-perturbed-regret} gives a regret bound,
  conditioned on facing $\Asig'$, with probability
  $1-\frac{\delta}{2}$.
  By a union bound, these regret bounds hold for facing $\Asig$ with
  probability $1-\delta$.
\end{proof}

\begin{remark}
  An expected regret bound can again be obtained directly from the high-probability bound by taking e.g $\delta = \frac{1}{T}$.
  There is a slight twist: The norms of the contexts are not bounded when facing $\Asig''$, which occurs with some small probability at most $\frac{\delta}{2}$.
  However, expected regret from this case is still bounded by $2T$ because, at each time step, the difference in \emph{expectation} between any two choices $i,i'$ is at most $2$ (using spherical symmetry of $\Asig''$).
\end{remark}
\fi

\section{Multiple Parameter Setting}\label{sec:multi}
In the multiple parameter setting, recall that each arm $i\in [k]$ has
an unknown true parameter $\bi{i}$ and the goal is to have low regret
compared to the algorithm that picks the largest $\bi{i}\pxit{i}{t}$
at each round $t$.  Here, we cannot hope for the greedy algorithm to
achieve vanishing regret without any initial information, as it can
never learn about parameters of arms it does not pull \iffull(we
formalize this with a lower bound in Section \ref{sec:lb}).\fi 
However, we can show that in the presence of perturbations, it
suffices to have a small amount of initial information in the form of
$n$ samples $(\pxi{i},\ri{i})$ for each arm $i$.
We refer to this as an $n$-sample ``warm
start'' to Greedy. \iffull The full algorithm is presented in
Algorithm \ref{alg:greedy-multiple}.

\begin{algorithm}
  \begin{algorithmic}
    \STATE Start with $n$ initial observations for each arm: $\left(\mxit{i}{1}, \yit{i}{1} : i=1,\dots,k\right)$.
    \FOR{$t = 1$ to $T$}
      \STATE Define $\bhit{i}{t} = \arg\min_{\b} ||\mxit{i}{t}\b - \yit{i}{t}||_2^2$ for all $i=1,\dots,k$. \\
             (When the covariance matrix is invertible, this is: $\bhit{i}{t} = \left(\XTXi\right)^{-1} \trans{(\mxit{i}{t})} \yit{i}{t}$.)
      \STATE Observe contexts $\pxit{1}{t}, \dots, \pxit{k}{t}$.
      \STATE Choose arm $i^t = \arg\max \bhit{i}{t} \cdot \pxit{i}{t}$ and observe reward $\rit{i^t}{t}$.
      \STATE Append the new observations $\pxit{i^t}{t},\rit{i^t}{t}$ to $(\mxit{i^t}{t+1},\yit{i^t}{t+1})$, and for arm $j \neq i^t$, let $(\mxit{j}{t+1},\yit{j}{t+1})=(\mxit{j}{t},\yit{j}{t})$.
    \ENDFOR
  \end{algorithmic}
  \caption{Greedy (multiple parameter)} \label{alg:greedy-multiple}
\end{algorithm}
\fi

As discussed in Section \ref{subsec:proof-approach}, the key idea is
as follows.  Analogous to the single parameter setting, the diversity
condition implies that additional datapoints we collect for an arm
improve the accuracy of the estimate $\bhit{i}{t}$.  Meanwhile, the
\emph{margin condition} implies that for sufficiently accurate
estimates, when an arm is optimal ($\bi{i}\pxit{i}{t}$ is largest),
the perturbations have a good chance of causing Greedy to pull that
arm ($\bhit{i}{t}\pxit{i}{t}$ is largest).  Thus, the initial data
sample kickstarts Greedy with reasonably accurate estimates, causing
it to regularly pull optimal arms and accrue more data points, thus
becoming more accurate.

\paragraph{Notation and preliminaries.}
Let $t_i(t)$ be the number of rounds arm $i$ is pulled prior to round
$t$, including the warm start.  Let $S_i$ be the set of rounds $t$
such that arm $i$ is pulled by Greedy, and let $S^*_i$ be the set of
rounds in which arm $i$ has the highest reward.

Recall that in the single parameter setting, the definitions of
``good'' and ``auspicious'' captured rounds where perturbations are
not be too extreme.\jm{noise or perturbations? Check. Swapped to
  perturbation, not sure it's right but think it is} We introduce
their multi-parameter analogues below.

Fix a round $t$, the current Greedy hypotheses
$\bhit{1}{t},\dots,\bhit{k}{t}$, and choices of an adversary
$\uxit{1}{t},\dots,\uxit{k}{t}$.  We now define the ``threshold''
$\chit{i}{t} \defeq \max_{j\neq i} \bhit{j}{t} \pxit{j}{t} ,$ a random
variable depending on $\{\eit{j}{t} : j \neq i\}$.  We say an outcome
of $\chit{i}{t}$ is \emph{$r$-\hgood (for arm $i$)} if
$\chit{i}{t} \leq \bhit{i}{t} \uxit{i}{t} + r\|\bhit{i}{t}\|.$ We say
the collection
$(\bhit{1}{t},\uxit{1}{t},\dots,\bhit{k}{t},\uxit{k}{t})$ is
\emph{$r$-\haus for $i$} if $\Pr{\eit{1}{t},\dots,\eit{k}{t}}{\text{$\chit{i}{t}$ is $r$-\hgood}
  \given i^t = i} \geq \frac{1}{2}$.  Again we
shorten this to saying ``round $t$ is $r$-\haus''.

We will also need corresponding definitions for capturing when arm $i$
is actually the optimal arm.  These are analogous, but replace
$\bhit{i}{t}$ with $\bi{i}$.  Define the ``threshold''
$\cit{i}{t} \defeq \max_{j\neq i} \bi{j} \pxit{j}{t}$ and say an
outcome of $\cit{i}{t}$ is \emph{$r$-good (for arm $i$)} if
$\cit{i}{t} \leq \bi{i}\uxit{i}{t} + r\|\bi{i}\|.$ Say round $t$ is
\emph{$r$-auspicious for $i$} if $\pr{\text{$\cit{i}{t}$ is $r$-good}
  \given i^t = i} \geq \frac{1}{2} .$

\iffull\subsection{Regret framework for perturbed adversaries}
\fi
Similarly to Lemma \ref{lemma:singleparam-greedy-regret}, the regret
of Greedy shrinks as each $\bhit{i}{t} \to \bi{i}$.  The proof is
essentially identical, but in this case, we prove this for each arm $i\in[k]$.
\begin{lemma} \label{lemma:multiparam-greedy-regret} In the multiple
  parameter setting, the regret of Greedy is bounded by
  $\sum_{i=1}^k \textrm{Regret}_i(T)$ with
    \[ \textrm{Regret}_i(T) = R \left(\sum_{t\in S_i} \left\|\bi{i} - \bhit{i}{t} \right\| \right) + R \left(\sum_{t\in S_i^*} \left\| \bi{i} - \bhit{i}{t} \right\| \right) . \]
\end{lemma}


\subsection{Diversity condition and convergence}

We now show that with enough observations, the diversity condition
implies that estimates converge to the true parameters.  The only
difference from the results in Section \ref{sec:greedy} is that now,
these results will refer to particular arms' parameters
$\bhit{i}{t} \to \bi{i}$, and will depend on the number of
observations $t_i(t)$ from those arms.  For some analogous claims, the
proofs are identical but for these notational changes, and are
omitted.  \bo{I should probably write them out....}

\iffull
\begin{lemma}[Analogue of Lemma \ref{lemma:good-diversity}] \label{lemma:multiple-good-diversity}
  Against a perturbed adversary satisfying $(r,\lambda_0)$ diversity, for all $i,t$, we have
    \[ \lmin \left( \E{\trans{(\pxit{i}{t})}\pxit{i}{t} \given i^t = i, \chit{i}{t} \mbox{ is }r\mbox{-\hgood} }\right) \geq \lambda_0 . \]
  \end{lemma}

\begin{lemma}[Analogue of \Cref{cor:diverse-all-converge}]
 \label{lemma:multiple-diversity-convergence}
 Consider Greedy in the multiple parameter setting with an
 $R$-bounded, $(r,\lambda_0)$-diverse perturbed adversary.  Fix
 $i\in[k]$, $t_i(t)$ the number of data points collected for
 $i$ before round $t$.  Let
 $t_{\min}(\delta) := \max\left\{32\ln(4/\delta) ~,~ \frac{80R^2 \ln(2d/\delta)}{\lambda_0}\right\}$ and fix
 a particular $t$ with $t_i(t) \geq t_{\min}(\delta)$.  If at least
 $\frac{t_i(t)}{2}$ of the rounds $t' \leq t$ for which arm $i$ was
 pulled are $r$-\haus for $i$, then with probability at least
 $1-\delta$,
    \[ \| \bi{i} - \bhit{i}{t} \| \leq \frac{32\sqrt{2Rds \ln(2t_i(t)d/\delta)}}{\lambda_0 \sqrt{t_i(t)}} . \]
\end{lemma}

\fi  

It will be helpful to introduce some notation for a minimum number of
samples (i.e. pulls of an arm) that suffice to apply concentration
results and show that estimates $\bhit{i}{t}$ are ``accurate'',
i.e. close to $\bi{i}$.  Let
\[ n_{\min}(\delta,R,d,k,\lambda_0) \defeq \max\left\{ 128\ln\left(\frac{192k}{\delta}\right) ~,~ \frac{320R^2
  \ln\left(320R^2 d k / \delta\right)}{\lambda_0} \right\} . 
 \]
All parameters are generally clear from
context and fixed constants dependent only on the instance, except for
$\delta$, which is a parameter the analyst may vary, so for shorthand
we will write $n_{\min}(\delta)$.

\begin{lemma} \label{lemma:multiple-diversity-all-converge-careful}
  Consider an $R$-bounded, $(r,\lambda_0)$-diverse perturbed
  adversary.  Suppose, for each $i$, at most
  $\frac{n_{\min}(\delta)}{2}$ rounds where Greedy pulls $i$ are not
  $r$-\haus for $i$.  Then with probability at least $1-\delta$,
  $\forall i, t$ with $t_i(t) \geq n_{\min}(\delta)$,
    \[ \|\bi{i} - \bhit{i}{t}\| \leq \frac{32 \sqrt{6Rds \ln(2t_i(t)dk/\delta)}}{\lambda_0 \sqrt{t_i(t)}} . \]
\end{lemma}
\iffull\begin{proof} Fix an arm $i$.  For each $t_i(t)$, we apply
  Lemma \ref{lemma:multiple-diversity-convergence} with
  $\delta(t) = \frac{\delta}{k t_i(t)^2}\frac{6}{\pi^2}$.  Note that
  $\sum_{j : t_i(t) = j} \delta(t) \leq \frac{\delta}{k}$, so a union
  bound over time steps and arms give a $\delta$ probability of
  failure.

  There are two steps to applying Lemma
  \ref{lemma:multiple-diversity-convergence}.  First, we must satisfy
  the assumptions of the lemma by showing that for all $t$,
  $t_i(t) \geq t_{\min}(\delta(t))$.  Second, we apply the guarantee of
  the lemma by plugging in $\delta(t)$.

  For the first, consider the two cases of $t_{\min}(\delta(t))$ separately.
  If $t_{\min}(\delta(t)) = 32\ln(4/\delta(t))$, then it suffices to set $t_i(t)$ to at least
  \begin{align*}
    32\ln(4/\delta(t))
      &=    32\ln\frac{2 k t_i(t)^2 \pi^2}{3\delta}  \\
      &\leq 32\ln\frac{3^2 k^2 t_i(t)^2}{\delta^2}  \\
      &=    64\ln\frac{3kt_i(t)}{\delta} .
  \end{align*}
  Let $A = 64$ and $B = \frac{3k}{\delta}$, then by Lemma \ref{lemma:log-bound}, it suffices for $t_i(t) \geq 2A \ln(AB) = 128 \ln\frac{192k}{\delta}$.
  For the other case of $t_{\min}(\delta(t))$, we have
  \begin{align}
    t_{\min}(\delta(t))
      &=    \frac{80R^2\ln(2d/\delta)}{\lambda_0}  \nonumber \\
      &=    \frac{80R^2 \ln\left(2t_i(t)^2dk \pi^2 / 6\delta\right)}{\lambda_0}  \nonumber  \\
      &\leq \frac{160R^2 \ln\left(2 t_i(t) d k / \delta\right)}{\lambda_0}  \label{eqn:tmin-case-multi}
  \end{align}
  where we used that
  \begin{align*}
    \ln \frac{2t_i(t)^2dk \pi^2} {6\delta}
      &\leq \ln\left( \frac{t_i(t)^2 d^2 k^2}{\delta^2} \frac{\pi^2}{3} \right)  \\
      &\leq 2\ln \left( \frac{t_i(t) d k}{\delta} (2) \right).
  \end{align*}
  Let $A = \frac{160 R^2}{\lambda_0}$ and $B = \frac{2dk}{\delta}$.
  Then for $t_i(t)$ to exceed (\ref{eqn:tmin-case-multi}), we require $t_i(t) \geq A \ln(Bt_i(t))$, which by Lemma \ref{lemma:log-bound} holds for all $t_i(t) \geq 2A \ln(AB)$ (using that $t_i(t) \geq 1$, $A \geq 0$, and $B \geq e$).
  So it suffices for
    \[ t_i(t) \geq \max \begin{cases} 128 \ln\left(\frac{192k}{\delta}\right) \\
                                      \frac{320 R^2 \ln\left(320 R^2 d k / \delta\right)}{\lambda_0} .  \end{cases} \]
  In particular, we set the warm start size $n$ equal to the right hand side, ensuring that the inequality holds for all $i,t$.

  For the second step, we plug in $\delta(t)$ to Lemma \ref{lemma:multiple-diversity-convergence}:
  \begin{align*}
    \| \bi{i} - \bhit{i}{t} \|
      &\leq \frac{32 \sqrt{2Rds \ln(2t_i(t)dk/\delta(t)}}{\lambda_0 \sqrt{t_i(t)}}  \\
      &=    \frac{32 \sqrt{2Rds \ln(2t_i(t)^3dk\pi^2/6\delta}}{\lambda_0 \sqrt{t_i(t)}}  \\
      &=    \frac{32 \sqrt{2Rds \ln(t_i(t)^3dk\pi^2/3\delta}}{\lambda_0 \sqrt{t_i(t)}}  \\
      &\leq \frac{32 \sqrt{6Rds \ln(2t_i(t)dk/\delta}}{\lambda_0 \sqrt{t_i(t)}}.
  \end{align*}
  In the last inequality, we used:
  \begin{align*}
    \ln\frac{t_i(t)^3dk\pi^2} {3\delta}
      &\leq \ln\left( \frac{t_i(t)^3d^3k^3}{\delta^3} \frac{\pi^2} {3} \right)  \\
      &=    3\ln\left( \frac{t_i(t)dk}{\delta} \frac{\pi^{2/3}}{3^{1/3}} \right)  \\
      &\leq 3\ln \frac{2t_i(t)dk}{\delta}.  \qedhere
  \end{align*}
\end{proof}
\fi


\begin{lemma}\label{lemma:multiple-half-ausp}
  \jm{removed analogue statement, we didn't have it for the previous
    lemma} If the perturbed adversary is
  $(r,\frac{1}{T})$-centrally-bounded, then with probability at least
  $1-\delta$, for each arm $i$, all but
  $2 + \sqrt{\frac{1}{2}\ln\frac{k}{\delta}}$ rounds in which
  $i^t = i$ are $r$-\haus for $i$.
\end{lemma}

Combining these implies for a large enough warm start, the estimates are within a small enough margin.
\begin{lemma} \label{lemma:multi-warm-convergence}
  \bo{New, key}
\iffalse  In the multiple parameter setting with an $R$-bounded, $(r,\tfrac{1}{T})$ centrally bounded, $(r,\lambda_0)$-diverse perturbed adversary, with probability $1-\delta$, for all arms $i$ and rounds $t$ with
\else
For an $R$-bounded, $(r,\tfrac{1}{T})$ centrally bounded, $(r,\lambda_0)$-diverse adversary, with probability $1-\delta$,  if
\fi
    \[ t_i(t) \geq n^* \defeq \max\begin{cases}
         4 + \sqrt{2\ln(2k/\delta)}  \\
         n_{\min}(\delta/2) \\
         \frac{49152 Rds}{\left(\alpha \lambda_0 \min_j \|\bi{j}\|\right)^2} \ln\left(\frac{98304Rd^2 k s}{\delta \left(\alpha \lambda_0 \min_j \|\bi{j}\|\right)^2}\right) , \end{cases}  \]
  we have
  \begin{equation}
    \|\bi{i} - \bhit{i}{t}\| \leq \min\left\{ \frac{32 \sqrt{6Rds \ln(4t_i(t)dk/\delta)}}{\lambda_0 \sqrt{t_i(t)}} ~,~ \frac{\alpha \min_j \|\bi{j}\|}{2R} \right\} .  \label{eqn:multi-bi-bound}
  \end{equation}
\end{lemma}

Each term in $n^*$ corresponds to a different requirement.  The first
ensures many $r$-\haus rounds for each $i$; the next two allow Lemma
\ref{lemma:multiple-diversity-all-converge-careful} to conclude
respectively each inequality in (\ref{eqn:multi-bi-bound}).  These
inequalities have the following purposes: the first to bound regret by summing
over all rounds; and the second to apply the margin condition, by ensuring that
all estimates are always sufficiently accurate.

\iffull
\begin{proof}[Proof of Lemma \ref{lemma:multi-warm-convergence}]
  By Lemma \ref{lemma:multiple-half-ausp}, with probability $1-\frac{\delta}{2}$, each arm $i$ has at most $\frac{n^*}{2}$ rounds in which $i^t = i$ but the round is not $r$-\haus for $i$ (using that $t_i(t)$ exceeds the first case in the definition of $n^*$).
  We can therefore obtain the first case of (\ref{eqn:multi-bi-bound}) by applying Lemma \ref{lemma:multiple-diversity-all-converge-careful}, which uses that $t_i(t)$ exceeds the second case in the definition of $n^*$: With probability $1-\frac{\delta}{2}$, for all $i$, for each $t$ with $t_i(t) \geq n^*$, $\|\bi{i} - \bhit{i}{t}\|$ is at most the first case of (\ref{eqn:multi-bi-bound}).

  Finally, we argue that the second case of (\ref{eqn:multi-bi-bound}) holds if $t_i(t)$ exceeds the third case in the definition of $n^*$.
  That is, we wish to show that
    \[ \frac{32 \sqrt{6Rds \ln(4t_i(t)dk/\delta)}}{\lambda_0 \sqrt{t_i(t)}} \leq \frac{\alpha \min_j \|\bi{j}\|}{2R} . \]
  Let $B = \frac{4dk}{\delta}$, $C = \frac{32}{\lambda_0}$, $D = 6Rds$, and $E = \alpha \min_j \|\bi{j}\| / 2$.
  We wish to show
    \[ \frac{C\sqrt{D\ln(Bt_i(t))}}{\sqrt{t_i(t)}} \leq E . \]
  This is equivalent to
    \[ \frac{t_i(t)}{\ln(Bt_i(t))} \geq \frac{D C^2}{E^2} . \]
  If we let $A = \frac{D C^2}{E^2} = \frac{24576 R d s}{\lambda_0^2 \alpha^2 \min_j \|\bi{j}\|^2}$, then by Lemma \ref{lemma:log-bound}, this is satisfied for all $t_i(t) \geq n \defeq 2A \ln(AB)$.
  So, plugging back in, it suffices for $t_i(t)$ to exceed the third case in the definition of $n^*$.
\end{proof}
\fi

\subsection{Margin condition and warm start}
Here, we wish to capture the benefits of the margin condition,
i.e. that arms which are often optimal are also actually pulled often
by Greedy.  Lemma \ref{lemma:good-margins-pull} translates the margins
condition more directly into our setting, saying that when arm $i$ is
optimal (and $\cit{i}{t}$ is $r$-good), it is optimal by a significant
margin ($\alpha \|\bi{i}\|$) with a significant probability
($\gamma$).
\begin{lemma} \label{lemma:good-margins-pull}
  Against a perturbed adversary satisfying $(r,\alpha,\gamma)$ margins, for all $i,t$ we have
    \[ \pr{\bi{i} \pxit{i}{t} > \cit{i}{t} + \alpha \|\bi{i}\| \given \bi{i} \pxit{i}{t} \geq \cit{i}{t}, \text{$\cit{i}{t}$ is $r$-good}} \geq \gamma . \]
\end{lemma}
\iffull
\begin{proof}
  The margin condition, applied to $\bi{i}$, says that for any $b \leq r\|\bi{i}\|$,
    \[ \pr{\bi{i} \eit{i}{t} > b + \alpha \|\bi{i}\| \given \bi{i}\cdot \eit{i}{t} \geq b} \geq \gamma . \]
  Fix $c \defeq \cit{i}{t}$ which is $r$-good, i.e. $c \leq \bi{i}\uxit{i}{t} + r\|\bi{i}\}$.
  Let $b \defeq \cit{i}{t} - \bi{i}\uxit{i}{t}$.
  Then we have $b \leq r\|\bi{i}\|$.
  And the condition $\bi{i}\cdot \eit{i}{t} \geq b$ is equivalent to $\bi{i}\cdot \pxit{i}{t} \geq c$, so we get
    \[ \pr{\bi{i} \pxit{i}{t} > c + \alpha \|\bi{i}\| \given \bi{i} \pxit{i}{t} \geq c, ~ \cit{i}{t} = c} \geq \gamma . \]
  Since this inequality holds for every $r$-good realization $c$, it holds when conditioning on the event that $\cit{i}{t}$ is $r$-good.
\end{proof}
\fi

\iffull Next, Lemma \ref{lemma:margins-chance-pull} shows that, if
estimates $\bhit{i}{t}$ are accurate, this implies that arm $i$ is
actually pulled with significant probability when it is optimal.
\begin{lemma} \label{lemma:margins-chance-pull}
  Suppose the perturbed adversary is $R$-bounded and has $(r,\alpha,\gamma)$ margins for
  some $r \leq R$.  Consider any round $t$ where for all $j$ we have
  $\| \bi{j} - \bhit{j}{t}\| \leq \frac{\alpha \min_{j'}
    \|\bi{j'}\|}{2R}$.  Then conditioned on the event that arm $i$ is
  optimal and $\cit{i}{t}$ is $r$-good for $i$, arm $i$ is pulled with
  probability at least $\gamma$ independently of all other rounds.
\end{lemma}
\begin{proof}
  For convenience, let $A = \alpha \min_{j'} \|\bi{j'}\|$.
  First, note that for all $j$,
  \begin{align}
    \left| \bi{j} \pxit{j}{t} - \bhit{j}{t} \pxit{j}{t} \right|
      &=    \left| \left(\bi{j} - \bhit{j}{t}\right) \pxit{j}{t} \right|  \nonumber \\
      &\leq \|\bi{j} - \bhit{j}{t}\| ~ \|\pxit{j}{t}\|  \nonumber \\
      &\leq \frac{\alpha \min_{j'} \|\bi{j'}\|}{2R} R  & \text{(assumptions)} \nonumber \\
      &= \frac{A}{2} .  \label{eqn:betadiff-ineq}
  \end{align}
  By Lemma \ref{lemma:good-margins-pull},
  \begin{equation}
    \pr{\bi{i} \pxit{i}{t} > \cit{i}{t} + \alpha \|\bi{i}\| \given \bi{i} \pxit{i}{t} \geq \cit{i}{t}, \text{$\cit{i}{t}$ is $r$-good}} \geq \gamma . \label{eqn:beta-good-margin}
  \end{equation}
  Note this conditions on the events that $i$ is optimal and $\cit{i}{t}$ is $r$-good.
  So we have by (\ref{eqn:betadiff-ineq}) that $\bhit{i}{t} \pxit{i}{t} \geq \bi{i} \pxit{i}{t} - \frac{A}{2}$, which implies by (\ref{eqn:beta-good-margin}) that with probability $\gamma$,
  \begin{align*}
    \bhit{i}{t}\pxit{i}{t}
      &>    \cit{i}{t} + \alpha\|\bi{i}\| - \frac{A}{2}.
  \end{align*}
  By definition of $A$, $\alpha\|\bi{i}\| \geq A$, so this implies
  \begin{align*}
    \bhit{i}{t}\pxit{i}{t}
      &>    \cit{i}{t} + \frac{A}{2}  \\
      &=    \max_{j \neq i} \bi{j}{t} \pxit{j}{t} + \frac{A}{2}  \\
      &\geq \max_{j \neq i} \bhit{j}{t} \pxit{j}{t}  & \text{(using (\ref{eqn:betadiff-ineq})).}
  \end{align*}
  This implies that arm $i$ is pulled by Greedy.
\end{proof}
\fi

Lemma \ref{lemma:margins-implies-pull} shows that if it has sufficiently accurate estimates
$\bhit{i}{t}$, Greedy will play $i$ a number of times that can be related to the number of rounds for which $i$ was optimal.
Recall that $S_i, S^*_i$ is the set of rounds in which $i^t = i$
(Greedy pulls arm $i$) and $i^*(t) = i$ (arm $i$ is optimal),
respectively.
\begin{lemma} \label{lemma:margins-implies-pull} Consider an
  $R$-bounded perturbed adversary with $(r,\alpha,\gamma)$ margins and
  assume
  $\|\bi{i} - \bhit{i}{t}\| \leq \tfrac{\alpha \min_j \|\bi{j}\|}{2R}$
  for all $i$ and $t$.  Suppose that for all but $C$ rounds
  $t\in S^*_i$, $t$ is $r$-auspicious for $i$. Let $S_i(t), S^*_i(t)$
  denote the rounds up to and including round $t$ in which $i$ is
  selected by Greedy and is optimal, respectively.

  With probability at least $1-\delta$, for all $t\in S^*_i$, if there
  exists some $t' > t , t'\in S^*_i$ such that
  \[|S^*_i(t') \setminus S^*_i(t)| >  C +
  \frac{2}{\gamma}\ln\frac{Tk}{\delta},\]
  then there exists some
  $t'' \in \left(S^*_i(t') \setminus S^*_i(t)\right) \cap
  \left(S_i(t')\setminus S_i(t)\right)$.
That is, arm $i$ is optimal at most
$C + \frac{2}{\gamma}\ln\frac{Tk}{\delta}$ rounds between being pulled
by Greedy and being optimal.
\end{lemma}
\iffull
\begin{proof}
  Fix an arm $i$. In each round $t$ where $i$ is optimal and
  $t$ is $r$-auspicious for $i$, $\cit{i}{t}$ is $r$-good independently with probability at least
  $\frac{1}{2}$. So by Lemma \ref{lemma:margins-chance-pull}, $i$ is
  pulled with probability at least $\frac{\gamma}{2}$ independently.
  There may be up to $C$ rounds in which $i$ is optimal but not
  auspicious.  Therefore, the chance of $C + z$ rounds occurring where
  $i$ is optimal without being pulled is at most
  $(1-\frac{\gamma}{2})^z \leq e^{-\gamma z / 2}$.  For
  $z = \frac{2}{\gamma}\ln\frac{Tk}{\delta}$, this bound is
  $\frac{\delta}{Tk}$.  A union bound over each possible start of such
  a sequence $t$ and arm $i$ gives the result.
\end{proof}
\fi

We also need the fact that most rounds in which $i$ is optimal are auspicious for $i$.
Note this is identical to Lemma \ref{lemma:multiple-half-ausp}, except that it applies to auspicious rounds and optimal arms rather than \haus rounds and Greedy's choice of arms.
The proof is the same except with syntactic changes.
\begin{lemma} \label{lemma:multi-opt-ausp} If the perturbed adversary
  is $(r,\tfrac{1}{T})$-centrally bounded, then with probability at
  least $1-\delta$, for all arms $i$, all but
  $2 + \sqrt{\tfrac{1}{2}\ln\tfrac{k}{\delta}}$ rounds in which
  $i^*(t) = i$ are $r$-auspicious for $i$.
\end{lemma}

When we combine Lemmas \ref{lemma:good-margins-pull},
\ref{lemma:margins-implies-pull}, and \ref{lemma:multi-opt-ausp}, we
get the key implication of the margin condition: With high
probability, after an arm is optimal for a constant number of rounds,
it is pulled by Greedy.
\begin{corollary} \label{cor:centralbound-implies-margin-pull}
  Consider an  $R$-bounded,  $(r,\tfrac{1}{T})$-centrally bounded perturbed adversary, with $(r,\alpha,\gamma)$ margins, and assume for all $i,t$ we have $\|\bi{i} - \bhit{i}{t}\| \leq \frac{\alpha \min_j \|\bi{j}\|}{2R}$.
  Then with probability $1-\delta$, for each $j \in S_i$,
\iffull    \[ \left| \left\{ t \in S_i^* : t_i(t) = j \right\} \right| \leq \frac{5}{\gamma}\ln\frac{2Tk}{\delta}  . \]
\fi
\end{corollary}
\iffull
\begin{proof}
  Combining Lemma \ref{lemma:multi-opt-ausp} and Lemma \ref{lemma:margins-implies-pull}, plugging in a failure probability of $\delta/2$ to each lemma, we get that with probability $1-\delta$ (by a union bound),
  \begin{align*}
    \left| \left\{ t \in S_i^* : t_i(t) = j \right\} \right|
      &\leq 2 + \sqrt{\frac{1}{2}\ln\frac{2k}{\delta}} + \frac{2}{\gamma}\ln\frac{2Tk}{\delta}   \\
      &\leq 2 + \ln\frac{2k}{\delta} + \frac{2}{\gamma}\ln\frac{2Tk}{\delta}  \\
      &\leq 2 + \frac{3}{\gamma}\ln\frac{2Tk}{\delta}  \\
      &\leq \frac{5}{\gamma}\ln\frac{2Tk}{\delta} .
  \end{align*}
\end{proof}
\fi

\iffull\subsection{General result}
\fi
\begin{theorem} \label{thm:multiparam-conditions-regret} Consider an
  $R$-bounded, $(r,\frac{1}{T})$-centrally, $(r,\lambda_0)$-diverse
  perturbed adversary with $(r,\alpha,\gamma)$ margins.  If the warm
  start size $n \geq n^*(\delta/2)$, where $n^*(\delta/2)$ is defined
  in Lemma \ref{lemma:multi-warm-convergence}, then with probability
  $1-\delta$,
    \[ \text{Regret}(T) \leq \frac{192 R^{3/2} \sqrt{6Tdks} \left(\ln\frac{4Tdk}{\delta}\right)^{3/2}}{\gamma \lambda_0} . \]
\end{theorem}

\iffull
\begin{proof}
  If $n \geq n^*(\delta/2)$, then in particular for every $t$ and $i$, the number of observations $t_i(t)$ of arm $i$ up to round $t$ is at least $n^*$.
  Therefore, by Lemma \ref{lemma:multi-warm-convergence}, with probability $1-\frac{\delta}{2}$,
  we have for all $i,t$ that
  \begin{equation*}
    \|\bi{i} - \bhit{i}{t}\| \leq \frac{32 \sqrt{6Rds \ln(4t_i(t)dk/\delta)}}{\lambda_0 \sqrt{t_i(t)}}
  \end{equation*}
  and furthermore
  \begin{equation*}
    \|\bi{i} - \bhit{i}{t}\| \leq \frac{\alpha \min_j \|\bi{j}\|}{2} .
  \end{equation*}
  This implies, by Lemma \ref{cor:centralbound-implies-margin-pull}, with probability $1-\frac{\delta}{2}$, for all $j \in S_i$,
    \[ \left| \left\{ t \in S_i^* : t_i(t) = j \right\} \right| \leq \frac{5}{\gamma}\ln\frac{4Tk}{\delta}  . \]
  So regret, by Lemma \ref{lemma:multiparam-greedy-regret}, is bounded by
    \[ \sum_{i=1}^k \textrm{Regret}_i(T) \]
  where
  \begin{align*}
    \textrm{Regret}_i(T)
      &= R \left(\sum_{t\in S_i} \left\|\bi{i} - \bhit{i}{t} \right\| \right) + R \left(\sum_{t\in S_i^*} \left\| \bi{i} - \bhit{i}{t} \right\| \right)  \\
      &\leq R \sum_{t\in S_i^*} \left(1 + \frac{5}{\gamma}\ln\frac{4Tk}{\delta}\right) \left\| \bi{i} - \bhit{i}{t} \right\|  \\
      &\leq \frac{6R}{\gamma} \ln\frac{4Tk}{\delta} \sum_{t\in S_i^*} \frac{32 \sqrt{6Rds \ln(4t_i(t)dk/\delta)}}{\lambda_0 \sqrt{t_i(t)}}  \\
      &\leq \frac{6R}{\gamma} \ln\frac{4Tk}{\delta} \sum_{t\in S_i^*} \frac{32 \sqrt{6Rds \ln(4Tdk/\delta)}} {\lambda_0 \sqrt{t_i(t)}}  \\
      &\leq \frac{192 R^{3/2} \sqrt{6ds} \left(\ln\frac{4Tdk}{\delta}\right)^{3/2}}{\gamma \lambda_0} \sum_{t\in S_i^*} \frac{1}{\sqrt{t_i(t)}}  \\
      &= \frac{192 R^{3/2} \sqrt{6ds} \left(\ln\frac{4Tdk}{\delta}\right)^{3/2}}{\gamma \lambda_0} \sum_{t'=1}^{|S_i^*|} \frac{1}{\sqrt{n^* + t'}}  \\
      &\leq \frac{192 R^{3/2} \sqrt{6ds} \left(\ln\frac{4Tdk}{\delta}\right)^{3/2}}{\gamma \lambda_0} \sqrt{\left|S_i^*\right|} .
  \end{align*}
  This gives
  \begin{align*}
    \text{Regret}(T)
      &\leq \frac{192 R^{3/2} \sqrt{6ds} \left(\ln\frac{4Tdk}{\delta}\right)^{3/2}}{\gamma \lambda_0} \left(\sum_{i=1}^k \sqrt{\left|S_i^*\right|} \right)  \\
      &\leq \frac{192 R^{3/2} \sqrt{6ds} \left(\ln\frac{4Tdk}{\delta}\right)^{3/2}}{\gamma \lambda_0} \left(k \sqrt{\frac{T}{k}} \right)  \\
      &= \frac{192 R^{3/2} \sqrt{6Tdks} \left(\ln\frac{4Tdk}{\delta}\right)^{3/2}}{\gamma \lambda_0} .
  \end{align*}
  We used that $\sum_{i=1}^k |S_i^*| = T$ and concavity of $\sqrt{\cdot}$.
\end{proof}

\begin{remark}
  Once again, one can derive bounds on expected regret from these high-probability guarantees.
  However, it is worth noting that, by setting $\delta = \frac{1}{T}$ to achieve an expected-regret bound, one introduces a $\log(T)$ factor into the size of the warm start $n^*$.
  Some such dependence on $T$ is necessary for sublinear expected regret, as can be shown with a variant of our lower bound technique.\bo{Correct? Or should we make this statement more careful and weasel-y?}
  However, this result is still quite positive for settings with perturbations: With only a $\log(T)$ amount of ``exploration'', Greedy obtains sublinear expected regret even with adversarial inputs.
  This may be compared to non-perturbed settings where polynomial amounts of exploration are needed even when inputs are i.i.d.
\end{remark}
\fi

\iffull
\subsection{$\sigma$-Perturbed adversary}

In this section, we consider Gaussian perturbations and show the following:
\fi

\begin{theorem} \label{theorem:multiparam-greedy-regret}
  Consider the multiple parameter setting, against the $\sigma$-perturbed adversary $\Asig$.
  For fixed $k$ (number of arms) and $s$ (rewards' subgaussian parameter) and $\sigma \leq O\left(\left(\sqrt{d\ln(Tkd/\delta)}\right)^{-1}\right)$, with a warm start size of
            \[ n =  O\left( \frac{ds}{\sigma^{12} \min_j \|\bi{j}\|^2} \ln\left(\frac{dks}{\delta \sigma \min_j\|\bi{j}\|}\right) \right) \]
          Greedy has, with probability at least $1-\delta$,
            \[ \text{Regret}(T) \leq O\left( \frac{\sqrt{T k d s} \left( \ln \frac{T d k}{\delta}\right)^{3/2}}{\sigma^2} \right).  \]
\end{theorem}

\iffull

To show this result, we need to show that the perturbed adversary $\Asig$ satisfies the margin condition, and furthermore, that that \emph{bounded} perturbed adversary $\Asig'$ does.
We formally define $\Asig'$ and $\Asig''$ in Appendix \ref{app:multi}, with a construction exactly analogous to the single parameter setting, and show the following:

\begin{lemma} \label{lemma:gauss-margins}
  The  distribution $\mathcal{N}(0,\sigma^2\Id)$ has $(r,\frac{\sigma^2}{r},\frac{1}{20})$ margins for all $r \geq \sigma$.
\end{lemma}
\begin{proof}
  We want to show that for $\ei{} \sim \mathcal{N}(0,\sigma^2\Id)$,  for all $b \leq r \|\b\|$,
    \[ \pr{\b \ei{} > b + \alpha \|\b\| \given \b \ei{} \geq b } \geq \gamma  . \]
  $\b \ei{}$ is distributed according to $\mathcal{N}(0,\|\b\|^2\sigma^2)$, so if we let $\eta = \frac{\b \ei{}}{\|\beta\|}$, then $\eta \sim \mathcal{N}(0,\sigma^2)$.
  For some $r' \leq r$,
  \begin{align*}
    \pr{\b \ei{} > b + \alpha \|\beta\| \given \b \ei{} \geq b }
      &= \pr{\eta \geq r' + \alpha \given \eta \geq r'}  \\
      &= \frac{\pr{\eta \geq r' + \alpha}} {\pr{\eta \geq r'}}  \\
      &= \frac{1 - \Phi\left(\frac{r'+\alpha}{\sigma}\right)} {1 - \Phi\left(\frac{r'}{\sigma}\right)}.
  \end{align*}
  By Lemma \ref{lemma:conditional-gaussian-decreasing},
  this is decreasing in $r'$ and so is minimized for $r'=r$.
  We use the Gaussian tail bounds (Lemma \ref{lemma:gaussian-hazard-rate})
    \[ \frac{\phi(z)}{2z} \leq 1 - \Phi(z) \leq \frac{\phi(z)}{z} . \]
  This gives
  \begin{align*}
    \frac{1 - \Phi\left(\frac{r+\alpha}{\sigma}\right)} {1 - \Phi\left(\frac{r}{\sigma}\right)}
      &\geq \frac{\phi\left(\frac{r+\alpha}{\sigma}\right)} {\phi\left(\frac{r}{\sigma}\right)}  \frac{r}{r+\alpha} \frac{1}{2}  \\
      &\geq \exp\left[ - \frac{(r+\alpha)^2 - r^2} {2\sigma^2} \right] \frac{r}{2(r+\alpha)}  \\
      &=    \exp\left[ - \frac{2r\alpha + \alpha^2} {2\sigma^2} \right] \frac{r}{2(r+\alpha)} .
  \end{align*}
  Using $\alpha \leq r$ (which follows from $r \geq \sigma$ and $\alpha = \frac{\sigma^2}{r}$),
  \begin{align*}
    \exp\left[ - \frac{2r\alpha + \alpha^2} {2\sigma^2} \right] \frac{r}{2(r+\alpha)}  \\
      &\geq \frac{1}{4} \exp\left[ - \frac{3 r \alpha}{2\sigma^2} \right]  \\
      &\geq \frac{1}{4} e^{-\frac{3}{2}} \approx 0.05578\dots .
  \end{align*}
  for $\alpha = \frac{\sigma^2}{r}$.
\end{proof}

\begin{lemma} \label{lemma:tvdist-margins}
  Suppose $\D$ has $(r,\alpha,\gamma)$ margins, and $\text{TV}(\D,\D') \leq \theta$ where $\text{TV}$ is total variation distance.
  If $\theta \leq \frac{1}{2}\inf_{\beta \neq 0} \Pr{e\sim\D}{\beta e > r + \alpha\|\beta\|}$ then $\D'$ has $(r,\alpha, \gamma/4)$ margins.
\end{lemma}
\begin{proof}
  Recall that $\text{TV}(\D,\D') \leq \theta$ implies that, for any measurable set of outcomes $S$,
    \[ \left| \Pr{e\sim\D}{e \in S} - \Pr{e\sim\D'}{e \in S} \right| \leq \theta . \]
  To prove $\D'$ has $(r,\alpha,\gamma/4)$ margins, consider any $b \leq r\|\beta\|$ and any $\beta$.
  We have
  \begin{align}
    \Pr{e\sim\D'}{\beta e > b + \alpha\|\beta\| \given \beta e \geq b}
      &=    \frac{\Pr{e\sim\D'}{\beta e > b + \alpha\|\beta\|}} {\Pr{e\sim\D'}{\beta e \geq b}}  \nonumber \\
      &\geq \frac{\Pr{e\sim\D}{\beta e > b + \alpha\|\beta\|} - \theta} {\Pr{e\sim\D}{\beta e \geq b} + \theta} . \label{eqn:tv-theta-ineq}
  \end{align}
  Now, by assumption and using that $b \leq r$, we have the chain of inequalities
  \begin{align*}
    \theta
      &\leq \frac{1}{2} \Pr{e\sim\D}{\beta e > r + \alpha\|\beta\|}  \\
      &\leq \frac{1}{2} \Pr{e\sim\D}{\beta e > b + \alpha\|\beta\|}  \\
      &\leq \frac{1}{2} \Pr{e\sim\D}{\beta e \geq b} .
  \end{align*}
  So (\ref{eqn:tv-theta-ineq}), which is decreasing as $\theta$ increases, yields
  \begin{align*}
    \Pr{e\sim\D'}{\beta e > b + \alpha\|\beta\| \given \beta e \geq b}
      &\geq \frac{\frac{1}{2}\Pr{e\sim\D}{\beta e > b + \alpha\|\beta\|}} {2\Pr{e\sim\D}{\beta e \geq b}}  \\
      &= \frac{1}{4}\Pr{e\sim\D}{\beta e > b + \alpha\|\beta\| \given \beta e \geq b}  \\
      &\geq \frac{\gamma}{4}
 \end{align*}
 because $\D$ has $(r,\alpha,\gamma)$ margins.
\end{proof}
\begin{lemma} \label{lemma:gauss-tvdist}
  Let $\D$ be $\mathcal{N}(0,\sigma^2\Id_d)$ and $\D'$ be distributed on $\R^d$ with each coordinate i.i.d. from an $[-\hat{R},\hat{R}]$-truncated $\mathcal{N}(0,\sigma^2)$ distribution.
  (That is, density $f(z) = 0$ for $|z| > \hat{R}$ and $f(z) = f_{\sigma}(z) / (1 - 2F_{\sigma}(-|z|))$ otherwise, where $f_{\sigma}, F_{\sigma}$ are the PDF and CDF respectively of the $\mathcal{N}(0,\sigma^2)$ distribution.)
  Suppose $\hat{R} \geq r + \alpha + \sigma\sqrt{2\ln(8d)}$ and $r + \alpha \geq 2\sigma$.
  Then $\text{TV}(\D,\D') \leq \frac{1}{2}\inf_{\beta\neq 0} \Pr{e\sim\D}{\beta e > r\|\beta\| + \alpha\|\beta\|}$.
\end{lemma}

\begin{proof}
  First, we claim that $\text{TV}(\D,\D') = \Pr{e\sim\D}{\max_j |e_j| > \hat{R}}$.
  This follows because if $\D,\D'$ both have well-defined densities, and $S^*$ is the (measurable) set of outcomes where $\D'$ places a lower density, then $\text{TV}(\D,\D') = \Pr{e\sim\D}{e\in S^*} - \Pr{e\sim\D'}{e\in S^*}$.
  In our case, $\D'$ places probability $0$ on any $e$ with $\max_j |e_j| > \hat{R}$, and for any other $e$, $\D$ places a higher probability density than $\D$ (as its density is equal to that of $\D$, but up-weighted by a conditioning).
  So
  \begin{align*}
    \text{TV}(\D,\D')
      &=    \Pr{e\sim\D}{\max_j |e_j| > \hat{R}}  \\
      &\leq d\Pr{e\sim\D}{|e_1| > \hat{R}}  \\
      &=    2d\Phi\left(- \frac{\hat{R}}{\sigma}\right)
  \end{align*}
  where $\Phi$ is the standard Gaussian CDF.

  Meanwhile, we calculate $\inf_{\beta\neq 0} \Pr{e\sim\D}{\beta e > r\|\beta\| + \alpha\|\beta\|}$.
  Let $\eta \sim \mathcal{N}(0,\sigma^2)$.
  Then for $e\sim\D$, the quantity $\beta e$ is distributed as $\|\beta\|\eta$, as $\beta e = \sum_{j=1}^d \beta_j e_j$ which is Gaussian with variance $\sigma^2 \sum_{j=1}^d \beta_j^2 = \sigma^2\|\beta\|^2$.
  So we have
  \begin{align*}
    \inf_{\beta\neq 0} \Pr{e\sim\D}{\beta e > r\|\beta\| + \alpha\|\beta\|}
      &= \inf_{\beta\neq 0} \pr{\|\beta\|\eta > r\|\beta\| + \alpha\|\beta\|}  \\
      &= \pr{\eta > r + \alpha}  \\
      &= \Phi\left(- \frac{r+\alpha}{\sigma}\right).
  \end{align*}
  Therefore, it remains to prove that
  \begin{equation} \label{eqn:tv-gaussian-bound}
    2d\Phi\left(- \frac{\hat{R}}{\sigma}\right) \leq \frac{1}{2}\Phi\left(- \frac{r+\alpha}{\sigma}\right) .
  \end{equation}
  By Gaussian tail bounds (Lemma \ref{lemma:gaussian-hazard-rate}),
    \[ \frac{\phi(z)}{z}\left(1 - \frac{1}{z^2}\right) \leq 1 - \Phi(z) \leq \frac{\phi(z)}{z} \]
  where $\phi$ is the standard Gaussian PDF.
  So (note that $\Phi(-z) = 1 - \Phi(z)$)
  \begin{align*}
    A &\defeq 2d\Phi\left( - \frac{\hat{R}}{\sigma}\right)  \\
      &\leq \frac{2d\sigma}{\hat{R}} \phi\left(\frac{\hat{R}}{\sigma}\right) .
  \end{align*}
  And
  \begin{align*}
    B &\defeq \frac{1}{2}\Phi\left(-\frac{r+\alpha}{\sigma}\right)  \\
      &\geq \frac{\sigma}{2(r+\alpha)}\left(1 - \frac{\sigma^2}{(r+\alpha)^2}\right)\phi\left(\frac{r+\alpha}{\sigma}\right)  \\
      &\geq \frac{3\sigma}{8(r+\alpha)} \phi\left(\frac{r+\alpha}{\sigma}\right)  \\
      &\geq \frac{\sigma}{4(r+\alpha)} \phi\left(\frac{r+\alpha}{\sigma}\right) .
  \end{align*}
  using that $r+\alpha \geq 2\sigma$ to lower-bound $\left(1 - \frac{\sigma^2}{(r+\alpha)^2}\right)$ by $\frac{3}{4}$.
  Now, both $A,B \geq 0$, so to show $A \leq B$, it suffices to show $\frac{B}{A} \geq 1$.
  Let $\bar{r} = r + \alpha$.
  We have $\frac{\phi(\bar{r}/\sigma)}{\phi(\hat{R}/\sigma)} = e^{\left(\hat{R}^2 - \bar{r}^2\right) / 2\sigma^2}$.
  So if $\hat{R} \geq \bar{r} + \sigma \sqrt{2\ln(8d)}$, then we have
  \begin{align*}
    \frac{A}{B}
      &=    e^{\left(\hat{R}^2 - \bar{r}^2\right) / 2\sigma^2} \frac{\sigma}{4\bar{r}} \frac{\hat{R}}{2d\sigma}  \\
      &\geq e^{\left(2\sigma^2 \ln(8d)\right) / 2\sigma^2} \frac{\sigma}{4\bar{r}} \frac{\hat{R}}{2d\sigma}  \\
      &= 8d \frac{\sigma}{4\bar{r}} \frac{\hat{R}}{2d\sigma}  \\
      &= \frac{\hat{R}}{\bar{r}}  \\
      &\geq 1.
  \end{align*}
  This proves Inequality \ref{eqn:tv-gaussian-bound}.
\end{proof}

By combining Lemmas \ref{lemma:gauss-margins} (margins of $\Asig$), \ref{lemma:tvdist-margins} (closeness of margins when two distributions are close), and \ref{lemma:gauss-tvdist} (closeness of $\Asig$ and $\Asig'$), we obtain Corollary \ref{cor:truncated-margins}.
\begin{corollary} \label{cor:truncated-margins}
  Suppose $\hat{R} \geq \frac{5r}{4} + \sigma\sqrt{2\ln(8d)}$ and $r \geq 2\sigma$.
  Then the $\hat{R}$-truncated $\sigma$-perturbed adversary $\Asig'$ has $\left(r, \frac{\sigma^2}{r}, \frac{1}{80}\right)$ margins.
\end{corollary}

In Appendix \ref{app:multi}, we prove the main result of this section,
Theorem \ref{theorem:multiparam-greedy-regret}, as a case of Theorem
\ref{theorem:multiparam-greedy-regret-regimes} (which considers both
large and small perturbations).  The proof combines the bounds on the
margin condition from Corollary \ref{cor:truncated-margins} along with
the diversity condition bound of Lemma
\ref{lemma:multiparam-bounded-adv-diverse}.  \fi

\section{Lower Bounds for the Multi-Parameter Setting}\label{sec:lb}
In this section, we show that Greedy can be forced to suffer linear
regret in the multi-parameter setting unless it is given a ``warm
start'' that scales polynomially with $\frac{1}{\sigma}$, the
perturbation parameter, and $1/\min_i ||\beta||_i$, the norm of the
smallest parameter vector. This shows that the polynomial dependencies
on these parameters in our upper bound cannot be removed. Both of our
lower bounds are in the fully stochastic setting -- i.e. they do not
require that we make use of an adaptive adversary. First, we focus on
the perturbation parameter $\sigma$.

\begin{theorem}\label{thm:lower}
  Suppose greedy is given a warm start of size
  $n \leq \left(\frac{1}{100 \sigma^2\ln \frac{\rho}{100}}\right)$ in a 
  $\sigma$-perturbed instance for some $\rho$. Then, there exists an
  instance for which Greedy incurs regret
  $\Omega( \frac{\rho}{\sqrt{n}} )$ with constant probability in its
  first $\rho$ rounds.
\end{theorem}

\begin{remark}
Theorem~\ref{thm:lower} implies for $T < \exp(\frac{1}{\sigma})$, either
\iffull
  \begin{itemize}
  \item  \fi $n = \Omega\left(\textrm{poly}\left(\frac{1}{\sigma}\right)\right)$, or
  \iffull \item \fi Greedy suffers linear regret.
\iffull
    \end{itemize}\fi
  \end{remark}
\iffull
\begin{remark}\label{remark:upper-sigma}
  Note that
  $n \leq \left(\frac{1}{100 \sigma^2\ln \frac{\rho}{100}}\right) $
  implies $\sigma \leq \sqrt{\frac{1}{100 n \ln \frac{\rho}{100}}} $.
\end{remark}
\fi

The lower bound instance is simple: one-dimensional, with two-arms and 
 model parameters $\bi{1} = \bi{2} = 1$. In each round
(including the warm start) the unperturbed contexts are
$\uxit{1}{} = 1$ and $\uxit{2}{} = 1 - 1/\sqrt{n}$, and so the
perturbed contexts $\pxit{1}{t}$ and $\pxit{2}{t}$ are drawn
independently from the Gaussian distributions
$\mathcal{N}(1, \sigma^2)$ and
$\mathcal{N}(1 - \frac{1}{\sqrt{n}}, \sigma^2)$, for
$\sigma = \sqrt{\frac{1}{100 n \ln \frac{\rho}{100}}} $. Informally,
we show that the estimators after the warm start have additive error
$\Omega\left(\frac{1}{\sqrt{n}}\right)$ with a constant probability,
and when this is true, with constant probability, arm $1$ will only be
pulled $\tilde{O}\left(n^{2/3}\right)$ rounds. Thus, with constant
probability greedy will pull arm $2$ nearly every round, even though
arm $1$ will be better in a constant fraction of rounds.

  \iffull Let us use the notation $\bhit{i}{}$ to refer to
the OLS estimator immediately after the warm start, and recall that
$\bhit{i}{t}$ refers to the OLS estimator after round $t$ along with
the warm start samples.

We now define several events which are useful for describing the
initialization and evolution of the OLS estimators. Fix some constants
$c_1, c_2$.  Define the event $C_1$ to be when the initial estimator
for arm $1$ from the warm start is small:
\begin{align}
\bhit{1}{} \leq \bi{1} - \frac{c_1}{\sqrt{n}}\label{eqn:smallone}
\end{align}
and event $C_2$ to be when the initial estimator for arm $2$ from the warm start
is large:
\begin{align}
\bhit{2}{} \geq \bi{2} + \frac{c_2}{\sqrt{n}}\label{eqn:bigtwo},
\end{align}
and the event $C$ which corresponds to $C_1$ and $C_2$ both
occurring.

Similarly, let $L^t_1$ be the event in which arm 1's estimator in round $t$
is significantly below its mean:
\begin{align}
\bhit{1}{t} \leq \bi{1} - \frac{4}{\sqrt{n}}\label{eqn:smallone}
\end{align}
and event $L^t_2$ the event in which arm $2$'s estimator in round $t$ is
not significantly below its mean:
\begin{align}
\bhit{2}{t} \geq \bi{2} - \frac{2}{\sqrt{n}}\label{eqn:bigtwo},
\end{align}
and $L^t$ the event in in which both $L^t_1, L^t_2$ occur in round $t$.
Let the event $G$ be the event in which all $t$ rounds have bounded
contexts, namely that for all $i, t$:
\[ \pxit{i}{t} \in \left[\uxit{i}{t} - \frac{1}{100\sqrt{n}}, \uxit{i}{t} + \frac{1}{100\sqrt{n}}\right].\]

We will first show that after a warm start of length $n$, event $C$
occurs with constant probability: the initial OLS estimates are off by
some multiple of their standard deviation (Lemma~\ref{lem:variance-ols}). If
this is the case, then the estimation error will often cause greedy to
select arm 2 during rounds in which arm 1 actually provides higher
reward.  In many of these rounds, greedy incurs regret
$\Omega\left(1/{\sqrt{n}}\right)$. For our lower bound to hold, it
suffices to show that the greedy algorithm makes such mistakes on at
least a constant fraction of $T$ rounds. We condition on the contexts
being bounded (event $G$) and event $C$ for the remainder of the
informal description.

If $L^t$ continues to hold ($1$'s estimator stays small and $2$'s
large, respectively), then on a constant fraction of rounds the greedy
algorithm will pick arm $2$ even though arm $1$ has reward that is
higher by $\frac{1}{\sqrt{n}}$. So, we aim to bound the number of
rounds for which $L^t$ does not hold (equivalently the number of
rounds in which $L^t_1$ is false, plus the number of of rounds for
which $L^t_2$ is false but $L^t_1$ is true).

We argue that with constant probability, $L^t_1$ is true for all
rounds and $L^t_2$ is false for at most $O(n^{2/3}\ln n^{2/3})$
rounds. While $L^t_1$ has been true for all previous rounds, the only
way to choose arm $1$ is for $L^t_2$ to be false; pulling arm $2$ and
having a small resulting estimate happens at most $O(n^{2/3})$ times
(Lemma~\ref{lem:big-no-drift}) with a constant probability. Once arm
$2$'s estimate is small and arm $1$'s is small, we argue that with
probability $1-\delta$, after $O(\sqrt{\ln\frac{1}{\delta}})$ rounds,
arm $2$ will be pulled again (and therefore either $L^t_2$ will be
true again or this pull of $2$ will count as one of the $O(n^{2/3})$
pulls of $2$ which can cause it to be small). So, while arm $1$'s
estimate is small, arm $1$ won't be pulled very many times, and if arm
$1$ isn't pulled very many times, there is constant probability it
will never get large enough to be played more often.

  We now formalize our lemmas before proving the main theorem. The
  first lemma states that the warm start of size $n$ has constant
  probability of producing OLS estimators with error on the order of
  $\frac{c_i}{\sqrt{n}}$ for any constant $c_i$.
  \begin{lemma}\label{lem:variance-ols}
    Suppose that $\uxit{1}{}, \uxit{2}{}\in [1/2, 1]$ and $\sigma < 1$. 
    Fix any warm start size $n$. Consider two batches of data
    $B_1 = \{(\pxit{1}{1}, \rit{1}{1}), \ldots, (\pxit{1}{n},
    \rit{1}{n})\}$
    and
    $B_2 = \{(\pxit{2}{1}, \rit{2}{1}), \ldots, (\pxit{2}{n},
    \rit{2}{n})\}$
    such that each $\pxit{i}{t} \sim
    \mathcal{N}(\uxit{i}{}, \sigma^2)$ and each $\rit{i}{t}\sim\mathcal{N}(\bi{i} \pxit{i}{t}, 1)$ independently . Then, for any constants
    $c_1, c_2$, there exists a constant $c^*$ such that

\[\Pr{B_1, B_2}
{\bhit{1}{} \leq \bi{1} - \frac{c_1}{\sqrt{n}} \wedge \bhit{2}{} \geq \bi{2} + \frac{c_2}{\sqrt{n}}} \geq c^*.\]
\end{lemma}

The next lemma formalizes the idea that, conditioned on an initially
too-large OLS estimate, there are relatively few rounds for which that
 OLS estimator is updated to be much below its expectation.

  \begin{lemma}\label{lem:big-no-drift}
    Suppose $\tfrac{8T}{\delta} <
    2^{n^{1/3}}$ and
    $\sigma < \frac{1}{100\sqrt{\ln2 T}}$. Fix an arm $i$. Suppose
    that $\uxit{i}{t} =1$ for all $t$. Let
\[S^T = \{(\eit{i}{t}, \rit{i}{t}) | i \in \{1,2\}, t \in (n, T + n], \eit{i}{t}\sim \mathcal{N}(0, \sigma^2),  \rit{i}{t} \sim \mathcal{N}(\bi{i} \pxit{i}{t}, 1) \}.\]
Then,
\[\Pr
{S^T}
{ \sum_{t=n}^{n+T} \mathbb{I}\left[\bhit{i}{t}\leq \bi{i} - \frac{2}{\sqrt{n}} \wedge i^t = i\right] \leq 0.00048 n ^{2/3}
  \large |  \bhit{i}{} \geq \bi{i} + \frac{120}{\sqrt{n}} }
\geq \frac{1}{2}. \]
   \end{lemma}

   The next lemma states that the contexts have magnitude that can be uniformly bounded
   with constant probability, and that the OLS estimators are somewhat
   smooth: if they are computed from $\geq n$ samples, to move them by an
   additive $\frac{1}{\sqrt{n}}$, one must have $n$ new observations.

\begin{lemma}\label{lem:smooth}
  Suppose $\bhit{i}{t}$ is computed from $t_i(t)$ samples,
  $\uxit{i}{t} = \uxit{i}{} \in \left( \frac{1}{2}, 1\right]$ and for
  some $n$, $\sigma < \frac{1}{100 \sqrt{n\ln{ 100 T}}}$. Then, with
  probability at least $.99$, for all rounds $t$, we have
  $\pxit{i}{t} \in \left[\uxit{i}{} - \frac{1}{100\sqrt{n}},
    \uxit{i}{} + \frac{1}{100\sqrt{n}}\right]$.
  Refer to this event as $G$.

    Furthermore, if after the estimator $\bhit{i}{t'}$ in round $t'$
    has been computed from $m = t_i(t') - t_i(t)$ additional samples,
\[\Pr{\nit{i}{t} \sim \mathcal{N}(0,1)}{\left| \bhit{i}{t'} - \bhit{i}{t}\right| \geq  \frac{\sqrt{\left(\frac{1}{t_i(t')} - \frac{t_i(t)}{t_i(t')^2}\right) \ln\frac{1}{\delta}}   \left(\uxit{i}{} + \frac{1}{100\sqrt{n}}\right)}{ \left(\uxit{i}{} - \frac{1}{100\sqrt{n}}\right)^2} | G } \leq 1 - \delta.\]

\end{lemma}
\begin{proof}[Proof of Lemma~\ref{lem:smooth}]
  The first claim follows by a direct application of concentration of a
  $\sigma$-sub-Gaussian random variable: with probability $1-\delta$, one is
  at most $\sigma \ln\frac{1}{\delta}$ from its mean, e.g. that for any fixed $t$
\[\pxit{i}{t} \in \left[\uxit{i}{} - \sigma\sqrt{ \ln\frac{1}{\delta}},  \uxit{i}{} +
      \sigma\sqrt{\ln\frac{1}{\delta}}\right] \]
and a union bound and algebra imply that for all $t$, we have
\[ \pxit{i}{t} \in \left[ \uxit{i}{} - \sigma \sqrt{\ln{\frac{T}{\delta}}},  \uxit{i}{} +  \sigma\sqrt{ \ln{\frac{T}{\delta}}}\right] \subseteq \left[ \uxit{i}{} - \frac{1}{100\sqrt{n}},  \uxit{i}{} +  \frac{1}{100\sqrt{n}}\right]\]
where this follows from an lower bound on $\delta$ of at most
$\frac{1}{100}$ and the upper bound on $\sigma$.

We now proceed to prove the second claim conditioned on $G$.  For
simplicity, define $\pxit{i}{j} = 0$ for any round $j$ for which we do
not have an observation of arm $i$.  Define
$\sum_{j < t} \left(\pxit{i}{j}\right)^2 = L$ , $\sum_{j \leq t} \nit{i}{j} \pxit{i}{j} = H$,
$\sum_{t < j < t'} \nit{i}{j} \pxit{i}{j} = K$ . Then, we can
upper-bound how much our estimator moves:
\[
\begin{split}
|\bhit{i}{t'} - \bhit{i}{t}| &  =\left| \frac{\sum_{j < t} \nit{i}{j} \pxit{i}{j} + \sum_{ t \geq  j \geq t} \nit{i}{j} \pxit{i}{j} }{\sum_{j \leq t'} \left(\pxit{i}{j}\right)^2}
-\frac{\sum_{j < t} \nit{i}{j} \pxit{i}{j}  }{\sum_{j < t} \left(\pxit{i}{j}\right)^2}\right|\\
 &  =\left| \frac{H+ K}{L + \sum_{t < j \leq t'} \left(\pxit{i}{j}\right)^2}
-\frac{H  }{L}\right|\\
& \leq  \left|\frac{K}{L + \sum_{t < j \leq t'} \left(\pxit{i}{j}\right)^2}\right|\\
& \leq\left|  \frac{\sum_{t < j \leq t'} \nit{i}{j} \cdot  \left(\uxit{i}{j} + \frac{1}{100\sqrt{n}}\right)}{(t_i(t')) \cdot\left(\uxit{i}{j} - \frac{1}{100\sqrt{n}}\right)^2}\right|\\
& \leq  \frac{\sqrt{\left(t_i(t') - t_i(t)\right) \ln\frac{1}{\delta}}   \left(\uxit{i}{} + \frac{1}{100\sqrt{n}}\right)}{(t_i(t')) \cdot \left(\uxit{i}{} - \frac{1}{100\sqrt{n}}\right)^2}\\
& =   \frac{\sqrt{\left(\frac{1}{t_i(t')} - \frac{t_i(t)}{t_i(t')^2}\right) \ln\frac{1}{\delta}}   \left(\uxit{i}{} + \frac{1}{100\sqrt{n}}\right)}{  \left(\uxit{i}{} - \frac{1}{100\sqrt{n}}\right)^2}\\
&\end{split}
\]
where the  last inequality (which holds with probability $1-\delta$)
follows from a Hoeffding bound for sub-Gaussian random variables and
noticing that $\sum_j \nit{i}{j}$ is Gaussian and $\uxit{i}{t} = \uxit{i}{}$ for all $i$.
\end{proof}

We now proceed with the proof of the main theorem.

\begin{proof}[Proof of Theorem~\ref{thm:lower}]
Let
    $B_1 = \{(\pxit{1}{1}, \rit{1}{1}), \ldots, (\pxit{1}{n},
    \rit{1}{n})\}$
,
    $B_2 = \{(\pxit{2}{1}, \rit{2}{1}), \ldots, (\pxit{2}{n},
    \rit{2}{n})\},$
and

\[S^t = \{(\eit{i}{t}, \rit{i}{t}) | i \in \{1,2\}, t \in (n, T + n], \eit{i}{t}\sim \mathcal{N}(0, \sigma^2),  \rit{i}{t} \sim \mathcal{N}(\bi{i} \pxit{i}{t}, 1) \}.\]

  We show that with constant probability, the regret of greedy is
  $\Omega\left(\left(T - n^{2/3}\right)\frac{1}{\sqrt{n}}\right)$.
  This implies greedy's overall regret is also
  $\Omega\left(\left(T - n^{2/3}\right)\frac{1}{\sqrt{n}}\right)$:
  for any event $E$, we have that
\begin{align*}
\regret{\textsc{greedy}} & \geq \Pr{S^t, B_1, B_2}{C , E, G}\cdot \regret{\textsc{greedy} | C, E, G}\\
& =   \frac{1}{\sqrt{n}} \Pr{S^t, B_1, B_2}{C , E, G}\cdot  \sum_{t}\Pr{S^t}{\bi{1}\pxit{1}{t} >\bi{2}\pxit{2}{t} + \frac{1}{\sqrt{n}} \wedge \bhit{1}{t}\pxit{1}{t} < \bhit{2}{t}\pxit{2}{t} | C, E, G}\\
& \geq  \frac{1}{\sqrt{n}} \Pr{S^t, B_1, B_2}{C , E, G} \cdot  \sum_{t} \left(\Pr{S^t}{ \bhit{1}{t}\pxit{1}{t} < \bhit{2}{t}\pxit{2}{t} |C, E, G}- \Pr{S^t}{ \bhit{1}{t}\pxit{1}{t} < \bhit{2}{t}\pxit{2}{t}  \wedge \bi{1}\pxit{1}{t} \leq \bi{2}\pxit{2}{t} + \frac{1}{\sqrt{n}} | C, E, G}\right) \\
& \geq  \frac{1}{\sqrt{n}} \Pr{S^t, B_1, B_2}{C , E, G} \cdot  \sum_{t} \left(\Pr{S^t}{ \bhit{1}{t}\pxit{1}{t} < \bhit{2}{t}\pxit{2}{t} |C, E, G}- \Pr{S^t}{  \bi{1}\pxit{1}{t} \leq \bi{2}\pxit{2}{t} + \frac{1}{\sqrt{n}} | C, E, G}\right) \\
& \geq  \frac{1}{\sqrt{n}} \Pr{S^t, B_1, B_2}{C , E, G}\cdot  \sum_{t} \left(\Pr{S^t}{ \bhit{1}{t}\pxit{1}{t} < \bhit{2}{t}\pxit{2}{t} | C, E, G}- \frac{1}{2}\right)\\
\end{align*}
where the second inequality follows from Lemma~\ref{lem:variance-ols}
and the last inequality follows from the definition of the instance:
the reward of arm $1$ is larger than that of arm $2$ with probability
at least $\frac{1}{2}$ since its average reward is higher and the
perturbations are iid and symmetric.

We note that, conditioned on $G$, if both
$\bhit{2}{t} \geq 1 -\frac{2}{\sqrt{n}}, \bhit{1}{t} \leq 1
-\frac{4}{\sqrt{n}}$, that
\[\bhit{2}{t} \pxit{2}{t} \geq (1 -\frac{2}{\sqrt{n}}) \left(1 - \frac{1}{\sqrt{n}} - \frac{1}{100\sqrt{n}}\right)  \geq (1- \frac{4}{\sqrt{n}}) (1+\frac{1}{100\sqrt{n}})  \geq \bhit{1}{t} \pxit{1}{t}.
\]
So, if arm $1$ is to be pulled conditioned on $G$, either
$\bhit{2}{t} \leq 1 -\frac{2}{\sqrt{n}}$ or
$ \bhit{1}{t} \geq 1 -\frac{4}{\sqrt{n}}$.  We use this fact below:

\begin{align}
\label{eqn:greedy}
\begin{split}
\regret{\textsc{greedy}}& \geq  \frac{1}{\sqrt{n}}  \Pr{S^t, B_1, B_2}{C , E, G}\cdot  \sum_{t} \left(1 - \Pr{}{ \bhit{1}{t}\pxit{1}{t} > \bhit{2}{t}\pxit{2}{t} | C, E, G}- \frac{1}{2}\right)\\
& =  \frac{1}{\sqrt{n}} \Pr{S^t, B_1, B_2}{C , E, G}\cdot  \sum_{t} \left(\frac{1}{2} - \Pr{S^t}{ \bhit{1}{t} > 1- \frac{4}{\sqrt{n}} \vee \bhit{2}{t}\leq 1 - \frac{2}{\sqrt{n}} | C, E, G}\right)\\
& = \frac{1}{\sqrt{n}}  \Pr{S^t, B_1, B_2}{C , E, G}\cdot  \sum_{t} \left(\frac{1}{2} - \Pr{S^t}{ \bhit{1}{t} > 1- \frac{4}{\sqrt{n}} \wedge \bhit{2}{t}> 1 - \frac{2}{\sqrt{n}} | C, E, G}  - \Pr{S^t}{  \bhit{2}{t}\leq 1 - \frac{2}{\sqrt{n}} | C, E, G}\right)
\end{split}
\end{align}

We now instantiate $E$ to be the event that
$\bhit{2}{t} \geq 1 - \frac{2}{\sqrt{n}} $ for all but
$0.0048 n^{2/3}$ rounds after it was pulled in round $t$ and that for
all of these times, at most $\sqrt{\ln\frac{n^{2/3}}{\delta}}$ pairs
of contexts arrive before one for which $\pxit{2}{t} > \pxit{1}{t}$,
causing arm $2$ to be pulled anew. This conditioning affords us the
following:
\begin{align}
\label{eqn:twopulls}
\begin{split}
& \sum_t \Pr{}{  \bhit{2}{t}\leq 1 - \frac{2}{\sqrt{n}} | C, E, G}\\
& \leq \sum_{t :\textrm{arm 2 pulled in round }t}  \Pr{}{  \bhit{2}{t}\leq 1 - \frac{2}{\sqrt{n}} | C, E, G} + \sum_{t:\textrm{arm 1 pulled in round } t} \Pr{}{  \bhit{2}{t}\leq 1 - \frac{2}{\sqrt{n}} | C, E, G} \\
& \leq 0.0048n^{2/3}(1 + \sqrt{\ln\frac{n^2/3}{\delta}})
\end{split}\end{align}
where the final bound follows from Lemma~\ref{lem:big-no-drift}.
We now upper-bound the probability arm $1$'s estimate gets large:

\begin{align}
\label{eqn:onesmall}
\begin{split}
& \sum_t \Pr{}{ \bhit{1}{t} \geq 1 - \frac{4}{\sqrt{n}},  \bhit{2}{t}\geq 1 - \frac{2}{\sqrt{n}} | C, E, G}\\
& \leq \sum_{t :\textrm{arm 1 pulled at most $n$ times before round  }t} \Pr{}{ \bhit{1}{t} \geq 1 - \frac{4}{\sqrt{n}},  \bhit{2}{t}\geq 1 - \frac{2}{\sqrt{n}} | C, E, G}
\\
&+  \sum_{t :\textrm{arm 1 pulled  more than  $n$ times before round  }t} \Pr{}{ \bhit{1}{t} \geq 1 - \frac{4}{\sqrt{n}},  \bhit{2}{t}\geq 1 - \frac{2}{\sqrt{n}} | C, E, G}\\
& =  \sum_{t } \Pr{}{ \bhit{1}{t} \geq 1 - \frac{4}{\sqrt{n}},  \bhit{2}{t}\geq 1 - \frac{2}{\sqrt{n}}, \textrm{arm 1 pulled  more than  $n$ times before round  }T | C, E, G}\\
& \leq T \delta
\end{split}\end{align}
where we used the concentration result for $\bhit{1}{t}, \bhit{1}{}$
of Lemma~\ref{lem:smooth} and the conditioning on $C$, a small
initialization of $\bhit{1}{}$, and finally the conditioning on $E$
implying that arm $1$ will have been pulled at most
$0.0048n^{2/3}\sqrt{\ln\frac{n^{2/3}}{\delta}}$ times with probability
$1-\delta$, which will mean its estimate stays small when these events occur.

We now collect terms from
Equations~\ref{eqn:greedy},~\ref{eqn:twopulls},~\ref{eqn:onesmall},
which together imply:

\begin{align}
\label{eqn:whichever}
\begin{split}
& \regret{\textsc{greedy}}\\
& \geq    \frac{1}{\sqrt{n}}  \Pr{S^t, B_1, B_2}{C , E, G}\cdot  \sum_{t} \left(\frac{1}{2} - \Pr{B_1, B_2}{ \bhit{1}{t} > 1- \frac{4}{\sqrt{n}} \wedge \bhit{2}{t}> 1 - \frac{2}{\sqrt{n}} | C, E, G}  - \Pr{B_1, B_2}{  \bhit{2}{t}\leq 1 - \frac{2}{\sqrt{n}} | C, E, G}\right)\\
& \geq    \frac{1}{\sqrt{n}}  \Pr{S^t, B_1, B_2}{C , E, G}\cdot \left(\frac{T}{2}  -  T \delta  - 0.0048n^{2.3} \sqrt{\ln\frac{n^{2/3}}{\delta}}\right)\\
& =    \frac{1}{\sqrt{n}}  \Pr{S^t, B_1, B_2}{ E |C, G}\cdot \Pr{S^t, B_1, B_2}{ C, G}\cdot \left(\frac{T}{2}  -  T \delta  - 0.0048n^{2.3} \sqrt{\ln\frac{n^{2/3}}{\delta}}\right)\\
& =    \frac{1}{\sqrt{n}}  \Pr{S^t, B_1, B_2}{ E |C, G}\cdot \Pr{B_1, B_2}{C} \Pr{S^t}{G}\cdot \left(\frac{T}{2}  -  T \delta  - 0.0048n^{2.3} \sqrt{\ln\frac{n^{2/3}}{\delta}}\right)\\
& \geq    \frac{1}{\sqrt{n}}
\left(\frac{1}{2} - \delta\right) \cdot c^*\cdot 0.99 \left(\frac{T}{2}  -  T \delta  - 0.0048n^{2.3} \sqrt{\ln\frac{n^{2/3}}{\delta}}\right)\\
\end{split}
\end{align}
where the last equality follows from the fact that $C$ and $G$ are
independent (one has to do with the warm start, the other with the
contexts arriving after the warm start), and the last inequality uses
Lemmas~\ref{lem:big-no-drift} and a union bound on the probability any
one of the $0.0048 n^{2/3}$ runs of arm $2$ being small without being
pulled lasting longer than $\sqrt{\ln\frac{n^{2/3}}{\delta}}$ being at
most $\delta$ , Lemmas~\ref{lem:variance-ols} and~\ref{lem:smooth}.
\end{proof}

\fi

We now turn our attention to showing that the warm start must also
grow with $1/\min_i ||\bi{i}||$. Informally, the instance we use to show
this lower bound has unperturbed contexts $\uxit{i}{t} = 1$ for both
arms and all rounds, and $\bi{1} = 8\epsilon, \bi{2} =
10\epsilon$.
We show again that the warm start of size
$n$ yields, with constant probability, estimators with error $\frac{c_i}{\sqrt{n}}$, causing Greedy to choose arm $2$ rather
than arm $1$ for a large number of rounds. When $2$ is not pulled too
many times, with constant probability its estimate remains small and
continues to be passed over in favor of arm $1$.

\begin{theorem}\label{thm:lower-difference}
  Let $\epsilon = \min_i |\bi{i}| $,
  $\sigma < \frac{1}{\sqrt{\ln \frac{T}{\delta}}}$ and
  $\frac{T}{\delta} < 2^{n^{1/3}}$.  Suppose Greedy is given a warm
  start of size $n \leq \frac{1}{2\epsilon}$.  Then, there is an
  instance which causes Greedy to incur expected regret
  \[  \Omega \left(\epsilon \left(e^{\frac{1}{ 18\sigma^2}} - n^{2/3}\right)\right).\]
\end{theorem}

\iffull
\begin{proof}
  Consider the instance $\uxit{1}{t} = \uxit{2}{t} = 1$ and
  $\bi{1} = 10\epsilon $ while $\bi{2}= 8\epsilon
  $.
  Lemma~\ref{lem:variance-ols} implies there is constant probability
  that
  $\bhit{1}{} \leq \bi{1} - \frac{c_1}{\sqrt{n}}, \bhit{2}{} \geq
  \bi{2} + \frac{c_2}{\sqrt{n}}$
  for the initial OLS estimators after a warm start of size $n$. If
  $\frac{1}{\sqrt{n}} > \epsilon$, this implies with constant
  probability that
  $\bhit{1}{} \leq \bi{1} - \frac{20}{\sqrt{n}} \leq \bi{1} -
  20\epsilon$
  and
  $\bhit{2}{} \geq \bi{2} + \frac{120}{\sqrt{n}} \geq \bi{2} +
  120\epsilon$.

  Lemma~\ref{lem:big-no-drift} separately implies that conditioned on
  this event, the expected number of rounds in which arm $2$ is pulled
  and then $\bhit{2}{t} \leq \bi{2} - \frac{2}{\sqrt{n}}$ is at most
  $0.00048n^{2/3}$ (and that with probability at least $\frac{1}{2}$,
  no more than twice this many rounds satisfy that inequality). Then,
  with probability $1-\delta$, once arm $2$ has a small estimate, at
  most $\sqrt{\ln\frac{1}{\delta}}$ rounds will occur before $2$ is
  pulled again; thus, with probability $1-\delta$, there are at most
  $0.00048n^{2/3} \sqrt{\ln\frac{n^{2/3}}{\delta}}$ rounds where both
  $\bhit{2}{t} < \bi{2} - \frac{2}{\sqrt{n}}$ and arm $1$ is small. By
  Lemma~\ref{lem:smooth}, this implies
  $\bhit{1}{t} \leq \bi{1} - \frac{15}{\sqrt{n}}$ with at least
  constant probability, since there have been $\tilde{O}(n^{2/3})$
  updates to the estimator.
  Since both of these are statements only involving the randomness of
  their respective arms\jm{this could use clarification but there are
    more important things...}, the two events are conditionally
  independent and therefore simultaneously occur with at least
  constant probability.

  We now analyze the regret in each round for which
  $\bhit{2}{t} \geq \bi{2} -\frac{2}{\sqrt{n}}$ and
    $\bhit{1}{t} \leq \bi{1} - \frac{15}{\sqrt{n}}$ (of which there
    are $\Omega(T - n^{2/3})$).  Note that in these rounds
    $\bhit{1}{t} \pxit{1}{t} = \bhit{1}{t} \left(1 + \eit{1}{t}\right)
    \leq 10\epsilon - \frac{15}{\sqrt{n}} + 10\epsilon \eit{1}{t} - \frac{15\eit{1}{t}}{\sqrt{n}}$ and
    $\bhit{2}{t} \pxit{2}{t} = \bhit{2}{t} \left(1 + \eit{2}{t}\right)
    \geq 8\epsilon - \frac{2}{\sqrt{n}} + 8\epsilon\eit{2}{t} - \frac{2\eit{2}{t}}{\sqrt{n}}$, thus
\begin{align*}
\label{eqn:twopull}
    \begin{split}\bhit{2}{t}\pxit{2}{t} - \bhit{1}{t}\pxit{1}{t} \geq -2\epsilon +  \frac{13}{\sqrt{n}} + \eit{2}{t} \left( 8\epsilon - \frac{2}{\sqrt{n}}\right)
      + \eit{1}{t}\left(\frac{15}{\sqrt{n}} - 10\epsilon\right)\geq
      \frac{11}{\sqrt{n}} + \eit{2}{t} \left( 8\epsilon -
        \frac{2}{\sqrt{n}}\right) +
      \eit{1}{t}\left(\frac{15}{\sqrt{n}} -
        10\epsilon\right)\end{split}\end{align*}
  which is greater than zero (thereby causing arm
  $2$ to be pulled) unless
  $\frac{11}{\sqrt{n}} < \eit{2}{t} \left( -8\epsilon +
    \frac{2}{\sqrt{n}}\right) + \eit{1}{t}\left(-\frac{15}{\sqrt{n}} +
    10\epsilon\right)
  $.  This only holds if one of
  $\max\left(|\eit{1}{t}|, |\eit{2}{t}|\right) \geq
  \frac{1}{3}$, which happens with probability at most
  $2e^{-\frac{1}{18\sigma^2}}$. Thus, with probability at least
  $\frac{1}{2}$, $ \bhit{2}{t}\pxit{2}{t} - \bhit{1}{t}\pxit{1}{t}>
  0$ for at least
  $e^{\frac{1}{18\sigma^2}}$ consecutive rounds (therefore causing
  $2$ to be pulled in all those rounds).

    Finally, we calculate the expected regret coming from these
    $e^{\frac{1}{18\sigma^2}}$ rounds of pulling arm $2$ rather than
    arm $1$.
    $\bi{2}\pxit{2}{t} - \bi{1}\pxit{1}{t} = 2\epsilon
    \left(4\pxit{2}{t} - 5\pxit{1}{t}\right) = \Omega(\epsilon)$ with
    any constant probability. So,  with constant probability, the expected regret is at least

  \[\Omega\left(\left(e^{\frac{1}{18\sigma^2}} -n^{2/3}\right)\epsilon\right).\]
  \end{proof}

\fi

\paragraph*{Acknowledgements}
We thank Mallesh Pai and Rakesh Vohra for helpful conversations at an early stage of this work. 

\bibliographystyle{plainnat}
\bibliography{references,refs}

\vfill
\break

\appendix

\section{Probability Tools and Inequalities} \label{app:probability}

\subsection{Subgaussians and concentration}

\paragraph{Subgaussian variables.}
We call a real-valued random variable $Z$ \emph{$\theta^2$-subgaussian} if its mean is zero and for all $b \in \R$, $\E{e^{bZ}} \leq e^{\theta^2 b^2 / 2}$.
\begin{fact}
  If $Z_1$ and $Z_2$ are $\theta_1^2$ and $\theta_2^2$-subgaussian respectively, and are independent, then:
  \begin{enumerate}
    \item For all $b > 0$, $\pr{Z_1 > b} \leq e^{-b^2/2\theta^2}$ and the same holds for $\pr{Z_1 < -b}$.
    \item For all $c \in \R$, $cZ_1$ is $c^2\theta_1^2$-subgaussian.
    \item $Z_1 + Z_2$ is $(\theta_1^2 + \theta_2^2)$-subgaussian.
  \end{enumerate}
  Also, a $\mathcal{N}(0,\sigma^2)$ variable is $\sigma^2$-subgaussian.
\end{fact}

\begin{lemma} \label{lem:subgauss-martingale}
  Let $\eta_1, \ldots, \eta_t$ be independent $s$-subgaussian random variables.
  Let $x^1, \ldots, x^t$ be vectors in $\R^d$ with each $x^{t'}$ chosen arbitrarily as a function of $(x^1,\eta_1),\ldots,(x^{t'-1},\eta_{t'-1})$ subject to $\|x^{t'}\| \leq r$.
  Then with probability at least $1-\delta$,
    \[ \left\|\sum_{t'=1}^t \eta_{t'} x^{t'}\right\| \leq \sqrt{2drts\ln(td/\delta)} . \]
\end{lemma}
\begin{proof}
  Let $S_t = \sum_{t'=1}^t \eta_{t'} x^{t'}$.
  (Note $S_t$ is $d$-dimensional.)
  Then because each $\eta_{t'}$ is mean-zero and $s$-subgaussian, for any fixed $x^{t'}$ with $\|x^{t'}\| \leq r$, the variable $\eta_{t'} x^{t'}$ is mean-zero and $rs$-subgaussian and independent conditioned on steps $1,\ldots,t'-1$.
  Therefore, each coordinate of $S_t$ is an $rs$-subgaussian martingale.
  By a Hoeffding inequality, Lemma \ref{lem:hoeffding-martingale}, we get that with probability $1-\delta/td$, that coordinate is at most $\sqrt{2 t r s \ln(td/\delta)}$.
  Union-bounding over all coordinates and summing their squares, we get with probability $1-\delta$, $\|S_t\|^2 \leq 2drts\ln(td/\delta)$.
\end{proof}

\begin{lemma} \label{lem:hoeffding-martingale}
  Let $Y_1,\ldots,Y_t$ be an \emph{$s$-subgaussian martingale}, i.e. each $Y_j$ is distributed mean-$0$ and $s$-subgaussian conditioned on $Y_1,\dots,Y_{j-1}$.
  Then
    \[ \pr{ \sum_{j=1}^t Y_j \geq \sqrt{2 t s \ln(1/\delta)} } \leq \delta . \]
\end{lemma}
\begin{proof}
  We have for any choice of $\theta > 0$
  \begin{align*}
    \pr{ \sum_{j=1}^t Y_j \geq w }
    &= \pr{ e^{\theta \sum_j Y_j} \geq e^{\theta w} }  \\
    &\leq \frac{\E{e^{\theta \sum_j Y_j}}}{e^{\theta w}}  & \text{Markov's inequality} \\
    &\leq e^{-\theta w} \prod_{j=1}^t \E{e^{\theta Y_j} \mid Y_1,\ldots,Y_{j-1}}  \\
    &= e^{-\theta w} \prod_{j=1}^t e^{\theta^2 s/2}  & \text{martingale and subgaussianity}  \\
    &= e^{\frac{\theta^2 t s}{2} -\theta w}  \\
    &\leq e^{ - \frac{w^2}{2ts} }  & \text{choosing $\theta = \frac{w}{ts}$.}
  \end{align*}
  Now choose $w = \sqrt{2ts \ln(1/\delta)}$.
\end{proof}

The following is Theorem 3.1 in \citet{tropp2011user}, from which we derive some direct corollaries.
\begin{lemma}[\citet{tropp2011user}] \label{lemma:matrix-chernoff-min}
  Let $z^1, \ldots, z^t$ be random, positive-semidefinite adaptively
  chosen matrices with dimension $d$.
  Suppose $\lmax(z^{t'}) \leq R^2$ almost surely for all $t'$.
  Let $Z^t = \sum_{t'=1}^t z^{t'}$ and $W^t = \sum_{t'=1}^t \E{z^{t'} \given z^1, \ldots, z^{t'-1}}$.
  Then for any $\mu$ and any $\alpha \in (0,1)$,
    \[ \pr{ \lmin(Z^t) \leq (1-\alpha)\mu \text{ and } \lmin(W^t) \geq \mu} \leq d \left( \frac{1}{e^{\alpha}(1-\alpha)^{1-\alpha}} \right)^{\mu / R^2} . \]
\end{lemma}

\begin{corollary} \label{cor:matrix-chernoff-min}
  In the setting of Lemma \ref{lemma:matrix-chernoff-min}, if $\mu \geq 10R^2 \ln(2d/\delta)$ and $\pr{\lmin(W^t) < \mu} \leq \frac{\delta}{2}$,
    \[ \pr{ \lmin(Z^t) \leq 0.5\mu } \leq \delta . \]
\end{corollary}
\begin{proof}
  By Lemma \ref{lemma:matrix-chernoff-min} with $\alpha = \frac{1}{2}$,
  \begin{align*}
    \pr{ \lmin(Z^t) \leq \frac{1}{2}\mu \text{ and } \lmin(W^t) \geq \mu}
      &\leq d \left( \frac{1}{e^{1/2}(\frac{1}{2})^{1/2}} \right)^{\mu / R^2}  \\
      &\leq d e^{-0.1 \mu / R^2}  \\
      &\leq \frac{\delta}{2}
  \end{align*}
  plugging in $\mu \geq 10R^2 \ln(2d/\delta)$.

  Now, we have
  \begin{align*}
    \pr{\lmin(Z^t) \leq \frac{1}{2}\mu}
      &=    \pr{\lmin(Z^t) \leq \frac{1}{2}\mu \text{ and } \lmin(W^t) \geq \mu} + \pr{\lmin(Z^t) \leq \frac{1}{2}\mu \text{ and } \lmin(W^t) < \mu}  \\
      &\leq \pr{\lmin(Z^t) \leq \frac{1}{2}\mu \text{ and } \lmin(W^t) \geq \mu} + \pr{\lmin(W^t) < \mu}  \\
      &\leq \frac{\delta}{2} + \frac{\delta}{2} .
  \end{align*}
\end{proof}

\begin{lemma}[\cite{laurent2000adaptive}] \label{lemma:chisq}
  Suppose $X \sim \chi^2_d$, i.e. $X=\sum_{i=1}^d X_i^2$ with each $X_i\sim \mathcal{N}(0,1)$ independently.
  Then
    \[ \pr{X \geq d + 2\sqrt{d\ln(1/\delta)} + 2\sqrt{\ln(1/\delta)}} \leq \delta . \]
\end{lemma}

\begin{corollary} \label{corollary:chisq}
  Let $Y \sim \mathcal{N}(0,\sigma^2\mathbf{I})$ with dimension $d$, then
    \[ \pr{\|Y\| \geq \sigma \sqrt{5d} \left(\ln(1/\delta)\right)^{1/4}} \leq \delta . \]
\end{corollary}

\begin{lemma} \label{lemma:binomial-tail}
  For $Y \sim \text{Binomial}(n,p)$ and $k < np$, $\pr{Y \leq k} < e^{-2\left(np - k\right)^2}$.
\end{lemma}
\begin{proof}
  By a known Binomial tail bound \bo{cite},
    \[ \pr{Y \leq k} < e^{-n KL(\frac{k}{n}, p)} \]
  where $KL(q,p) = q\log\frac{q}{p} + (1-q)\log\frac{1-q}{1-p}$, the KL-divergence between Bernoullis of parameter $p$ and $q$.
  By Pinsker's inequality \bo{cite}, $KL(q,p) \geq 2 TV(q,p)^2$ where $TV(q,p)$ is the total variation distance between these two Bernoullis, which by definition is $|q-p|$.
  This gives
    \[ \pr{Y \leq k} \leq e^{-n\left(p - \frac{k}{n}\right)^2} \]
  as desired.
\end{proof}

\begin{corollary} \label{cor:binomial-upper-tail}
  For $Y \sim \text{Binomial}(n,p)$ and $k > np$, $\pr{Y > k} \leq e^{-2\left(k - np\right)^2}$.
  Hence, with probability at least $1-\delta$, $Y \leq np + \sqrt{\frac{1}{2}\ln\frac{1}{\delta}}$.
\end{corollary}
\begin{proof}
  Apply Lemma \ref{lemma:binomial-tail} to the variable $Y' = n-Y$, which is distributed Binomial$(n,1-p)$; then
  \begin{align*}
    \pr{Y \geq k}
      &=    \pr{Y' \leq n-k}  \\
      &\leq e^{-2\left(n(1-p)-(n-k)\right)^2}  \\
      &=    e^{-2\left(k - np\right)^2}.
  \end{align*}
\end{proof}

\subsection{The Gaussian and truncated Gaussian distributions}

A \emph{truncated} variable is one conditioned on falling into a certain range.
\begin{fact} \label{fact:truncated-subgauss}
  For any $\hat{R} > 0$, a $[-\hat{R},\hat{R}]$-truncated $\mathcal{N}(0,\sigma^2)$ variable is $\sigma^2$-subgaussian.
\end{fact}
\begin{proof}
  Let $p$ be the probability that a $\mathcal{N}(0,\sigma^2)$ variable falls in the interval $[-\hat{R},\hat{R}]$.
  Then the density of $Z$ is $\frac{1}{p} \frac{1}{\sigma \sqrt{2\pi}} e^{-z^2/2\sigma^2}$ on $z \in [-\hat{R},\hat{R}]$ and $0$ otherwise.
  So we have:
  \begin{align*}
    \E{e^{bZ}}
      &= \frac{1}{p} \frac{1}{\sigma \sqrt{2\pi}} \int_{z=-\hat{R}}^{\hat{R}} e^{bz} e^{-z^2/2\sigma^2} dz  \\
      &= \frac{1}{p} \frac{1}{\sigma \sqrt{2\pi}} \int_{z=-\hat{R}}^{\hat{R}} \exp\left[ -\frac{1}{2\sigma^2}\left( z - b\sigma^2 \right)^2 + \frac{b^2\sigma^2}{2} \right] dz  \\
      &= e^{b^2\sigma^2/2} \frac{1}{p} \frac{1}{\sigma \sqrt{2\pi}} \int_{z=-\hat{R}}^{\hat{R}} e^{-\frac{1}{2\sigma^2}(z - b\sigma^2)^2}  \\
      &\leq e^{b^2\sigma^2/2} \frac{1}{p} \frac{1}{\sigma \sqrt{2\pi}} \int_{z=b\sigma^2-\hat{R}}^{b\sigma^2 + \hat{R}} e^{-\frac{1}{2\sigma^2}(z - b\sigma^2)^2}  \\
      &= e^{b^2\sigma^2/2} .
  \end{align*}
  The inequality is justified as follows: The integral (when weighted by the $\frac{1}{\sigma^2\sqrt{2\pi}}$ factor) computes the total probability of a $\mathcal{N}(b\sigma^2,\sigma^2)$ variable falling into a range of length $2\hat{R}$, and this is maximized by the range $[-b\sigma^2,b\sigma^2]$.
  The final equality is justified by the definition of $p$, since the integral (again weighted by the factor) is the probability of a Gaussian with variance $\sigma^2$ falling within $\hat{R}$ of its mean.
\end{proof}

The following, Lemma \ref{lemma:gaussian-hazard-rate} is relatively standard derivation of tail bounds for the Gaussian distribution, although we do not know of a reference containing this particular bound.\bo{wording ok?}
We make the following notational definitions for the lemma:
For integers $n,N \geq 0$, let (where an empty product equals one)
\begin{align*}
  g(x;n) &= \left(-1\right)^n \frac{1}{x^{2n}} \prod_{j=1}^n (2j-1) \\
  G(x;N) &= \sum_{n=0}^N g(x;n) .
\end{align*}
In other words, $g(x;n) = \left(-1\right)^n \frac{(1)(3)(5)\cdots (2n-1)}{x^{2n}}$, and in particular $g(x;0) = 1$ for all $x$.

\begin{lemma} \label{lemma:gaussian-hazard-rate}
  Let $\Phi(x)$ and $\phi(x)$ be the standard Gaussian CDF and PDF respectively.
  Then we have the following bounds on the tail (or hazard rate) of $\Phi(x)$: for all $x > 0$ and odd positive integers $N$,
    \[ 1 - \Phi(x) \geq \frac{\phi(x)}{x}G(x;N) \]
  and for all $x > 0$ and even nonnegative integers $N$,
    \[ 1 - \Phi(x) \leq \frac{\phi(x)}{x}G(x;N+1) . \]
  In particular,
    \[ \frac{\phi(x)}{x}\left(1 - \frac{1}{x^2}\right) \leq 1 - \Phi(x) \leq \frac{\phi(x)}{x}\left(1 - \frac{1}{x^2} + \frac{3}{x^4}\right) . \]
\end{lemma}
\begin{proof}
  We first show that $\frac{d}{dx}\left(\frac{G(x;N)}{x}\right) = G(x;N+1) - 1$.
  We have that $\frac{d}{dx}\left(\frac{1}{x^{2n+1}}\right) = (-1)(2(n+1)-1)\frac{1}{x^{2(n+1)}}$, so $\frac{d}{dx}\left(\frac{g(x;n)}{x}\right) = g(x;n+1)$.
  Therefore, $\frac{d}{dx}\left(\frac{G(x;N)}{x}\right) = \sum_{n=0}^N g(x;n+1) = G(x;N+1) - g(x;0) = G(x;N+1) - 1$.

  We also recall that $\phi(x) = \frac{1}{\sqrt{2\pi}}e^{-x^2/2}$ and observe that $\frac{d\phi}{dx} = -x \phi(x)$.

  Let $f(x) = - \frac{\phi(x)}{x} G(x;N)$.
  We have
  \begin{align*}
    \frac{df}{dx}
      &= - \frac{d\phi}{dx} \frac{G(x;N)}{x} - \phi(x) \frac{d}{dx}\left(\frac{G(x;N)}{x}\right)  \\
      &= \phi(x) G(x;N) - \phi(x) \left(G(x;N+1) - 1\right) \\
      &= \phi(x) \left(1 - g(x;N+1) \right).
  \end{align*}

  \paragraph{Lower bound.}
  Let $N$ be odd; then for all positive $x$, $g(x;N+1) \geq 0$.
  Therefore, eventually using that $\lim_{t\to\infty} f(t) = 0$,
  \begin{align*}
    1 - \Phi(x)
      &= \int_{t=x}^{\infty} \phi(t) dt  \\
      &\geq \int_{t=x}^{\infty} \phi(t) \left(1 - g(x;N+1)\right) dt  \\
      &= \int_{t=x}^{\infty} \frac{df}{dt} dt  \\
      &= 0 - f(x) \\
      &= \frac{\phi(x)}{x} G(x;N).
  \end{align*}

  \paragraph{Upper bound.}
  Let $N$ be even; then for all positive $x$, $g(x;N+1) \leq 0$.
  Then the exact same proof holds as in the lower bound, except the inequality changes from $\geq$ to $\leq$, and we obtain an upper bound.
\end{proof}

\begin{lemma} \label{lemma:upper-truncated-gaussian-variance}
  Let $e \sim \mathcal{N}(0,\sigma^2)$.
  Let $\beta \geq 2\sigma$.
  Then $\Var(e \mid e \leq \beta) \geq \Omega\left(\sigma^2\right)$.
\end{lemma}
\begin{proof}
  If $\eta \sim \mathcal{N}(0,1)$, then $\Var(e \mid e \leq \beta) = \sigma^2 \Var(\eta \mid \eta \leq b)$ where $b = \beta/\sigma$.
  We have
    \[ \Var(\eta \mid \eta \leq b) = 1 - b\frac{\phi(b)}{\Phi(b)} - \left(\frac{\phi(b)}{\Phi(b)}\right)^2 . \]
  For $b \geq 2$, we have $\Phi(b) > 0.977$ (as it is increasing in $b$), while $\phi(b) < 0.054$ and $b\phi(b) < 0.11$ (as both are decreasing in $b$).
  This gives a constant lower bound on $\Var(\eta \mid \eta \leq b)$.
\end{proof}

\begin{lemma} \label{lemma:lower-truncated-gaussian-variance}
  Let $e \sim \mathcal{N}(0,\sigma^2)$ and $a \geq 2\sigma$.
  Then $\Var(e \mid e > a) \geq \Omega\left(\frac{\sigma^4}{a^2}\right)$.
\end{lemma}
\begin{proof}
  If $\eta \sim \mathcal{N}(0,1)$, we have $\Var(e \mid e > a) = \sigma^2 \Var(\eta \mid \eta > \alpha)$ where $\alpha = a/\sigma$.
  The variance of the lower-truncated standard Gaussian is
  \begin{align*}
    \Var(\eta \mid \eta > \alpha)
      &= 1 + \alpha\frac{\phi(\alpha)}{1 - \Phi(\alpha)} - \left(\frac{\phi(\alpha)}{1 - \Phi(\alpha)}\right)^2
  \end{align*}
  Let $h$ be the ``hazard rate'' $\frac{\phi(\alpha)}{1-\Phi(\alpha)}$.
  Lemma \ref{lemma:gaussian-hazard-rate} implies that (using $x \geq 2$ for the second inequality)
    \[ h \geq \frac{x}{1 - \frac{1}{x^2} + \frac{3}{x^4}} \geq x . \]
  It follows that the variance, which can be rewritten as $1 + \alpha - h^2$, only decreases by plugging in the middle term as a lower bound on $h$.
  So
  \begin{align*}
    \Var(\eta \mid \eta > \alpha)
      &\geq 1 + \alpha\frac{\alpha}{1 - \frac{1}{\alpha^2} + \frac{3}{\alpha^4}} - \left(\frac{\alpha}{1 - \frac{1}{\alpha^2} + \frac{3}{\alpha^4}}\right)^2  \\
      &= 1 + \frac{\alpha^2}{1 - \frac{1}{\alpha^2} + \frac{3}{\alpha^4}} - \frac{\alpha^2}{1 - \frac{2}{\alpha^2} + \frac{7}{\alpha^4} - \frac{6}{\alpha^6} + \frac{9}{\alpha^8}}  \\
      &= \frac{1 - \frac{3}{\alpha^2} - O(\frac{1}{\alpha^4}) ~~ + \alpha^2 - 2 + \frac{7}{\alpha^2} - O(\frac{1}{\alpha^4}) ~~ - \alpha^2 + 1 - \frac{3}{\alpha^2}}{1 - \frac{3}{\alpha^2} - O(\frac{1}{\alpha^4})}  \\
      &= \frac{\frac{1}{\alpha^2} + O(\frac{1}{\alpha^4})}{1 - O(\frac{1}{\alpha^2})}  \\
      &= \Omega\left(\frac{1}{\alpha^2}\right) .  \qedhere
  \end{align*}
\end{proof}

\begin{lemma} \label{lemma:double-truncated-gaussian-variance}
  Let $e \sim \mathcal{N}(0,\sigma^2)$.
  Let $b \geq 2a$ and $a \geq 2\sigma$.
  Then $\Var(e \mid a \leq e \leq b) \geq \Omega\left(\frac{\sigma^4}{a^2}\right)$.
\end{lemma}
\begin{proof}
  If $\eta \sim \mathcal{N}(0,1)$, we have $\Var(e \mid a \leq e \leq b) = \sigma^2 \Var(\eta \mid \alpha \leq \eta \leq \beta)$ where $\alpha = a/\sigma$, $\beta = b/\sigma$.
  We have
  \begin{align*}
    \Var(\eta \mid \alpha \leq \eta \leq \beta)
      &= 1 + \frac{\alpha\phi(\alpha) - \beta\phi(\beta)}{\Phi(\beta) - \Phi(\alpha)} - \left( \frac{\phi(\alpha) - \phi(\beta)}{\Phi(\beta) - \Phi(\alpha)} \right)^2 .
  \end{align*}
  We show $\beta = 2\alpha$ is almost the same as the case $\beta = \infty$, Lemma \ref{lemma:lower-truncated-gaussian-variance}.
  First,
  \begin{align*}
    \alpha\phi(\alpha) - \beta\phi(\beta)
      &= \alpha\phi(\alpha) \left(1 - \frac{\beta\phi(\beta)}{\alpha\phi(\alpha)}\right)  \\
      &= \alpha\phi(\alpha) \left(1 - \frac{\beta}{\alpha}e^{-(\beta^2 - \alpha^2)/2}\right)  \\
      &\leq \alpha\phi(\alpha) \left(1 - 2e^{-3\alpha^2/2}\right)  \\
      &= \alpha\phi(\alpha) \left(1 - o(1)\right) .
  \end{align*}
  Second,
  \begin{align*}
    \Phi(\beta) - \Phi(\alpha)
      &= (1 - \Phi(\alpha))\left(1 - \frac{1 - \Phi(\beta)}{1 - \Phi(\alpha)}\right)  \\
      &\geq (1 - \Phi(\alpha))\left(1 - \frac{\phi(\beta)\left(\frac{1}{\beta}\right)}{\phi(\alpha)\left(1-\frac{1}{\alpha^2}\right)}\right)  \\
      &\leq (1 - \Phi(\alpha))\left(1 - \frac{\phi(\beta) \alpha^2}{\phi(\alpha) \beta(\alpha^2-1)}\right)  \\
      &\leq (1 - \Phi(\alpha))\left(1 - e^{-3\alpha^2/2}\frac{\alpha}{\alpha^2-1}\right)  \\
      &= (1 - \Phi(\alpha))\left(1 - o(1)\right) .
  \end{align*}
  Third,
  \begin{align*}
    \phi(\alpha) - \phi(\beta)
      &= \phi(\alpha)\left(1 - \frac{\phi(\beta)}{\phi(\alpha)}\right)  \\
      &\leq \phi(\alpha)\left(1 - e^{-3\alpha^2/2}\right) .
  \end{align*}
  Putting it together,
  \begin{align*}
    \Var(\eta \mid \alpha \leq \eta \leq \beta)
      &\geq 1 + \frac{\alpha \phi(\alpha) \left(1 - o(1)\right)}{\left(1-\Phi(\alpha)\right)\left(1 - o(1)\right)} - \left(\frac{\phi(\alpha) \left(1 - o(1)\right)}{\left(1-\Phi(\alpha)\right)\left(1 - o(1)\right)}\right)^2  \\
      &\geq \Var(\eta \mid \alpha \leq \eta) \left(1 - o(1)\right)  \\
      &\geq \Omega\left(\frac{1}{\alpha^2}\right)
  \end{align*}
  by Lemma \ref{lemma:lower-truncated-gaussian-variance}.
\end{proof}

\begin{lemma} \label{lemma:conditional-gaussian-decreasing}
  Let $\eta \sim \mathcal{N}(0,\sigma^2)$.
  Then for any $\alpha > 0$, the conditional ``margin probability''
    \[ \pr{\eta \geq b + \alpha \given \eta \geq b} \]
  is decreasing in $b$.
\end{lemma}
\begin{proof}
  \begin{align}
    \frac{d}{dr'}\left(\frac{1 - \Phi\left(\frac{r'+\alpha}{\sigma}\right)} {1 - \Phi\left(\frac{r'}{\sigma}\right)} \right)
      &= \frac{ -\phi\left(\frac{r'+\alpha}{\sigma}\right) \left(1 - \Phi\left(\frac{r'}{\sigma}\right)\right) + \phi\left(\frac{r'}{\sigma}\right) \left(1 - \Phi\left(\frac{r'+\alpha}{\sigma}\right)\right)} {\left(1 - \Phi\left(\frac{r'}{\sigma}\right)\right)^2}  \label{eqn:ddr-of-conditional-gauss}
  \end{align}
  Now,
  \begin{align*}
    \phi\left(\frac{r'+\alpha}{\sigma}\right)
      &= \frac{1}{\sqrt{2\pi}} e^{-\frac{(r')^2 + 2r'\alpha + \alpha^2}{2\sigma^2}}  \\
      &= \phi\left(\frac{r'}{\sigma}\right) e^{-\frac{2r'\alpha + \alpha^2}{2\sigma^2}} .
  \end{align*}
  So (\ref{eqn:ddr-of-conditional-gauss}) is negative if and only if the following quantity is negative:
  \begin{align*}
    \left(1 - \Phi\left(\frac{r'+\alpha}{\sigma}\right)\right) - e^{-\frac{2r'\alpha + \alpha^2}{2\sigma^2}} \left(1 - \Phi\left(\frac{r'}{\sigma}\right)\right)
      &= \frac{1}{\sqrt{2\pi}} \int_{z=0}^{\infty} \left(e^{-\left(z+\frac{r'+\alpha} {\sigma}\right)^2 / 2} ~ - ~ e^{-\frac{2r'\alpha + \alpha^2}{2\sigma^2}} e^{-\left(z + \frac{r'}{\sigma}\right)^2 / 2} \right) dt .
  \end{align*}
  The difference inside the integral is
  \begin{align*}
    &\exp\left[ -\left( \frac{z^2}{2} + \frac{zr'}{\sigma} + \frac{z\alpha}{\sigma} + \frac{r'\alpha}{\sigma^2} + \frac{(r')^2}{2\sigma^2} + \frac{\alpha^2}{2\sigma^2} \right) \right] - \exp\left[ -\left( \frac{r'\alpha}{\sigma^2} + \frac{\alpha^2}{2\sigma^2} + \frac{z^2}{2} + \frac{zr'}{\sigma} + \frac{(r')^2}{2\sigma^2} \right) \right]  \\
    &= \exp\left[ - \left( \frac{z^2}{2} + \frac{zr'}{\sigma} + \frac{r'\alpha}{\sigma^2} + \frac{(r')^2}{2\sigma^2} + \frac{\alpha^2}{2\sigma^2}\right)\right]  \left( \exp\left[ -\frac{z\alpha}{\sigma} \right] - \exp\left[ 0 \right] \right)  \\
    &\leq 0
  \end{align*}
  because $\exp\left[-\frac{za}{\sigma}\right] \leq 1$,
  using that $\frac{z\alpha}{\sigma} \geq 0$ as each of $z,\alpha,\sigma \geq 0$.
  This implies the entire integral is nonpositive, which implies (\ref{eqn:ddr-of-conditional-gauss}) is nonpositive, as claimed.
\end{proof}

\subsection{Other inequalities}

\begin{lemma} \label{lemma:log-bound}
  Let $B \geq e = 2.718\dots$ and $A \geq 0$. Then for all $n \geq \max\{1, 2A\ln(AB)\}$, we have
    \[ n \geq A \ln(Bn) . \]
\end{lemma}
\begin{proof}
  We have $\ln(Bn) > 0$, so $n \geq A \ln(Bn)$ if and only if $\frac{n}{\ln Bn} \geq A$.
  We now prove $\frac{n}{\ln Bn} \geq A$ if $n \geq 2 A \ln(Bn)$ and $n \geq 1$.

  We have $\frac{d}{dn}\left(\frac{n}{\ln(Bn)}\right) = \frac{\ln(Bn) - 1}{\ln(Bn)^2} \geq 0$ for all $n \geq 1$ (because $Bn \geq e$).
  So $\frac{n}{\ln Bn}$ is minimized by minimizing $n$.

  First consider the case $2A\ln(AB) \leq 1$, giving the constraint $n \geq 1$.
  By the above derivative discussion, suffices to prove that $\frac{1}{\ln B} \geq A$.
  This case is equivalent to the constraint $\ln B \leq \frac{1}{2A} - \ln A$, which because $B \geq e$, also implies $A \leq 0.72\dots$.
  We have
  \begin{align*}
    \frac{1}{\ln B}
      &\geq \frac{1}{\frac{1}{2A} - \ln A}  \\
      &=    \frac{2A}{1 - 2A\ln A} .
  \end{align*}
  Now we claim $\frac{2}{1-2A\ln A} \geq 1$ for all $A \in (0,1]$ as $A\ln A$ is negative on this interval and minimized as $A=\frac{1}{e}$, giving $\frac{2}{1-2A\ln A} \geq \frac{2}{1 + 2/e} \approx 1.152\dots$.
  This implies $\frac{1}{\ln B} \geq A$ as desired.

  Next consider the case $2A\ln(AB) \geq 1$; since $\frac{n}{\ln(Bn)}$ is increasing in $n$, it is minimized at the lower-bound of $n$, which is $2A\ln(AB)$:
  \begin{align*}
    \frac{n}{\ln B n}
      &\geq \frac{2A \ln(AB)}{\ln(B 2 A \ln(AB))}  \\
      &=    A \frac{2\ln(AB)}{\ln(AB) + \ln(2\ln(AB))}  \\
      &\geq A \frac{2\ln(AB)}{\ln(AB) + \ln(AB)}  \\
      &= A
  \end{align*}
  using that $2\ln(AB) \leq AB$ for all positive $AB$.
\end{proof}


\section{Proofs for Single-Parameter Setting}\label{sec:missing-single}
\begin{proof}[Proof of Lemma~\ref{lemma:singleparam-greedy-regret}]
  Recall that
  \[ \textrm{Regret}(\pxit{}{1},i^1,\ldots,\pxit{}{T},i^T) = \sum_{t=1}^T \left(\b\cdot \pxit{t}{i^{*}} - \b \cdot \pxit{i^t}{t}\right). \]
  Each term in the sum is bounded by
  $\|\b\| \cdot \|\pxit{i^*}{t} - \pxit{i^t}{t}\| \leq 2R$, so the first
  $t_{\min}$ terms are bounded by $2 R t_{\min}$.  For the remainder,
  the only case in which the $t$'th term of the sum is non-zero is
  when $i^t \neq i^*(t)$. In this case, we have:
  \begin{eqnarray*}
    \b\cdot \pxit{i^*}{t} - \b \cdot \pxit{i^t}{t} &=& (\b\cdot \pxit{i^*}{t} - \bhit{}{t} \cdot \pxit{i^*}{t}) - (\b \cdot \pxit{i^t}{t}- \bhit{}{t} \cdot \pxit{i^t}{t}) + (\bhit{}{t} \cdot \pxit{i^*}{t} - \bhit{}{t} \cdot \pxit{i^t}{t}) \\
    &\leq&  (\b\cdot \pxit{i^*}{t}- \bhit{}{t} \cdot \pxit{i^*}{t}) - (\b\cdot \pxit{i^t}{t}- \bhit{}{t} \cdot \pxit{i^t}{t}) \\
    &\leq& \left|(\b\cdot\pxit{i^*}{t}- \bhit{}{t} \cdot \pxit{i^*}{t}) \right| + \left|  (\b\cdot \pxit{i^t}{t}- \bhit{}{t} \cdot \pxit{i^t}{}) \right|
  \end{eqnarray*}
  where the first inequality follows because by definition,
  $i^t = \arg\max_i \bhit{}{t} \cdot\pxit{i}{t}$, and so
  $ (\bhit{}{t} \cdot \pxit{i^*}{t}- \bhit{}{t} \cdot \pxit{i^t}{t})
  \leq 0$ and the second from the fact that
  $a - b \leq | a- b | \leq |a | + |-b| = |a| + |b|$ for any $a, b$
  (by Cauchy Swartz).

  Now, use that, for any $x$,
  $\left| \b\cdot x - \bhit{}{t} \cdot x \right| = \left| (\b-
    \bhit{}{t}) \cdot x \right| \leq \| \b- \bhit{}{t} \| ~ \|x\|$,
  and $\|x\| \leq R$, and the claim follows.
\end{proof}

\begin{proof}[Proof of Lemma~\ref{lem:beta-bound}]
  Given that $Z^t = \XTX$ is invertible, the OLS estimator is
  \begin{align*}
    \bhit{}{t}
      &= \XTXinvXT \left(X^t \b + \nit{}{t}\right)  \\
      &= \b + (Z^t)^{-1} \trans{(X^t)} \nit{}{t}.
  \end{align*}
  Then the difference can be written as
  \begin{align*}
    \| \b - \bhit{}{t} \|
      &= \left\| (Z^t)^{-1} \trans{(X^t)} \nit{}{t} \right\| \\
      &\leq \lmax\left((Z^t)^{-1}\right) \left\| \trans{(X^t)} \nit{}{t} \right\|  \\
      &= \frac{1}{\lmin\left(Z^t\right)} \left\| \trans{(X^t)} \nit{}{t} \right\| .
  \end{align*}
  Because the rewards' errors $\nit{}{t}$ are mean-zero and
  $s$-subgaussian, a standard concentration bound for martingales with
  subgaussian difference (Lemma \ref{lem:subgauss-martingale}) implies
  that with probability $1-\delta$, we have
  $\| \trans{(X^t)} \nit{}{t} \| \leq \sqrt{2dRts\ln(td/\delta)}$.
\end{proof}

Before we present the formal proof of \Cref{cor:diverse-all-converge},
we will prove the following more general result:

\begin{lemma} \label{lemma:diversity-convergence} Consider Greedy in
  the single parameter setting with an $R$-bounded,
  $(r,\lambda_0)$-diverse adversary.  Let
  $t_{\min}(\delta) := \max\left\{32\ln(4/\delta) ~,~ \frac{80R^2 \ln(2d/\delta)}{\lambda_0}\right\}$ and fix
  a particular $t \geq t_{\min}$.  If at least $\frac{t}{2}$ of the
  rounds $t' \leq t$ are $r$-\haus ~ for $i_{t'}$, then with
  probability at least $1-\delta$,
  \[ \| \b - \bhit{}{t} \| \leq \frac{16\sqrt{2dRs \ln(2td/\delta)}}{\lambda_0 \sqrt{t}} . \]
\end{lemma}

\begin{proof}[Proof of \Cref{lemma:diversity-convergence}]
  We will show that with probability $1-\frac{\delta}{2}$,
  \begin{equation}
    \lmin\left(Z^t\right) \geq \frac{t \lambda_0}{16},  \label{eqn:lmin-goal}
  \end{equation}
  which also ensures that $Z^t$ is invertible.
  Then by Lemma \ref{lem:beta-bound}, with probability $1-\frac{\delta}{2}$, we will have
  \[ \| \b - \bhit{}{t} \| \leq \frac{16}{t\lambda_0} \sqrt{2dRts \ln(2td/\delta)} , \]
  which completes the proof.

  To show (\ref{eqn:lmin-goal}),
  we will apply a concentration inequality for minimum eigenvalues,
  Corollary \ref{cor:matrix-chernoff-min}, due to
  \citet{tropp2011user}. The inequality requires two conditions for
  each context $x = \pxit{i^{t'}}{t'}$: an upper bound on the maximum
  eigenvalue of $\trans{x}x$, and a lower bound on the minimum
  eigenvalue of its expectation.  The upper bound follows from
  $\lmax\left(\trans{x}x \right) = \max_{w: \|w\| = 1} w (\trans{x}x)
  \trans{w} \leq \|x\|^2 \leq R^2$ for any context $x$, by
  boundedness.

  Second: let $W^t = \sum_{t'=1}^t
    \Ex{t'-1}{\trans{(\pxit{i^{t'}}{t'})}\pxit{i^{t'}}{t'}}$.
  We must show that with probability $1 - \frac{\delta}{4}$,
  $\lmin \left( W^t \right) \geq
  \frac{t'\lambda_0}{8}$, where the probability is over the perturbations.
  By superadditivity of the minimum
  eigenvalue,
  \begin{align}
    \lmin \left(W^t \right)
      &\geq \sum_{t'=1}^t \lmin \left( \Ex{t'-1}{\trans{(\pxit{i^{t'}}{t'})}\pxit{i^{t'}}{t'}} \right)  .  \label{eqn:single-eigen-sum}
  \end{align}
  By concavity of the minimum eigenvalue,
  \begin{align*}
    \lmin \left( \Ex{t'-1}{\trans{(\pxit{i^{t'}}{t'})}\pxit{i^{t'}}{t'}} \right)
      &\geq \Pr{t'-1}{\text{$\chit{i^{t'}}{t'}$ is $r$-\hgood}} \lmin \left( \Ex{t'-1}{\trans{(\pxit{i^{t'}}{t'})}\pxit{i^{t'}}{t'} \given \text{$\chit{i^{t'}}{t'}$ is $r$-\hgood}} \right)
  \end{align*}
  where we have lower-bounded the rest of the expectation (cases where
  $\chit{i}{t'}$ is not $r$-\hgood) by $0$, as it is expected minimum
  eigenvalue of an expectation over positive semi-definite matrices.
  By Lemma \ref{lemma:good-diversity},
    \[ \lmin \left( \Ex{t'-1} {\trans{(\pxit{i^{t'}}{t'})}\pxit{i^{t'}}{t'} \given \text{$\chit{i^{t'}}{t'}$ is $r$-\hgood}} \right) \geq \lambda_0 . \]
  Meanwhile, by assumption, at least half the rounds $t' \leq t$ are $r$-{\haus } for $i^{t'}$, and for these rounds, $\pr{\text{$\chit{i}{t'}$ is $r$-\hgood} \given i^{t'}=i} \geq \frac{1}{2}$.
  So (\ref{eqn:single-eigen-sum}) stochastically dominates the random variable $\lambda_0 W$ where $W \sim \text{Binomial}(\frac{t}{2},\frac{1}{2})$: it is the sum of at least $\frac{t}{2}$ terms, each of which is at least $\lambda_0$ with probability at least $\frac{1}{2}$, conditioned on the previous terms, and is at least $0$ otherwise.
  So by a Chernoff bound,
  \begin{align*}
    \pr{ \lmin(W^t) \leq \lambda_0 \frac{t}{8}}
      &\leq \exp\left[-\frac{\left(\frac{t}{4} - \frac{t}{8}\right)^2} {2\frac{1}{2}\frac{t}{2}} \right]  \\
      &=    \exp\left[-\frac{t}{32} \right]  \\
      &\leq \frac{\delta}{4}
  \end{align*}
  for $t \geq 32 \ln(4/\delta)$.
  Therefore, we can apply Corollary \ref{cor:matrix-chernoff-min} with $\mu = \frac{t\lambda_0}{8}$ to obtain that, as long as $\mu$ is large enough, we have with probability $1 - \frac{\delta}{2}$, $\lmin(Z^t) \geq \frac{t \lambda_0}{16}$.
  Here ``large enough'' precisely is the requirement $t \geq 32 \ln(4/\delta)$ and $\frac{t \lambda_0}{8} \geq 10R^2 \ln(2d/\delta)$.
\end{proof}

Lemma \ref{lemma:diversity-convergence} showed that the estimate
$\bhit{}{t}$ is accurate for any fixed round; we can now extend this
to show that any bounded, diverse adversary with relatively few
inauspicious rounds will cause Greedy's estimators to converge quickly
for all rounds.

\begin{proof}[Proof of \Cref{cor:diverse-all-converge}]
  We plug in $\frac{\delta}{T}$ to Lemma
  \ref{lemma:diversity-convergence}.  Because only a total of
  $\frac{t_{\min}(\delta/T)}{2}$ rounds $t$ are not auspicious for
  $i^t$, every $t \geq t_{\min}$ satisfies the assumptions of Lemma
  \ref{lemma:diversity-convergence}.  By a union-bound over time
  steps, this gives with probability $1-\delta$, for all
  $t \geq t_{\min}(\delta/T)$,
    \[ \| \b - \bhit{}{t} \| \leq \frac{16 \sqrt{2dRs \ln(2Ttd/\delta)}} {\lambda_0 \sqrt{t}} . \]
  For a more convenient bound, we upper-bound $t$ by $T$ in the numerator, then use $\ln(2T^2d/\delta) \leq 2\ln(2Td/\delta)$.
\end{proof}

\begin{proof}[Proof of \Cref{lemma:asigprime-bounded}]
  For $R$-boundedness: Each context $\pxit{i}{t}$ produced by $\Asig'$ is of the form $\uxit{i}{t} + \eit{i}{t}$ with $\|\uxit{i}{t}\| \leq 1$ and $\|\eit{i}{t}\| \leq \sqrt{d \hat{R}^2} = \sqrt{d}\hat{R}$, giving $\|\pxit{i}{t}\| \leq 1 + \sqrt{d}\hat{R}$.

  For $(r,\frac{1}{T})$-central boundedness:
  Fix any unit vector $w$, arm $i$, and round $t$.
  Recall that $Q^t \eit{i}{t} = \temppit{i}{t}$ for an orthonormal $Q^t$ where each coordinate of $\temppit{i}{t}$ is independently drawn $\mathcal{N}(0,\sigma^2)$ truncated to $[-\hat{R},\hat{R}]$.
  Let $w' = Q^t w$.
  Then
  \begin{align*}
    \pr{w\cdot \eit{i}{t} \geq r}
      &= \pr{(Q^tw)\cdot (Q^t \eit{i}{t}) \geq r}  \\
      &= \pr{w' \cdot \temppit{i}{t} \geq r}  \\
      &= \pr{\sum_{j=1}^d w_j' \left(\temppit{i}{t}\right)_j \geq r }  \\
  \end{align*}
  Each $\left(\temppit{i}{t}\right)_j$ is $\sigma^2$-subgaussian (by Fact \ref{fact:truncated-subgauss}), so $w_j'\left(\temppit{i}{t}\right)_j$ is $(\sigma w_j')^2$-subgaussian, and their sum is $\sigma^2 \|w'\|^2 = \sigma^2$-subgaussian.
  So by properties of subgaussians, this probability is at most $e^{-r^2/2\sigma^2} \leq \frac{1}{T}$ for $r \geq \sigma \sqrt{2 \ln(T)}$.
%
\end{proof}

We will present a generalization of
\Cref{lemma:singleparam-bounded-perturbed-regret} that considers both
the``small-$\sigma$'' and the ``large-$\sigma$'' regimes.

\begin{lemma}[Generalization of \Cref{lemma:singleparam-bounded-perturbed-regret}]\label{general-bo}
  Let $r = \sigma \sqrt{2\ln(T)}$ and
  $\hat{R} = 2\sigma \sqrt{2\ln(Tkd/\delta)}$ and consider the bounded
  perturbed adversary $\Asig'$ with this choice of $\hat{R}$.  With
  probability at least $1-\frac{\delta}{2}$, for fixed $s$ and $k$ and
  $\sigma \leq O\left((\sqrt{d \ln(Tkd/\delta)})^{-1}\right)$, Greedy
  has regret bounded by
   \[  \max \begin{cases} O\left( \frac{\sqrt{T d s}\left(\ln\frac{Td}{\delta}\right)^{3/2}} {\sigma^2} \right)            &
\sigma \leq \left(2 \sqrt{2 d \ln\left(Tkd/\delta\right)}\right)^{-1}\\
                            O\left( \frac{d^2 \sqrt{T s} \left(\ln\frac{Tkd} {\delta} \right)^{3}} {\sqrt{\sigma}} \right) & \text{otherwise.} \end{cases}  \]
\end{lemma}

\begin{proof}
  By Lemma \ref{lemma:asigprime-bounded}, $\Asig'$ is $R$-bounded and
  $(r,\frac{1}{T})$-centrally bounded, where $R = 1+\sqrt{d}\hat{R}$.
  By Lemma \ref{lemma:singleparam-bounded-adv-diverse}, $\Asig'$ is
  $(r, \lambda_0)$-diverse for
  \begin{align*}
    \lambda_0
      &= \Omega\left(\frac{\sigma^4}{r^2}\right)  \\
      &= \Omega\left(\frac{\sigma^2}{\ln T}\right) .
  \end{align*}
  Therefore, by Theorem \ref{thm:singleparam-conditions-regret},
  with probability $1 - \frac{\delta}{2}$, for fixed $s$ and $k$
  the regret\footnote{One can
    obtain regret bounds for the other cases as well by plugging in
    our bounds on $R$ and $\lambda_0$, but we omit this in order to
    simplify the presentation, as these regimes are not of much interest to
    us in this paper.}  
  of Greedy is bounded by
  \begin{align}
    \text{Regret}(T)
      &\leq O\left( \frac{R^{3/2} \sqrt{T d s \ln(4Td/\delta)}}{\lambda_0}  \right).  \label{eqn:singleparam-multicase-regret}
  \end{align}
  Plugging in $\lambda_0$ and dropping the constant $4$ gives
  \begin{align*}
    \text{Regret}(T)
      &\leq O\left( \frac{R^{3/2} \sqrt{T d s \ln(Td/\delta)} \ln(T)}{\sigma^2}  \right)  \\
      &\leq O\left( \frac{R^{3/2} \sqrt{T d s} \left(\ln\frac{Td}{\delta}\right)^{3/2}} {\sigma^2} \right) .
  \end{align*}
  We have
  \begin{align*}
    R &= 1 + \sqrt{d}\hat{R}  \\
      &\leq 2\max\left\{ 1 ~,~ \sqrt{d}\hat{R} \right\}  \\
      &=    2\max\left\{ 1 ~,~ 2\sigma \sqrt{2 d \ln\left(Tkd/\delta\right)} \right\} .  \\
  \end{align*}
  This gives a ``small-$\sigma$'' and ``large-$\sigma$'' regime.
  So for the case $R = 2$ (which occurs when $\sigma \leq \left(2 \sqrt{2 d \ln\left(Tkd/\delta\right)}\right)^{-1}$), we have with probability $1-\frac{\delta}{2}$,
  \begin{align*}
    \text{Regret}(T)
      &\leq O\left( \frac{\sqrt{T d s}\left(\ln\frac{Td}{\delta}\right)^{3/2}} {\sigma^2} \right).
  \end{align*}
  For the other case of ``large $\sigma$'' and $R = 4\sigma\sqrt{2d\ln(Tkd/\delta)}$, we have with probability $1-\frac{\delta}{2}$,
  \begin{align*}
    \text{Regret}(T)
      &\leq O\left( \frac{R^{3/2} \sqrt{T d s} \left(\ln\frac{Td}{\delta}\right)^{3/2}} {\sigma^2} \right)  \\
      &=    O\left( \frac{d^2 \sqrt{T s} \left(\ln\frac{Tkd} {\delta} \right)^{3}} {\sqrt{\sigma}} \right)
  \end{align*}
\end{proof}

\begin{remark}
  The regret bound for the ``large-$\sigma$'' regime immediately
  follows from the result of \Cref{general-bo}.
\end{remark}

\section{Proofs for the Multiple Parameter Setting} \label{app:multi}

\begin{proof}[Proof of Lemma \ref{lemma:multiparam-greedy-regret}]
  Let $i_*(t) = \argmax_i \bi{i} \cdot \pxit{i}{t}$ denote the optimal
  arm at round $t$.  Its context is $\bi{i^*(t)}$ and for shorthand,
  let $\pxit{i^*}{t}$ denote its context at that round.  Let $i^t$
  denote the arm pulled by Greedy at round $t$.
  \begin{align*}
    \text{Regret}
      &= \sum_{t=1}^T \bi{i^*(t)} \cdot \pxit{i^*}{t} - \bi{i^t} \cdot \pxit{i^t}{t}  \\
  \end{align*}
  We have
  \begin{align*}
    \bi{i^*(t)} \cdot \pxit{i^*}{t} - \bi{i^t} \cdot \pxit{i^t}{t}
      &= \left(\bi{i^*(t)} - \bhit{i^*(t)}{t}\right) \cdot \pxit{i^*}{t} - \left(\bi{i^t} -\bhit{i^t}{t}\right) \cdot \pxit{i^t}{t} ~ + ~ \left( \bhit{i^*(t)}{t} \cdot \pxit{i^*}{t} -\bhit{i^t}{t}\cdot \pxit{i^t}{t} \right)  \\
      &\leq \left(\bi{i^*(t)} - \bhit{i^*(t)}{t}\right) \cdot \pxit{i^*}{t} - \left(\bi{i^t} -\bhit{i^t}{t}\right) \cdot \pxit{i^t}{t}  \\
      &\leq \left\| \bi{i^*(t)} - \bhit{i^*(t)}{t} \right\| R + \left\| \bi{i^t} -\bhit{i^t}{t}\right\| R .
  \end{align*}
  We used that, by definition of Greedy, at each time step
  $\bhit{i^*(t)}{t} \cdot \pxit{i^*}{t} \leq\bhit{i^t}{t}\cdot
  \pxit{i^t}{t}$.  To complete the proof, group all terms by the arms
  $i$.
\end{proof}

\subsection{Perturbed adversary}

In the multiple parameter setting, we construct $\Asig,\Asig',\Asig''$ as follows.
We formally define $\tempit{i}{t}, \temppit{i}{t}, \tempppit{i}{t}$ exactly as in Section \ref{subsection:single-perturbed-adv}, namely $\tempit{i}{t} \sim \text{N}(0,\sigma^2\Id_d)$ i.i.d., $\temppit{i}{t} \in \R^d$ has each coordinate i.i.d. from an $[-\hat{R},\hat{R}]$-truncated $\mathcal{N}(0,\sigma^2)$ distribution, and $\tempppit{}{}$ has all coordinates of each $\tempppit{i}{t}$ drawn from a joint independent Gaussian conditioned on at least one coordinate of some $\tempppit{i}{t}$ having absolute value at least $\hat{R}$.

Now, given $\bhit{i}{t}$ for each arm $i$ at round $t$, let
$Q_i^t$ be an orthonormal change-of-basis matrix such that
$Q_i^t \bhit{i}{t} = (\|\bhit{i}{t}\|, 0, \ldots, 0)$.
Then for each $i$ and history $\hist{t}$, we let
\begin{align*}
  \Asig(\hist{t})_i   &= \Adv(\hist{t})_i + (Q_i^t)^{-1}\tempit{i}{t}  \\
  \Asig'(\hist{t})_i  &= \Adv(\hist{t})_i + (Q_i^t)^{-1}\temppit{i}{t}  \\
  \Asig''(\hist{t})_i &= \Adv(\hist{t})_i + (Q_i^t)^{-1}\tempppit{i}{t}.
\end{align*}

\begin{lemma}[Analogue of Lemma \ref{lemma:rhat-prob-mixture}] \label{lemma:multi-rhat-prob-mixture}
  In the multiple parameter setting, $\Asig$, $\Asig'$, and $\Asig''$ satisfy the following:
  \begin{enumerate}
    \item $\Asig$ is the Gaussian $\sigma^2$-perturbed adversary.
    \item $\Asig$ is a mixture distribution of $\Asig'$ and $\Asig''$; furthermore, the probability of $\Asig'$ in this mixture is at least $1 - 2Tkde^{-\hat{R}^2 / (2\sigma^2)}$.
    \item Under $\Asig'$, at each time step $t$, each coordinate of $Q_i^t \eit{i}{t}$ is distributed independently as a $\mathcal{N}(0,\sigma^2)$ variable truncated to $[-\hat{R},\hat{R}]$.
  \end{enumerate}
\end{lemma}
The proof is identical to the proof of Lemma \ref{lemma:rhat-prob-mixture} with notational changes and is omitted; the same holds for the following Lemmas \ref{lemma:multi-asigprime-bounded} and \ref{lemma:multiparam-bounded-adv-diverse}.

\begin{lemma}[Analogue of Lemma \ref{lemma:asigprime-bounded}] \label{lemma:multi-asigprime-bounded}
  For any choice of
  $\hat{R}$, $\Asig'$ is $R$-bounded and
  $(r,\frac{1}{T})$-centrally bounded for
  $r \geq \sigma \sqrt{2\ln T}$ and $R = 1+\sqrt{d}\hat{R}$.
\end{lemma}

\begin{lemma}[Analogue of Lemma \ref{lemma:singleparam-bounded-adv-diverse}] \label{lemma:multiparam-bounded-adv-diverse}
  $\Asig'$ satisfies $(r,\lambda_0)$ diversity for $\lambda_0 = \Omega(\sigma^4 / r^2)$
  when choosing $\hat{R} \geq 2r$ and $r \geq \sigma$.
\end{lemma}

Theorem \ref{theorem:multiparam-greedy-regret} is a special case of the following.

\begin{theorem} \label{theorem:multiparam-greedy-regret-regimes}
  In the multiple parameter setting, against the $\sigma$-perturbed adversary $\Asig$, for fixed $k$ (number of arms) and $s$ (rewards' subgaussian parameter):
  \begin{enumerate}
    \item In the ``small-$\sigma$'' regime with $\sigma \leq O\left(\frac{1}{\sqrt{d\ln(Tkd/\delta)}}\right)$, with a warm start size of
            \[ n =  O\left( \frac{ds}{\sigma^{12} \min_j \|\bi{j}\|^2} \ln\left(\frac{dks}{\delta \sigma \min_j\|\bi{j}\|}\right) \right) \]
          Greedy has, with probability at least $1-\delta$,
            \[ \text{Regret} \leq O\left( \frac{\sqrt{T k d s} \left( \ln \frac{T d k}{\delta}\right)^{3/2}}{\sigma^2} \right).  \]
    \item Otherwise (``large-$\sigma$''), with a warm start size of \ar{Appendix}
            \[ n = \max \begin{cases}  O\left(d \left(\ln\frac{\sigma T d k s} {\delta}\right)^2 \right)  \\
                                       O\left(\frac{d s \ln(T)^{3}}{\sigma^6 \min_j \|\bi{j}\|^2} \ln \left(\frac{d k s \ln(T)}{\sigma \min_j \|\bi{j}\|} \right)  \right), \end{cases} \]
          Greedy has, with probability at least $1-\delta$,
            \[ \text{Regret} \leq O\left( \frac{d^{5/4} \sqrt{T k s} \left(\ln\frac{T d k}{\delta}\right)^{9/4}}{\sqrt{\sigma}} \right). \]
  \end{enumerate}
\end{theorem}

\begin{proof}
  As in the single parameter setting, we split the probability space
  of the adversary into two cases.  With probability at least
  $1-\frac{\delta}{2}$, Greedy faces the bounded perturbed adversary
  $\Asig'$.
  For analysis, we choose $r = \sigma \sqrt{2\ln T}$ and $\hat{R} = 3\sigma\sqrt{2\ln(Tkd/\delta)}$, which implies $\hat{R} \geq 2r$ and also implies $\hat{R} \geq \frac{5r}{4} + \sigma \sqrt{2\ln(8d)}$, as required to apply Corollary \ref{cor:truncated-margins} and conclude the margin condition.
  (These are the same choices as in Theorem \ref{thm:singleparam-perturbed-regret}, but with a factor $3$ for $\hat{R}$ instead of $2$.)
  By Lemma \ref{lemma:rhat-prob-mixture}, the probability of facing $\Asig'$ is at least $1-\frac{\delta}{2}$.
  By Lemma \ref{lemma:asigprime-bounded}, $\Asig'$ is $(r, \frac{1}{T})$-centrally bounded and is $R$-bounded for
  \begin{align*}
    R &= 1 + \sqrt{d}\hat{R}  \\
      &\leq 2 \max\left\{1, 3 \sigma \sqrt{2d \ln(Tkd/\delta)} \right\} .
  \end{align*}
  By Lemma \ref{lemma:singleparam-bounded-adv-diverse}, $\Asig'$ satisfies $(r,\lambda_0)$-diversity for
  \begin{align*}
    \lambda_0
      &= \Omega\left(\frac{\sigma^4} {r^2}\right)  \\
      &= \Omega\left(\frac{\sigma^2} {\ln(T)} \right).
  \end{align*}
  Finally, by Corollary \ref{cor:truncated-margins}, $\Asig'$ satisfies $(r, \alpha, \gamma)$-margins for $\alpha = \frac{\sigma^2}{r}$, $\gamma = \frac{1}{80}$.
  The general result of Theorem \ref{thm:multiparam-conditions-regret} gives that, for
  \begin{align*}
    n \geq \max \begin{cases}
                  \Theta\left(\ln(k/\delta)\right)  \\
                  \Theta\left(\frac{R^2 \ln(R^2 dk /\delta)}{\lambda_0}\right)  \\
                  \Theta\left(\frac{Rds}{\left(\alpha \lambda_0 \min_j \|\bi{j}\|\right)^2} \ln\left(\frac{R d^2 k s}{\delta \left(\alpha \lambda_0 \min_j \|\bi{j}\|\right)^2} \right) \right) ,
                \end{cases}
  \end{align*}
  with probability $1-\frac{\delta}{2}$ against $\Asig'$,
  \begin{align*}
    \text{Regret}
      &\leq O\left( \frac{R^{3/2} \sqrt{T k d s} \left( \ln \frac{T d k}{\delta}\right)^{3/2}}{\sigma^2} \right).
  \end{align*}
  This implies that the above regret bounds hold with probability $1-\delta$ against $\Asig$.

  We now consider the regimes.
  In the small-$\sigma$ regime of $\sigma \leq O\left(\frac{1}{\sqrt{2d\ln(Tkd/\delta)}}\right)$, we have $R,r \leq O(1)$.
  So
    \[ \text{Regret} \leq O\left( \frac{\sqrt{T k d s} \left( \ln \frac{T d k}{\delta}\right)^{3/2}}{\sigma^2} \right).  \]
  Here the warm start size required is $n \geq \max\{n_1,n_2,n_3\}$ with
  \begin{align*}
    n_1 &= O\left(\ln(k/\delta)\right).  \\
    n_2 &= O\left(\frac{R^2 r^2 \ln(R^2 d k/\delta)}{\sigma^4}\right) \\
        &= O\left(\frac{\ln(d k/\delta)}{\sigma^4}\right) .  \\
    n_3 &= O\left(\frac{ds}{\left((\sigma^2/r)(\sigma^4/r^2)\min_j \|\bi{j}\|\right)^2} \ln\left(\frac{d k s}{\delta (\sigma^2/r) (\sigma^4/r^2) \min_j \|\bi{j}\|} \right) \right)  \\
        &= O\left(\frac{ds}{\sigma^{12} \min_j \|\bi{j}\|^2} \ln \left(\frac{dks}{\delta \sigma \min_j \|\bi{j}\|} \right) \right) .
  \end{align*}
  For fixed $k$ and $s$, the bound on $n_3$ is asymptotically largest.

  In the large-$\sigma$ regime, we have $r = O(\sigma \sqrt{\ln T})$ and $R = \Theta\left(\sigma \sqrt{d \ln(Tkd/\delta)}\right)$.
  So
    \[ \text{Regret} \leq O\left( \frac{d^{5/4} \sqrt{T k s} \left(\ln\frac{T d k}{\delta}\right)^{9/4}}{\sqrt{\sigma}} \right). \]
  Here $n \geq \max\{n_1,n_2,n_3\}$ with the same $n_1$ as above, and
  \begin{align*}
    n_2 &= O\left(d \left(\ln\frac{\sigma T d k s} {\delta}\right)^2 \right).  \\
    n_3 &= O\left(\frac{d s \ln(T)^{3}}{\sigma^6 \min_j \|\bi{j}\|^2} \ln \left(\frac{d k s \ln(T)}{\sigma \min_j \|\bi{j}\|} \right)  \right).
  \end{align*}
  For fixed $k,s$ these are asymptotically larger than $n_1$.
\end{proof}

\section{Proofs of Lemmas from the Lower Bound}\label{sec:lb-lemmas}

   \begin{proof}[Proof of Lemma~\ref{lem:big-no-drift}]
     Fix $\delta = \frac{1}{8}$ for the remainder of the proof.  With
     probability $1-\delta$, we can assume
     $\pxit{i}{t} \in\left[\uxit{i}{t} - \frac{1}{100}, \uxit{i}{t} +
       \frac{1}{100}\right]$ by a union bound over all rounds $t$ and
     a Chernoff bound on Gaussian noise, by the upper bound on
     $\sigma$.

     Consider breaking all of $T$ rounds into epochs of size
     $c n^{\frac{1}{3}}$ for $c = \frac{1}{1.01}$. If we show that
     $\bhit{i}{t} > \bi{i} - \frac{1}{\sqrt{n}}$ for each
     $t \in G= \{0, c\cdot n^{\frac{1}{3}}, 2c\cdot n^{\frac{1}{3}},
     \ldots\}$, then $\bhit{i}{t'} \geq 1-\frac{2}{\sqrt{n}}$ for any
     $t' \notin G$, since with probability $1-\delta$, we have that
     $\pxit{i}{t} \nit{i}{t} \leq 1.01\nit{i}{t} \leq 1.01
     \sqrt{\ln(T/\delta)}$ for all rounds, so since
     $T/\delta = 8 T < 2^{n^{1/3}}$ , the total cumulative shift between
     rounds in $G$ in the numerator cannot exceed $\sqrt{n}$ and
     therefore the estimator does not change by $\frac{1}{\sqrt{n}}$
     between epochs.

We will calculate the probability that $j\in G$ samples are such that
$\frac{\sum_{t = 1}^n \pxit{i}{t} \yit{i}{t} +\sum_{t = n+1}^j \pxit{i}{t} \yit{i}{t}}
{\sum_{t = 1}^n (\pxit{i}{t} )^2 +\sum_{t = n+1}^j (\pxit{i}{t})^2} \leq 1-
\frac{2}{\sqrt{n}}$ but
$\frac{\sum_{t=1}^n \pxit{i}{t} \yit{i}{t}}{\sum_{t=1}^n (\pxit{i}{t})^2} \geq 1 +
\frac{c_i}{\sqrt{n}}$ for $c_i > 0$. The OLS estimator for $\bhit{i}{j}$ is
\begin{align}
  \bhit{i}{j} = \frac{\sum_{t=1}^n \pxit{i}{t} \yit{i}{t} +\sum_{t=n+1}^{j+n} \pxit{i}{t} \yit{i}{t} }{\sum_{t=1}^n (\pxit{i}{t} )^2 +\sum_{t=n+1}^{j+n} (\pxit{i}{t})^2} = \bi{i}+  \frac{\sum_{t=1}^n \pxit{i}{t} \nit{i}{t} +\sum_{t=n+1}^{j+n} \pxit{i}{t} \nit{i}{t} }{\sum_{t=1}^n (\pxit{i}{t} )^2 +\sum_{t=n+1}^{j+n} (\pxit{i}{t})^2}.\label{eqn:ols-noise}
\end{align}
which we will manipulate in two separate ways, one for $j > 100 n$ and
one for $j\leq 100n $.

Fix the future values of $\pxit{i}{t}$ for $t \in [j+n]$. Then the
$j$th OLS estimator,  has noise term of
the form
\[\bhit{i}{j} - \bi{i} = \frac{\sum_{t = 1}^n \pxit{i}{t}\nit{i}{t} +\sum_{t = n+1}^{j+n} \pxit{i}{t}\nit{i}{t} }{\sum_{t = 1}^{j+n} (\pxit{i}{t})^2}\]
Notice that the distribution over
$\sum_{t = 1}^n \pxit{i}{t}\nit{i}{t}$ stochastically dominates that
of the same term with each value of $\nit{i}{t}$ replaced with an iid
draw from a Gaussian, since we have conditioned on this noise term
being something larger than its minimum value. So, consider replacing
these terms (in the numerator only) with $\bar{\nit{i}{t}}$, new iid
draws from the standard Gaussian distribution to construct the new
estimator this estimator $\overline{\bhit{i}{t}}^j$: it has only has
larger probability of being less than any particular $x$ than does
$\bhit{i}{j}$. Since all of these noise terms are iid and Gaussian, we
then have that
$\overline{\bhit{i}{j}} \sim \mathcal{N}\left(\bi{i},
  \frac{1}{\sum_{t=1}^{j+n} (\pxit{i}{t})^2}\right)$. Then, the
probability that $\overline{\bhit{i}{j}} \leq \bi{i}- \frac{2}{\sqrt{n}}$
is at most $e^{-{\frac{\sum_{t=1}^{j+n} (\pxit{i}{t})^2}{2n}}}$, using
Hoeffding's inequality for subgaussian random variables.  Finally,
with probability $1-\delta$ we have that
$\pxit{2}{t} \geq .99$ for all rounds $t$ as
mentioned above, implying
$(\pxit{2}{t})^2 \geq 
.5$. In sum this is upper-bounded by
$e^{-{\frac{(j+n) }{4n}}} = e
^{-\frac{1}{4}\left(\frac{j}{n}+1\right)}$.

Now, consider $j < 100 n$. We now argue just about those $j$ samples
in the numerator of the OLS estimator.  Again fixing the values of
$\pxit{i}{t}$ for all $t$, consider the term
$\sum_{t=n+1}^{j+n} \pxit{i}{t} \nit{i}{t}$.
Since $\bhit{i}{} \geq \bi{i} +\frac{c_i}{\sqrt{n}}$, this implies
\begin{align*}
  \sum_{t=1}^n \pxit{i}{t} \yit{i}{t} & =  \sum_{t=1}^n \bi{i}(\pxit{i}{t})^2 + \pxit{i}{t} \nit{i}{t}\\
                           &\geq (\bi{i} + \frac{c_i}{\sqrt{n}}) \left(\sum_{t=1}^n (\pxit{i}{t})^2\right)\\
\end{align*}
which then means that
\begin{align*}
  \sum_{t=1}^n \pxit{i}{t} \nit{i}{t} & \geq \left(\bi{i} + \frac{c_i}{\sqrt{n}} -\bi{i}\right) \sum _{t=1}^n (\pxit{i}{t})^2\\
                              & =  \frac{c_i}{\sqrt{n}}  \sum _{t=1}^n (\pxit{i}{t})^2\\
                              &  \geq \frac{c_i\sqrt{n}}{2}
\end{align*}
where the  last line holds by the lower bound on $\pxit{i}{t} \geq .99$.  Then, if this is true for the first $n$
terms, if $\bhit{i}{t} < \bi{i}- \frac{2}{\sqrt{n}}$, it must be
that $\sum_{t=n+1}^{j+n} \pxit{i}{t} \nit{i}{t} \leq -
\frac{c_i\sqrt{n}}{2}$. Again, using Hoeffding for the sum of
subgaussian random variables, the probability of this event is at
most
$e^{- \frac{c_i^2 n}{8\sum_{t=n}^{j+n} (\pxit{i}{t})^2}} \leq e^{-
  \frac{c_i^2 n}{16 j}} \leq e^{- \frac{c_i^2 }{1600}} $, using the
upper-bound on $\pxit{i}{t}$ and on $j \leq 100 n$.

In total, when all of these probability $1-\delta$ events hold, the
expected number of rounds for which
$\bhit{i}{j} \leq \bi{i}-\frac{2}{\sqrt{n}}$ is at most
\begin{align*}
 &  \sum_{j \in \{0, cn^{\frac{1}{3}}, \ldots  100n\}}  e^{- \frac{c_i^2 }{1600}}    + \sum_{j \in \{100n + cn^{1/3}, 100n + 2cn^{\frac{1}{3}}, \ldots\}  }^\infty  e
  ^{-\frac{1}{4}\left(\frac{j}{n}+1\right)}\\
  & \leq   100 n^{2/3} e^{-(c_i )^2/1600 } + e^{-25}\sum_{j \in \{cn^{1/3}, 2 cn^{1/3}, \ldots \}}^\infty e^{-\frac{j}{4n}}\\
  & \leq   100 n^{2/3} e^{-(c_i )^2/1600 } + e^{-25}\int_{j = 0}^\infty e^{-\frac{j c n^{1/3}}{4n}} \\
          & \leq   100 n^{2/3} e^{-(c_i)^2/1600 } + 4n^{2/3} e^{-25} \\
\end{align*}
which, for $c_i > 120$, is at most $0.00048 n^{2/3}$. This upper bound
on the expected number of rounds in which the inequality fails to hold
holds with probability at least $1-4 \delta \geq \frac{1}{2}$ by a
union bound over any one of these $1-\delta$-events failing to hold.
\end{proof}

\begin{proof}[Proof of Lemma~\ref{lem:variance-ols}]
  We write the initial OLS estimator for any arm as
  \begin{align}
    \bhit{i}{} = \frac{\sum_{t \in [n]} \pxit{i}{t} \yit{i}{t}}{\sum_{t\in [n]} (\pxit{i}{t})^2} =  \frac{\sum_{t \in [n]} \pxit{i}{t} \left(\bi{i} \pxit{i}{t} + \nit{i}{t}\right)}{\sum_{t\in [n]} (\pxit{i}{t})^2} =   \bi{i} + \frac{\sum_{t\in [n]} \pxit{i}{t}\nit{i}{t}}{\sum_{t\in[n]}(\pxit{i}{t})^2}\label{eqn:ols-init}
    \end{align}
  We now note that
  $ \frac{\sum_t \pxit{1}{t} \nit{1}{t}}{\sum_t (\pxit{1}{t})^2}\sim \mathcal{N}(0,
  \frac{1}{\sum_t (\pxit{1}{t})^2})$ for fixed $\pxit{1}{t}$'s, since $\nit{1}{t}$ is
  drawn according to a Gaussian distribution.

  For any constant $c_i$, with constant probability, a Gaussian random
  variable is $c_i$ standard deviations away from its mean, so with
  constant probability
  \begin{align}
    \frac{\sum_t \pxit{i}{t} \nit{i}{t}}{\sum_t (\pxit{i}{t})^2} \leq  -
    \frac{c_i}{\sqrt{\sum_t (\pxit{i}{t})^2}} = - \frac{c_i} {\sqrt{\sum_t (\uxit{i}{t})^2 + 2 \uxit{i}{t}\eit{i}{t} + (\eit{i}{t})^2}}\label{eqn:manip-noise-upper}
    \end{align}
where the last inequality came from expanding the definition of
$\pxit{i}{t}$.

Analogously, we can upper-bound the noise with constant probability
for any constant $c_i$ by 
\begin{align}
    \frac{\sum_t \pxit{i}{t} \nit{i}{t}}{\sum_t (\pxit{i}{t})^2} \geq  
    \frac{c_i}{\sqrt{\sum_t (\pxit{i}{t})^2}} =  \frac{c_i} {\sqrt{\sum_t (\uxit{i}{t})^2 + 2 \uxit{i}{t}\eit{i}{t} + (\eit{i}{t})^2}}\label{eqn:manip-noise-lower}
    \end{align}

We then continue by noting that, fixing the values of $\uxit{i}{t}$ for all $t$, 
$\sum_t \eit{i}{t}\uxit{i}{t} \sim \mathcal{N}(0, \sum_{t} \uxit{i}{t} \sigma^2)$ and
$\sum_t (\eit{i}{t} / \sigma)^2\sim \chi^2(n)$, and so with probability at
least $1-\delta$,
$|\sum_t \eit{i}{t} \uxit{i}{t} |\leq 2 \sigma \sqrt{\sum_{t} \uxit{i}{t} \ln \frac{2}{\delta}} $ and
also
 $0 \leq \sum_t (\eit{1}{t} / \sigma)^2 \leq n + 2 \sqrt{n \ln\frac{2}{\delta} } +
2 \sqrt{\ln\frac{2}{\delta}}$.

Thus, combining Equation~\ref{eqn:manip-noise-upper} with these facts, with constant
probability we have that  for sufficiently large $n$
 \begin{align*} \frac{\sum_t \pxit{i}{t} \nit{1}{t}}{\sum_t (\pxit{i}{t})^2} &  \leq -
   \frac{c_i} {\sqrt{\sum_t (\uxit{i}{t})^2 + 4   \sigma \sqrt{\sum_{t} \uxit{i}{t} \ln \frac{2}{\delta}} +\sigma^2 \left(n + 2 \sqrt{n \ln\frac{2}{\delta} } +
   2 \sqrt{\ln\frac{2}{\delta}}\right)}}\\
                                                                             & \leq \frac{-c_i}{\sqrt{d_i^2n + 4 \sqrt{d_i n\ln\frac{2}{\delta}} + n + 2 \sqrt{n\ln\frac{2}{\delta}} + 2\sqrt{\ln\frac{2}{\delta}} }}\\
                                                                             & \leq \frac{-c_i}{d_i\sqrt{n + 4 \sqrt{d_i n} + n + 2 \sqrt{n} + 2 }}
   \\
                                                                             & \leq \frac{-c_i}{10 d_i\sqrt{n}}
 \end{align*}
 where the second inequality follows from the assumptions that
 $\sigma \in [0,1]$ and $ \uxit{i}{t} = d_i $, the third from choosing
 $\delta \geq \frac{3}{4}$ implying $\ln\frac{2}{\delta} < 1$, and the
 last line holds for sufficiently large $n$. Taking
 $c'_i = \tfrac{c_i}{ 10 d_i}$ yields the desired lower bound when
 combined with Equation~\ref{eqn:ols-init}.

Similarly, by these facts and
Equation~\ref{eqn:manip-noise-lower}, with constant probability we have 
 \begin{align*} \frac{\sum_t \pxit{i}{t} \nit{1}{t}}{\sum_t (\pxit{i}{t})^2} &  \geq 
   \frac{c_i} {\sqrt{\sum_t (\uxit{i}{t})^2 + 4   \sigma \sqrt{\sum_{t} \uxit{i}{t} \ln \frac{2}{\delta}} + \sigma^2 \left(n + 2 \sqrt{n \ln\frac{2}{\delta} } +                                                                               2 \sqrt{\ln\frac{2}{\delta}}\right)}}\\   
                                                                             & = \frac{c_i} {\sqrt{d_i^2 n + 4   \sigma \sqrt{d^2_i n  \ln \frac{2}{\delta}} +\sigma^2 \left(n + 2 \sqrt{n \ln\frac{2}{\delta} } +                                                                               2 \sqrt{\ln\frac{2}{\delta}}\right)}}\\
                                                                             & \geq \frac{c_i} {\sqrt{d_i^2 n + 4  \sqrt{d^2_i n } + \left(n + 2 \sqrt{n} +                                                                               2 \right)}}\\
                                                                             &   \geq \frac{c_i} {\sqrt{d_i^2 n + 4  \sqrt{d^2_i} n  + \left(n + 2 \sqrt{n} +                                                                               2 \right)}}\\
                                                                             &   \geq \frac{c_i} {\sqrt{ n + 4 n  + \left(n + 2 \sqrt{n} +                                                                               2 \right)}}\\
   &   \geq \frac{c_i} {\sqrt{8 n}}
 \end{align*}
 where the second inequality uses $\sigma < 1$ and inserting
 $\delta \geq \frac{3}{4}$, and where $d_i \leq 1$ was used in the
 second-to-last step and the last holds for sufficiently large
 $n$. Replacing $c'_i = \tfrac{c_i}{\sqrt{8}}$ yields the desired
 lower bound when combined with Equation~\ref{eqn:ols-init}.
\end{proof}

\end{document}